\documentclass[twoside,11pt]{article}

\usepackage{jmlr2e}

\usepackage{amsfonts,mathrsfs,amsmath}
\usepackage{amsmath}
\usepackage{bm}
\usepackage{color}
\usepackage{graphicx}
\usepackage{caption,subcaption}
\usepackage{url}
\usepackage{enumerate,anysize,color}

\usepackage{bm}
\usepackage{algorithm}
\usepackage{algpseudocode}

\newtheorem{thm}{Theorem}
\newtheorem{pro}{Proposition}
\newtheorem{rem}{Remark}
\newtheorem{Exa}{Example}
\newtheorem{as}{Assumption}

\usepackage{tablefootnote}
\usepackage{booktabs}

\newcommand{\be}{\begin{equation}}
\newcommand{\ee}{\end{equation}}
\newcommand{\bea}{\begin{eqnarray*}}
\newcommand{\eea}{\end{eqnarray*}}
\newcommand{\bflalign}{\begin{flalign*}}
\newcommand{\eflalign}{\end{flalign*}}

\newcommand{\mR}{\mathbb{R}}
\newcommand{\mN}{\mathbb{N}}
\newcommand{\mE}{\mathbb{E}}
\newcommand{\mcE}{\mathcal{E}}

\newcommand{\mcS}{\mathcal{S}}

\newcommand{\mcN}{\mathcal{N}}

\newcommand{\mcX}{\mathcal{X}}
\newcommand{\mcV}{\mathcal{V}}
\newcommand{\bz}{{\bf z}} 
\newcommand{\bx}{{\bf x}} 
\newcommand{\by}{{\bf y}}

\newcommand{\tr}{\operatorname{tr}}
\newcommand{\la}{\langle}
\newcommand{\ra}{\rangle}
\newcommand{\eref}[1] {(\ref{#1})}

\newcommand{\TK}{\mathcal{T}} 

\newcommand{\TKL}{\mathcal{T}_{\PRegPar}} %
\newcommand{\TXL}{\mathcal{T}_{{\bf x}\PRegPar}}
\newcommand{\TKLP}{\mathcal{T}_{\tilde{\lambda}}} %
\newcommand{\TXLP}{\mathcal{T}_{{\bf x}\PRegPar}}

\newcommand{\LKL}{\LK_{\PRegPar}}

\newcommand{\LK}{\mathcal{L}}
\newcommand{\LX}{\mathcal{L}_{\bf x}}
\newcommand{\IK}{\mathcal{S}_{\rho}}
\newcommand{\TX}{\mathcal{T}_{\bf x}}
\newcommand{\SX}{\mathcal{S}_{\bf x}}

\newcommand{\TXS}{\mathcal{T}_{{\bf x}_s}} 
\newcommand{\LXS}{\mathcal{L}_{{\bf x}_s}} %
\newcommand{\SXS}{\mathcal{S}_{{\bf x}_s}} 

\newcommand{\HK}{H}

\newcommand{\LR}{L^2_{\rho_X}}

\newcommand{\J}{{\bf J}}

\newcommand{\bfep}{{\bm{\epsilon}}} 

\newcommand{\GL}{G_t} 
\newcommand{\GLB}{\widetilde{G}_{\lambda}} 
\newcommand{\FL}{r_{t+1}} 

\newcommand{\FR}{f_{\rho}}
\newcommand{\FH}{f_{\HK}}

\newcommand{\DZF}{\Delta^{\bf z}_1}
\newcommand{\DZS}{\Delta^{\bf z}_2}
\newcommand{\DZT}{\Delta^{\bf z}_3}

\newcommand{\RegPar}{\Sigma_1^t}
\newcommand{\RegParGD}{\lambda_t}
\newcommand{\PRegPar}{\tilde{\lambda}}

\newcommand{\LESRA}{g_{\lambda}^{{\bf z}_s}} 
\newcommand{\FLESRA}{g_{\lambda}^{{\bf z}_1}}  
\newcommand{\EDSRA}{\bar{g}_{\lambda}^{\Samples}} 

\newcommand{\LEPSRA}{h_{\lambda}^{{\bf z}_s}} 
\newcommand{\FLEPSRA}{h_{\lambda}^{{\bf z}_1}}  
\newcommand{\EDPSRA}{\bar{h}_{\lambda}^{\Samples} }
\newcommand{\PFSRA}{\tilde{r}_{\lambda}} 
\newcommand{\EPSRA}{h_{\lambda}^{{\bf z}}} 
\newcommand{\ESRA}{g_{\lambda}^{{\bf z}}} 

\newcommand{\Samples}{\bar{\bf z}}
\newcommand{\Inputs}{\bar{\bf x}}
\newcommand{\Outputs}{\bar{\bf y}}

\usepackage{longtable}

\ShortHeadings{Optimal Convergence for Distributed Learning with SGM and SA}{Lin and Cevher}
\firstpageno{1}

\begin{document}

\title{Optimal Convergence for Distributed Learning with Stochastic Gradient Methods and Spectral Algorithms}

\author{\name Junhong Lin \email junhong.lin@epfl.ch 
       \AND
       \name Volkan Cevher \email volkan.cevher@epfl.ch \\
       \addr Laboratory for Information and Inference Systems \\
   \'{E}cole Polytechnique F\'{e}d\'{e}rale de Lausanne \\
   CH1015-Lausanne, Switzerland}

\editor{}

\maketitle

\begin{abstract}
We study generalization properties of distributed algorithms in the setting of nonparametric regression over a reproducing kernel Hilbert space (RKHS). We first investigate distributed stochastic gradient methods (SGM), with mini-batches and multi-passes over the data.
We show that optimal generalization error bounds (up to logarithmic factor) can be retained for distributed SGM  provided that the partition level is not too large. 
We then extend our results to spectral algorithms (SA), including kernel ridge regression (KRR), kernel principal component analysis, and gradient methods. 
Our results are superior to the state-of-the-art theory. Particularly, our results show that distributed SGM has a smaller theoretical computational complexity, compared with distributed KRR and classic SGM. Moreover, even for non-distributed SA, they provide the first optimal, capacity-dependent convergence rates, for the case that the regression function may not be in the RKHS.

%
%
\end{abstract}


\section{Introduction}
In statistical learning theory, a set of $N$ input-output pairs from an unknown distribution is observed.  
The aim is to learn a function which can predict future outputs given the corresponding inputs.   
The quality of a predictor is often measured in terms of the mean-squared error. In this case, the conditional mean, which is called as the regression function, is optimal among all  the measurable functions \citep{cucker2007learning,steinwart2008support}.

In nonparametric regression problems, the properties of the regression function are not known a priori. 
Nonparametric approaches, which can adapt their complexity to the problem, are key to good results. 
Kernel methods is one of the most common nonparametric approaches to learning \citep{scholkopf2002learning,shawe2004kernel}. It is based on 
choosing a RKHS as the hypothesis space in the design of learning algorithms.  With an appropriate reproducing kernel,
RKHS can be used to approximate any smooth function.

The  classical algorithms to perform learning task are regularized algorithms, such as KRR (also called as Tikhonov regularization in inverse problems), kernel principal component regression (KPCR,  also known as spectral
cut-off regularization in inverse problems),  and more generally, SA.  
From the point of view of inverse problems, such approaches amount to solving an empirical, linear operator equation with the empirical covariance operator replaced by a regularized one \citep{engl1996regularization,bauer2007regularization,gerfo2008spectral}.  Here, the regularization term controls the complexity of the solution to against over-fitting and to ensure best generalization ability.  
Statistical results on generalization error had been developed in \citep{smale2007learning,caponnetto2007optimal}
for KRR and in \citep{caponnetto2006,bauer2007regularization}  for SA. 

Another type of algorithms to perform learning tasks is based on iterative
procedure \citep{engl1996regularization}. In this kind of algorithms, an empirical objective function
is optimized in an iterative way with no explicit constraint or penalization, and the
regularization against overfitting is realized by early-stopping the empirical procedure.
Statistical results on generalization error and the regularization roles of the number of iterations/passes have been investigated in \citep{zhang2005boosting,yao2007early} for gradient methods  (GM, also known as Landweber algorithm in inverse problems), in \citep{caponnetto2006,bauer2007regularization} for accelerated gradient methods (AGM, known as $\nu$-methods in inverse problems) in \citep{blanchard2010optimal} for conjugate gradient methods (CGM), and in \citep{lin2017optimal} for (multi-pass) SGM. Interestingly, GM and AGM can be viewed as special instances of SA \citep{bauer2007regularization}, but CGM and SGM can not \citep{blanchard2010optimal,lin2017optimal}. 

The above mentioned algorithms suffer from computational burdens at least of order $O(N^2)$ due to the nonlinearity of kernel methods.  Indeed, a standard execution of KRR requires  $O(N^2)$ in space and $O(N^3)$ in time, while SGM after $T$-iterations requires $O(N)$ in space and $O(N T)$ (or $T^2$) in space.  Such approaches would be prohibitive when dealing with large-scale learning problems.  These thus motivate one to study distributed learning algorithms \citep{mcdonald2009efficient,zhang2012communication}.
The basic idea of  distributed learning is very simple: randomly divide a dataset of size $N$ into $m$ subsets of equal size, compute an independent estimator using a fixed algorithm on each subset, and then average the local solutions into a global predictor.
Interestingly, distributed learning technique has been successfully combined with  KRR \citep{zhang2015divide,lin2017distributed} and more generally, SA \citep{guo2017learning,blanchard2016parallelizing}, and it has been shown that statistical results on generalization error can be retained provided that the number of partitioned subsets is not too large.  Moreover, it was highlighted \citep{zhang2015divide} that  distributed KRR not only
allows one to handle large datasets that restored on multiple machines, but also 
leads to a substantial reduction in computational
complexity versus the standard approach of performing KRR on all $N$ samples.

In this paper, we study distributed SGM, with multi-passes over the data and mini-batches. The algorithm is a combination of distributed learning technique and (multi-pass) SGM \citep{lin2017optimal}: it randomly partitions a dataset of size $N$ into $m$ subsets of equal size, computes an independent estimator by SGM for each subset, and then averages the local solutions into a global predictor.  
We show that with appropriate choices of algorithmic parameters, optimal generalization error bounds up to a logarithmic factor can be achieved for distributed SGM provided that the partition level $m$ is not too large.  

The proposed configuration has certain advantages on computational complexity. For example, without considering any benign properties of the problem such as the regularity of the regression function \citep{smale2007learning,caponnetto2007optimal} and a capacity assumption on the RKHS \citep{zhang2006learning,caponnetto2007optimal}, even implementing on a single machine, distributed SGM has a convergence rate of order $O(N^{-1/2} \log N)$, with a computational complexity $O(N)$ in space and $O(N^{3/2})$ in time, compared with $O(N)$ in space and $O(N^2)$ in time of classic SGM performing on all $N$ samples,
or $O(N^{3/2})$ in space and $O(N^2)$  in time of distributed KRR.  Moreover, the  approach dovetails naturally with parallel and distributed computation: we are guaranteed a superlinear speedup with $m$ parallel processors (though we must still communicate the function estimates from each processor).  

The proof of the main results is based on a similar (but a bit different) error decomposition from \citep{lin2017optimal}, which decomposes the excess risk into three terms: bias, sample and computational variances.  The error decomposition allows one to study distributed GM and distributed SGM simultaneously.
Different to those in \citep{lin2017optimal} which rely heavily on the intrinsic relationship of GM with the square loss, in this paper, an integral operator approach \citep{smale2007learning,caponnetto2007optimal} is used, combining with some novel and refined analysis,  
see Subsection  \ref{subsec:error} for further details.


We then extend our analysis to distributed SA and derive similar optimal results on generalization error for distributed SA, based on the fact that GM is a special instance of SA. 

This paper is an extended version of the conference version \citep{lin2018optimal} where results for distributed SGM are given only.
In this version, we additionally provide statistical results for distributed SA,  including their proofs, as well as some further discussions.

We highlight that our contributions are as follows. 
\begin{itemize}
	\renewcommand{\labelitemi}{$-$}
	\setlength{\parskip}{0pt}
	\setlength{\itemsep}{0pt plus 1pt}
	\item We provide the first results with optimal convergence rates (up to a logarithmic factor) for distributed SGM, showing that distributed SGM has a smaller theoretical computational complexity, compared with distributed KRR and non-distributed SGM. As a byproduct,  we derive  optimal convergence rates (up to a logarithmic factor) for non-distributed SGM, which improve the results in \citep{lin2017optimal}. 
	\item Our results for distributed SA improves previous results from \citep{zhang2015divide}
	for distributed KRR, and from \citep{guo2017learning} for distributed SA, with a less strict condition on the partition number $m$. Moreover, they
	 provide the first optimal rates for distributed SA in the non-attainable cases (i.e., the regression function may not be in the RKHS).
	\item As a byproduct, we provide the first results with
	 optimal, capacity-dependent rates for  the non-distributed SA in the non-attainable cases, filling a theoretical gap since \citep{smale2007learning,caponnetto2007optimal} for KRR using the integral-operator approach.
\end{itemize}

The remainder of the paper is organized as follows.  Section \ref{sec:supervised} introduces the supervised learning setting. Section \ref{sec: distributed SGM} describes distributed SGM, and then presents theoretical results on generalization error for distributed SGM, following with simple comments.  Section \ref{sec:distributed SRA} introduces distributed SA,
and then gives statistical results on generalization error. 
Section \ref{sec:discussion} discusses and compares our results with related work. Section \ref{sec:proof}  provides the proofs for distributed SGM.  Finally,  proofs for auxiliary lemmas and results for distributed SA are provided in the appendix.

\section{Supervised Learning Problems} \label{sec:supervised}
We consider a supervised learning problem.
Let $\rho$ be a probability measure on a measurable space $Z=X\times Y,$ where  $X$ is a compact-metric input space and $Y \subseteq \mR$ is the output space. $\rho$ is fixed but unknown.  Its information can be only known through a set of samples
$\Samples=\{z_i=(x_i, y_i)\}_{i=1}^N$ of $N\in\mN$ points, which we assume to be i.i.d..
We denote $\rho_X (\cdot)$ the induced marginal measure on $\HK$
of $\rho$ and $\rho(·|x)$ the conditional probability measure on $\mR$ with respect to $x \in \HK$ and $\rho$. We assume that $\rho_X$ has full support in $X$ throughout.

The quality of a predictor $f: X \to Y$ can be measured in terms of the expected risk with a square loss defined as
\be\label{generalization_error}
\mcE(f) = \int_{Z} (f(x) - y)^2 d\rho(z).
\ee
In this case, the function minimizing the expected risk over all measurable functions is the regression function given by
\be\label{regressionfunc}
\FR(x) = \int_Y y d \rho(y | x),\qquad x \in X.
\ee
The performance of an estimator $f \in \LR$ can be measured in terms of generalization error (excess risk), i.e., $\mcE(f) - \mcE(f_{\rho}).$
It is easy to prove that
\be\label{excesserror}
\mcE(f) - \mcE(f_{\rho}) = \|f- f_{\rho}\|^2_{\rho}.
\ee
Here, $\LR$ is the Hilbert space of square integral functions with respect to $\rho_X$, with its induced norm given by
$\|f\|_{\rho} = \|f\|_{\LR} = \left (\int_X |f(x)|^2 d \rho_X\right)^{1/2}$.  

Kernel methods are based on choosing the hypothesis space as a RKHS. Recall that a reproducing kernel $K$ is a symmetric function $K: X
\times X \to \mR$ such that $(K(u_i, u_j))_{i, j=1}^\ell$ is
positive semidefinite for any finite set of points
$\{u_i\}_{i=1}^\ell$ in $X$. The reproducing kernel $K$ defines a RKHS $(\HK, \|\cdot\|_{\HK})$ as the
completion of the linear span of the set $\{K_x(\cdot):=K(x,\cdot):
x\in X\}$ with respect to the inner product $\la K_x,
K_u\ra_{\HK}:=K(x,u).$

Given only the samples $\Samples$, the goal is to learn the regression function through efficient  algorithms.

\section{Distributed Learning with Stochastic Gradient Methods}\label{sec: distributed SGM}
In this section,  we first state the distributed SGM.
We then present theoretical results  for distributed SGM and non-distributed SGM, following with simple discussions. 

\subsection{Distributed SGM}
Throughout this paper, as that in \citep{zhang2015divide}, we assume that\footnote{For the general case, one can consider the weighted averaging scheme, as that in \citep{lin2017distributed}, and our analysis still applies with a simple modification.} the sample size $N=mn$ for some positive integers $n,m$, and we randomly decompose $\Samples$ as $\bz_1\cup \bz_2\cup \cdots \cup\bz_m$ with $|\bz_1|=|\bz_2|=\cdots = |\bz_m|=n$. For any $s \in [m],$  we write $\bz_s = \{(x_{s,i},y_{s,i})\}_{i=1}^n.$ 
We study distributed SGM, with mini-batches and multi-pass over the data, as detailed in Algorithm \ref{alg:sgm}.
For any $t\in \mN^{+},$ the set of the first $t$ positive integers is denoted by $[t]$.

\begin{algorithm}[htb]
	\caption{Distributed learning with stochastic gradient methods}
	\label{alg:sgm}
	\begin{algorithmic}[1]
		\Require{Number of partitions $m$, 
			mini-batch size $b \leq N/m$,  total number of iterations $T$, step-size sequence $\{\eta_t>0\}_{t=1}^T$, and kernel function $K(\cdot,\cdot)$}
		\State 
		Divide $\Samples$ evenly and uniformly at random
		into the $m$ disjoint subsets, $\bz_1,\cdots, \bz_m$.
		\State For every $s \in [m],$ compute a local estimate via $b$-minibatch SGM over the sample $\bz_s$:  $f_{s,1} = 0$ and 
			\be\label{eq:algsgm}
		f_{s,t+1}=f_{s,t} - \eta_t {1 \over b} \sum_{i= b(t-1)+1}^{bt} (f_{s,t}(x_{s,j_{s,i}}) - y_{s, j_{s,i}}) K_{x_{s, j_{s,i}}} ,
		\qquad t\in [T]. \ee
		Here, $j_{s,1},j_{s,2},\cdots,j_{s,bT}$ are i.i.d. random variables from the uniform distribution on $[n]$.\footnote{Note that the random variables $j_{s,1},\cdots, j_{s,bT}$ are conditionally independent given
			the sample $\bf z_s$.}
		\State Take the averaging over these local estimators: 
		$\bar{f}_{T} = {1 \over m} \sum_{s=1}^m f_{s,T}.$
		\Ensure  the function $ \bar{f}_T$
	\end{algorithmic}
\end{algorithm}


In the algorithm, at each iteration $t$, for each $s\in [m],$ the local estimator updates its current solution by subtracting a scaled gradient estimate. It is easy to see that the gradient estimate at each iteration for the $s$-th local estimator is an unbiased estimate of the full gradient of the empirical risk over $\bz_s.$ The global predictor is the average over these local solutions. In the special case $m=1$, the algorithm reduces to the classic multi-pass SGM.

There are several free parameters, the step-size $\eta_t$, the mini-batch size $b$, the total number of iterations/passes, and the number of partition/subsets $m$. 
All these parameters will affect the algorithm's generalization properties and computational complexity. In the coming subsection, we will show how these parameters can be chosen so that the algorithm can generalize optimally, as long as the number of subsets $m$ is not too large. Different choices on $\eta_t$, $b$, and $T$ correspond to different regularization strategies. In this paper, we are particularly interested in the cases that both $\eta_t$ and $b$ are fixed as some universal constants that may depend on the local sample size $n$, while $T$ is tuned. 

The total number of iterations $T$ can be bigger than the local sample size $n$, which means
that the algorithm can use the data more than once, or in another words, we can run the algorithm with multiple passes over the data. 
Here and in what follows, the number of (effective) `passes' over the data is referred to ${bt \over n}$
after $t$ iterations of the algorithm.

The numerical realization of the algorithm and its performance on a synthesis data can be found in \citep{lin2018optimal}.
The space and time complexities for each local estimator are
\be
O(n)\quad \mbox{and} \quad O(bnT),
\ee
respectively. The total space and time complexities of the algorithm are
\be\label{eq:timComp}
O(N)\quad \mbox{and} \quad O(bNT).
\ee

\subsection{Generalization Properties for Distributed Stochastic Gradient Methods}
\label{subsec:generalization}

In this section, we state our results for distributed SGM, following with simple discussions. Throughout this paper, we make the following assumptions.

\begin{as}\label{as:basic}
	$\HK$ is separable and $K$ is continuous. Furthermore, for some $\kappa \in [1,\infty[$, 
	\be\label{eq:HK}
	K(x,x) \leq \kappa^2, \quad \forall x \in X,
	\ee
	and for some $M,\sigma \geq 0$, 
	$$\int_Y y^2 d\rho(y|x) \leq M^2,$$
	\be\label{noiseExp}
	\int_{Y} (\FR(x) - y)^2 d\rho(y|x) \leq \sigma^2, \quad \rho_{ X}\mbox{-almost surely}.
	\ee
\end{as}

The above assumptions are quite common in statistical learning theory, see e.g., \citep{steinwart2008support,cucker2007learning}. The constant $\sigma$ from Equation \eref{noiseExp} measures the noise level of the studied problem. The condition $\int_Y y^2 d\rho(y|x) \leq M^2$ implies that the regression function is bounded almost surely,
\be\label{eq:FRbound}
|\FR(x)| \leq M. 
\ee
It is trivially satisfied when $Y$ is bounded, for example, $Y = \{-1,1\}$ in the classification problem.
To state our first result, we define an inclusion operator $\IK: \HK \to \LR$, which is continuous under Assumption \eref{eq:HK}.

\begin{corollary}\label{cor:simpCa}
	Assume that $\FR \in \HK$ and 
	$$m \leq N^{\beta},\quad 0 \leq \beta<{1\over 2}.$$
	Consider Algorithm \ref{alg:sgm} with any of the following choices on $\eta_t$, $b$ and $T$. \\
	1) $\eta_t = \eta \simeq m/\sqrt{N}$ for all $t \in [T_*],$ $b=1$, and $T_* = N/ m.$
	\\
	2) $\eta_t =\eta = \simeq {1 \over \log N}$ for all $t \in [T_*],$ $b \simeq \sqrt{N}/ m$, and $T_* \simeq \sqrt{N} \log N.$ \\
	Then,
	$$
	\mE\| \IK \bar{f}_{T_*+1} - \FR\|_{\rho}^2 \lesssim 
	N^{-1/2}  \log N.
	$$
	Here and throughout this section, we use the notations $a_1\lesssim a_2$ to mean $a_1 \leq  C a_2$ for some positive constant $C$ depending only on  $\kappa, M, \sigma, \|\IK\|, \|\FR\|_{\HK}$ , and
	$a_1 \simeq a_2$ to mean $a_2 \lesssim a_1  \lesssim a_2$.
\end{corollary}
The above result provides generalization error bounds for distributed SGM with two different choices on step-size $\eta_t$, mini-batch size $b$ and total number of iterations/passes. The convergence rate is optimal up to a logarithmic factor, in the sense that it nearly matches the minimax rate $N^{-{1/2}}$ in \citep{caponnetto2007optimal} and the convergence rate $N^{-{1/2}}$ for KRR \citep{smale2007learning,caponnetto2007optimal}.   The number of passes to achieve optimal error bounds in both cases is roughly one. The above result asserts that distributed SGM generalizes optimally after one pass over the data for two different choices on step-size and mini-batch size, provided that the partition level $m$ is not too large. In the case that $m \simeq\sqrt{N},$ according to \eref{eq:timComp}, the computational complexities  are $O(N)$ in space and  $O(N^{1.5})$ in time, comparing with $O(N)$ in space and $O(N^2)$ in time of classic SGM. 

Corollary \ref{cor:simpCa} provides statistical results for distributed SGM without considering any further benign assumptions about the learning problem, such as  the regularity of the regression function and the capacity of the RKHS.
In what follows, we will show how the results can be further improved, if we make these two benign assumptions.

The  first benign assumption relates to the regularity of the regression function.   We introduce the integer operator $\LK: \LR \to \LR$, defined by 
$\LK f = \int_{X} f(x)K(x,\cdot) d \rho_X$. Under Condition \eqref{eq:HK}, $\LK$ is  positive trace class operators \citep{cucker2007learning}, and hence $\LK^{\zeta}$ is well defined using the spectral theory.
\begin{as}\label{as:regularity}
	There exist $\zeta> 0$ and $R>0$, such that $\|\LK^{-\zeta} \FR \|_{\rho} \leq R.$
\end{as}

This assumption characterizes how large the subspace that the regression function lies in. The bigger the $\zeta$ is, the smaller the subspace is, the stronger the assumption is, and the easier the learning problem is, as $\LK^{\zeta_1}(\LR) \subseteq \LK^{\zeta_2}(\LR)$ if $\zeta_1 \geq \zeta_2.$ Moreover, if $\zeta = 0,$ we are making no assumption, and if  $\zeta ={1\over 2}$, we are requiring that there exists some $\FH \in \HK$ such that $\FH = \FR$ almost surely \cite[Section 4.5]{steinwart2008support}.

The next assumption relates to the capacity of the hypothesis space.
\begin{as}\label{as:eigenvalues}
	For some $\gamma \in [0,1]$ and $c_{\gamma}>0$, $\LK$ satisfies
	\be\label{eigenvalue_decay}
	\tr(\LK(\LK+\lambda I)^{-1}) \leq c_{\gamma} \lambda^{-\gamma}, \quad \mbox{for all } \lambda>0.
	\ee
\end{as} 
The left hand-side of  \eref{eigenvalue_decay} is called effective
dimension \citep{zhang2006learning} or degrees of freedom \citep{caponnetto2007optimal}. It is related to covering/entropy number conditions, see \citep{steinwart2008support}. 
The condition \eref{eigenvalue_decay} is naturally satisfied with $\gamma =1$, since $\LK$ is a trace class operator which implies that its eigenvalues $\{\sigma_i\}_i$ satisfy $ \sigma_i \lesssim i^{-1}.$  Moreover, if the eigenvalues of $\LK$ satisfy a polynomial decaying condition $\sigma_i \sim i^{-c}$ for some $c>1$, or if $\LK$ is of finite rank,
then the condition \eref{eigenvalue_decay} holds with $\gamma = 1/c$, or with $\gamma=0$.  The case $\gamma=1$ is refereed as the capacity independent case. A smaller $\gamma$ allows deriving faster convergence rates for the studied algorithms, as will be shown in the following results.


Making these two assumptions, we have the following general results for distributed SGM.

\begin{thm}\label{thm:main}
	Under Assumptions  \ref{as:regularity} and \ref{as:eigenvalues}, let  $\eta_t= \eta$ for all $t\in [T]$ with $\eta$ satisfying 
	\be\label{etaRestri}
	0<\eta \leq {1 \over 4 \kappa^2 \log T} .
	\ee Then for any $t \in [T]$ and $\PRegPar = n^{\theta - 1}$ with $\theta \in [0,1]$, the following results hold.\\
	1) For $\zeta\leq 1,$
	\begin{align}\label{eq:mainErr}
	\begin{split}
 	\mE\| \IK \bar{f}_{t+1} - \FR\|_{\rho}^2 
	\leq  ((\PRegPar \eta t )^2\vee Q_{\gamma, \theta, n}^{2\zeta \vee1} \vee \log t )[
	C_5{(R + {\bf 1}_{\{2\zeta <1\}} \|\FR\|_{\infty})^2 \over (\eta t)^{2\zeta}} + C_8 { \sigma^2 \over N \PRegPar^\gamma} + {C}_{10} {M^2 \eta \over mb}  ].
	\end{split}
	\end{align}
	2) For $\zeta>1,$
	\be\label{eq:mainErr2}
	\begin{split}
	 \mE\| \IK \bar{f}_{t+1} - \FR\|_{\rho}^2 
	\leq ((\PRegPar \eta t)^{2\zeta} \vee Q_{\gamma, \theta, n} \vee ({ (\eta t)^{2\zeta-1}\over n^{(\zeta-1/2)\wedge 1} }) \vee \log t )[ C_5{R^2 \over (\eta t)^{2\zeta}} + C_8 {\sigma^2 \over N \PRegPar^\gamma} + C_{10}{M^2 \eta \over mb}].
	\end{split}
	\ee
	Here,
	\be\label{eq:qthetan}
	Q_{\gamma, \theta, n} = 1 \vee [\gamma(\theta^{-1} \wedge \log n)]
	\ee 
and	$C_5$, $C_6$, $C_8,$ $C_{10}$ are positive constants depending only on $\kappa^2,\zeta, c_{\gamma}, \|\LK\|$ which will be given explicitly in the proof, see \eqref{eq:const5}, \eqref{eq:const6}, \eqref{eq:const8} and \eqref{eq:const10}. 
\end{thm}

In the above result, we only consider the setting of a fixed step-size. Results with a decaying step-size can be directly derived following our proofs in the coming sections, combining with some basic estimates from \citep{lin2017optimal}.  The error bound from  \eref{eq:mainErr} depends on the number of iteration $t$, the step-size $\eta$, the mini-batch size,
the number of sample points $N$ and the partition level $m$. It holds for any pseudo regularization parameter $\PRegPar$ where $ \PRegPar \in [n^{-1},1]$. When $ t \leq n/\eta$, for $\zeta \leq 1,$ we can choose $\PRegPar= (\eta t)^{-1}$, and ignoring the logarithmic factor and constants, \eref{eq:mainErr} reads as 
\begin{align}\label{eq:mainErrA}
\mE\| \IK \bar{f}_{t+1} - \FR\|_{\rho}^2 \lesssim  {1 \over (\eta t)^{2\zeta}} + {(\eta t)^{\gamma} \over N} + {\eta \over mb} .
\end{align}
The right-hand side of the above inequality is composed of three terms. The first term is related to the regularity parameter $\zeta$ of the  regression function $\FR$, and it results from estimating bias. The second term depends on the sample size $N,$ and it results from estimating sample variance. The last term results from estimating computational variance due to random choices of the sample points. In comparing with the error bounds derived for classic SGM performed on a local machine, one  can see that averaging over the local solutions can reduce sample and computational variances, but keeps bias unchanged. 
As the number of iteration $t$  increases, the bias term decreases, and the sample variance term increases. This is a so-called trade-off problem in statistical learning theory. Solving this trade-off problem leads to the best choice on number of iterations. Notice that the computational variance term is independent of the number of iterations $t$ and it depends on the step-size, the mini-batch size, and the partition level. To derive optimal rates, it is necessary to choose a small step-size, and/or a large mini-batch size, and a suitable partition level. 
In what follows, we provide different choices of these algorithmic parameters, corresponding to different regularization strategies, while leading to the same optimal convergence rates up to a logarithmic factor.

\begin{corollary}\label{cor:fastRatP}
	Under Assumptions  \ref{as:regularity} and \ref{as:eigenvalues}, let $\zeta \leq 1$, $2\zeta+\gamma> 1$ and 
	\be\label{eq:partNum}
	m \leq  N^{\beta}, \quad \mbox{with } 0\leq \beta < {2\zeta+\gamma - 1 \over 2\zeta+\gamma}.
	\ee
	Consider Algorithm \ref{alg:sgm} with any of the following choices on $\eta_t$, $b$ and $T_*$. \\
	1) $\eta_t \simeq n^{-1}$ for all $t \in [T_*]$, $b=1$, and $T_* \simeq N^{1 \over 2\zeta+\gamma}n .$\\
	2) $\eta_t \simeq n^{-1/2}$ for all $t \in [T_*]$, $b \simeq  \sqrt{n}$, and $T_* \simeq N^{1 \over 2\zeta+\gamma}\sqrt{n} .$ \\
	3) $\eta_t \simeq N^{-{2 \zeta \over 2\zeta+\gamma}} m$ for all $t \in [T_*],$ $b=1$, and $T_* \simeq N^{2\zeta + 1 \over 2\zeta +\gamma}/ m.$
	\\
	4) $\eta_t \simeq {1 \over \log N}$ for all $t \in [T_*],$ $b \simeq N^{2\zeta \over 2\zeta +\gamma}/ m$, and $T_* \simeq N^{ 1 \over 2\zeta +\gamma} \log N.$ \\
	Then,
	$$
	\mE\| \IK \bar{f}_{T_*+1} - \FR\|_{\rho}^2 \lesssim 
	N^{-{2\zeta \over 2\zeta +\gamma}}  \log N.
	$$
\end{corollary} 
We add some comments on the above theorem. First,
the convergence rate is optimal up to a logarithmic factor, as it is almost the same as that for KRR from \citep{caponnetto2007optimal,smale2007learning} and also it nearly matches the minimax lower rate $O(N^{-{2\zeta \over 2\zeta+\gamma}})$ in \citep{caponnetto2007optimal}. In fact, let $\mathcal{P}(\gamma,\zeta)$
($\gamma \in (0,1)$ and $\zeta \in [1/2,1]$) be the set of probability measure $\rho$ on $Z,$ such that Assumptions
 \ref{as:basic}-\ref{as:eigenvalues} are satisfied. Then the following minimax lower rate  is a direct consequence of  \cite[Theorem 2]{caponnetto2007optimal}:
$$
\liminf_{N \to \infty} \inf_{f^N}  \sup_{\rho \in \mathcal{P}(\gamma,\zeta)} \Pr
\left( \Samples \in Z^N :	\mE\| \IK {f}^{N} - \FR\|_{\rho}^2 > C N^{-2\zeta \over 2\zeta+\gamma}
\right) = 1,
$$
for some constant $C > 0$ independent on $N$, where the infimum in the middle is taken over all algorithms as a map $ Z^N \ni \Samples \mapsto f^N \in \HK$. 
Alternative minimax lower rates (perhaps considering other quantities, $R$ and $\sigma^2$) could be found in
\cite[Theorem 3]{caponnetto2007optimal} and \cite[Theorem 3.5]{blanchard2016optimal}.
Second, distributed SGM saturates when $\zeta >1.$ The reason for this is that averaging over local solutions can only reduce sample and computational variances, not bias. Similar saturation phenomenon is also observed when analyzing distributed KRR in \citep{zhang2015divide,lin2017distributed}. 
Third, the condition $2\zeta+\gamma > 1$ is equivalent to assuming that the learning problem can not be too difficult. We believe that such a condition is necessary for applying distributed learning technique to reduce computational costs, as there are no means to reduce computational costs if the learning problem itself is not easy.
Fourth, as the learning problem becomes easier (corresponds to a bigger $\zeta$), the faster the convergence rate is, and moreover
the larger the number of partition $m$ can be. 
Finally, different parameter choices leads to different regularization strategies. In the first two regimes, the step-size and the mini-batch size are fixed as some prior constants (which only depends on $n$), while the number of iterations depends on some unknown distribution parameters. In this case, the regularization parameter is the number of iterations, which in practice can be tuned by using cross-validation methods.  Besides, the step-size and the number of iterations in the third regime, or the mini-batch size and the number of iterations in the last regime, depend on the unknown distribution parameters, and they have some regularization effects.  The above theorem asserts that distributed SGM with differently suitable choices of parameters can generalize optimally, provided the partition level $m$ is not too large.

\subsection{Optimal Rate for Multi-pass SGM on a Single Dataset}
As a direct corollary of Theorem \ref{thm:main}, we derive the following results for classic multi-pass SGM. 
\begin{corollary}\label{thm:sgm}
	Under Assumptions  \ref{as:regularity} and \ref{as:eigenvalues}, 
	consider Algorithm \ref{alg:sgm} with $m=1$ and any of the following choices on $\eta_t$, $b$ and $T_*$. 
	\\
	1) $\eta_t \simeq N^{-1}$ for all $t \in [T_*]$, $b=1$, and $T_* \simeq N^{\alpha + 1}.$\\
	2) $\eta_t \simeq N^{-1/2}$ for all $t \in [T_*]$, $b \simeq  \sqrt{N}$, and $T_* \simeq N^{\alpha +1/2} .$ \\
	3) $\eta_t \simeq N^{-{2 \zeta \alpha}}$ for all $t \in [T_*],$ $b=1$, and $T_* \simeq N^{\alpha (2\zeta+1)}.$
	\\
	4) $\eta_t \simeq {1 \over \log N}$ for all $t \in [T_*],$ $b \simeq N^{2\zeta \alpha}$, and $T_* \simeq N^{ \alpha} \log N.$ \\
	Here,  $$\alpha ={1 \over  (2\zeta+\gamma)\vee 1}.$$
	Then,
	\be
	\mE\| \IK \bar{f}_{T_*+1} - \FR\|_{\rho}^2 \lesssim  \begin{cases}
		N^{-{2\zeta \over 2\zeta + \gamma}} \log N,& \quad \mbox{if } 2\zeta + \gamma>1;\\ 
		N^{-{2\zeta}}  \log N,& \quad \mbox{otherwise}. 
	\end{cases}
	\ee
\end{corollary}
The above results provide generalization error bounds for multi-pass SGM trained on a single dataset.
The derived convergence rate is optimal in the minimax sense \citep{caponnetto2007optimal,blanchard2016optimal} up to a logarithmic factor. Note that SGM  does not have a saturation effect, and  optimal convergence rates can be derived for  any $\zeta \in ]0,\infty].$
Corollary \ref{thm:sgm} improves the result in \citep{lin2017optimal} in two aspects. First, the convergence rates
are better than those  (i.e., 
$	O(N^{-{2\zeta \over 2\zeta + \gamma}} \log N)$ if $2\zeta +\gamma \geq 1$ or 
$O(N^{-{2\zeta}} \log^4 N)$ otherwise) from \citep{lin2017optimal}.
Second, the above theorem does not require the extra condition $m \geq m_{\delta}$ made in \citep{lin2017optimal}.

\section{Distributed Learning with  Spectral Algorithms}\label{sec:distributed SRA}
In this section, we first state distributed SA. We then present  theoretical results for distributed SA, following with simple discussions.
Finally, we give convergence results for classic SA.  

\subsection{Distributed Spectral Algorithms}

In this subsection, we present distributed SA. 
We first recall that a filter function is defined as follows.

\begin{definition}[Filter functions]
	Let $\Lambda$ be a subset of $\mR_+.$
	A class of functions $\{\GLB: [0, \kappa^2] \to [0,\infty[, \lambda \in \Lambda \}$ is said to be filter functions with qualification $\tau $ ($\tau\geq 0$) if there exist some positive constants $E,F_{\tau}<\infty$ such that
	\be
	\label{eq:GLproper1}
	\sup_{\alpha\in [0,1]} \sup_{\lambda\in \Lambda}\sup_{u \in ]0,\kappa^2] } |u^{\alpha}\GLB(u)|\lambda^{1-\alpha} \leq E ,
	\ee
	and 
	\be\label{eq:GLproper4}
	\sup_{\alpha\in [0, \tau]} \sup_{\lambda \in \Lambda}	\sup_{u\in]0, \kappa^2]} |(1 - \GLB(u)u)|u^{\alpha}\lambda^{-\alpha} \leq F_{\tau} .\ee
\end{definition}

\begin{algorithm}
	\caption{Distributed learning with spectral algorithms}
	\label{alg:dSpe}
	\begin{algorithmic}[1]
	\Require{Number of partitions $m$, filter function $\GLB$
		, and kernel function $K(\cdot,\cdot)$}
	\State 
	Divide $\Samples$ evenly and uniformly at random
	into $m$ disjoint subsets, $\bz_1, \bz_2, \cdots, \bz_m$
	\State For every $s \in [m],$ compute a local estimate via SA over the samples $\bz_s$:  
	\footnote{Let $L$ be a self-adjoint, compact operator over a separable Hilbert space. $\GLB(L)$ is an operator on $L$ defined by spectral calculus: suppose that $\{(\sigma_i, \psi_i)\}_i$ is a set of
				normalized eigenpairs of $L$ with the eigenfunctions $\{\psi_i\}_i$ forming an orthonormal basis
				of $\HK$, then $\GLB(\TXS) = \sum_i \GLB(\sigma_i) \psi_i \otimes \psi_i.$}
			\bea
			\LESRA = \GLB(\TXS) {1\over n}\sum_{i=1}^n y_{s,i} K_{s,i}, \quad \TXS = {1 \over n} \sum_{i=1}^{n} \la \cdot, K_{x_{s,i}} \ra K_{x_{s,i}}
			\eea
	\State Take the averaging over these local estimators: 
	$
		\EDSRA = {1 \over m}\sum_{s=1}^m \LESRA
		$
	\Ensure  the function $ \EDSRA$
\end{algorithmic}
\end{algorithm}

%
In the algorithm, $\lambda$ is a regularization parameter which should be appropriately chosen in order to achieve best performance. In practice, it can be tuned by using the cross-validation methods.
SA is associated with some given filter functions.  Different filter functions correspond to different regularization algorithms. The following examples provide several common filter functions, which leads to different types of regularization methods, see e.g. \citep{gerfo2008spectral,bauer2007regularization}.

\begin{Exa}[KRR]
	The choice $\GLB(u) = (u+\lambda)^{-1}$ corresponds to Tikhonov regularization or the
	regularized least squares algorithm.  It is easy to see that
	$\{\GL(u): \lambda\in \mR_+\}$ is a class of  filter functions with qualification $\tau=1$, and 
	$E = F  = 1$. 
\end{Exa}

\begin{Exa}[GM]
	Let $\{\eta_k>0\}_k$ be such that $\eta_k \kappa^2 \leq 1$ for all $k \in \mN.$
	Then as will be shown in Section \ref{sec:proof},
	$$\GLB(u) = \sum_{k=1}^t \eta_k \prod_{i=k+1}^t (1 - \eta_i u)$$ where we identify $\lambda = (\sum_{k=1}^t \eta_k)^{-1},$
	corresponds to gradient methods  or Landweber iteration algorithm. 
	The qualification $\tau$ could be any positive number, 
	$E =1 ,$ and $F_\tau  = (\tau/\mathrm{e})^{\tau}$. 
\end{Exa}


\begin{Exa}[Spectral cut-off]
	Consider the spectral cut-off or truncated singular value decomposition (TSVD) defined by
	$$
	\GLB(u) = 
	\begin{cases}
	u^{-1}, & \mbox{if } u \geq \lambda,\\
	0, & \mbox{if } u < \lambda.
	\end{cases}
	$$
	Then
	the qualification $\tau$ could be any positive number and $E = F_{\tau} = 1$. 
\end{Exa}

\begin{Exa}[KRR with bias correction] \label{exa:KRRbc}
	The function $\GLB(u) = \lambda(\lambda+x)^{-2} + (\lambda+x)^{-1}$
	corresponds to KRR with bias correction. It is easy to show that the qualification $\tau=2$, $E = 2$ and $F_{\tau}=1.$
\end{Exa}

\footnotetext{Let $L$ be a self-adjoint, compact operator over a separable Hilbert space. $\GLB(L)$ is an operator on $L$ defined by spectral calculus: suppose that $\{(\sigma_i, \psi_i)\}_i$ is a set of
	normalized eigenpairs of $L$ with the eigenfunctions $\{\psi_i\}_i$ forming an orthonormal basis
	of $\HK$, then $\GLB(\TXS) = \sum_i \GLB(\sigma_i) \psi_i \otimes \psi_i.$}

The implementation of the algorithms is very standard using the representation theorem, for which we thus skip the details.

\subsection{Optimal Convergence for Distributed Spectral Algorithms}

We have the following general results for distributed  SA. 
\begin{thm}\label{thm:mainDFliter}
	Under Assumptions  \ref{as:regularity} and \ref{as:eigenvalues},  
	let $\GLB$ be a filter function with   qualification $\tau \geq  (\zeta\vee 1)$, and $\EDSRA$  be given by Algorithm \ref{alg:dSpe}. Then
	for any $\PRegPar = n^{\theta - 1}$ with $\theta \in [0,1]$, the following results hold.\\
	1) For $\zeta\leq 1,$
	\begin{align}
	\mE\| \IK \EDSRA - \FR\|_{\rho}^2 \leq ( Q_{\gamma,\theta,n}^{2\zeta \vee 1} \vee {\PRegPar^2 \over \lambda^2} )[ C_5' (R+{\bf 1}_{\{2\zeta<1\}} \|\FR\|_{\infty})^2{\lambda^{2\zeta}} + C_8'{\sigma^2 \over N \PRegPar^\gamma} ].
	\end{align}
	2) For $\zeta>1,$
	\begin{align}
	\mE\| \IK \EDSRA - \FR\|_{\rho}^2 \leq ( {\lambda^{1-2\zeta} \over n^{(\zeta-1/2)\wedge 1} } \vee Q_{\gamma,\theta,n} \vee {\PRegPar^{2\zeta} \over \lambda^{2\zeta}} )[ C_6' R^2{\lambda^{2\zeta}} + C_8'{\sigma^2 \over N \PRegPar^\gamma} ].
	\end{align}
	Here, $Q_{\gamma,\theta,n}$ is given by \eqref{eq:qthetan}, and $C_5',$ $C_6'$ and $C_8'$ are positive constants depending only $\kappa,\zeta,E, F_{\tau},$ $c_{\gamma}$ and  $\|\LK\|$. 
\end{thm}

The above results provide generalization error bounds for distributed SA. The upper bound depends on the number of partition $m$, the regularization parameter $\lambda$ and total sample size $N$.  When  the regularization parameter $\lambda> 1/n,$ by setting $\PRegPar = \lambda$,
the derived error bounds for $\zeta\leq 1$ can be simplified as  
\begin{align*}
\mE\| \IK \EDSRA - \FR\|_{\rho}^2 \lesssim {\lambda^{2\zeta}} + {1 \over N \lambda^\gamma} .
\end{align*}
There are two terms in the upper bound. They are raised from estimating bias and sample variance.
Note that there is a trade-off between the bias term and the sample variance term. Solving this trade-off leads to the best choice on regularization parameter.
Note also that similar to that for distributed SGM, distributed SA also saturates when $\zeta >1.$

\begin{corollary}\label{cor:DSRA}
	Under the assumptions  of Theorem \ref{thm:mainDFliter},  let $2\zeta+\gamma>1$, $\lambda= N^{-{1 \over 2\zeta+\gamma}}$
	and  the number of partitions satisfies \eref{eq:partNum}.
	Then 
	\begin{align}
	\mE\| \IK \EDSRA - \FR\|_{\rho}^2 \lesssim N^{-{2\zeta \over 2\zeta+\gamma}}.
	\end{align}
\end{corollary}
The convergence rate from the above corollary is optimal as it matches exactly the minimax rate in \citep{caponnetto2007optimal}, and it is better than the rate for distributed SGM from Theorem \ref{thm:main}, where the latter has an extra logarithmic factor.
According to Corollary \ref{cor:DSRA}, distributed SA with an appropriate choice of regularization parameter $\lambda$ can generalize optimally, if the number of partitions is not too large. 
To the best of our knowledge, the above corollary is the first optimal statistical result for distributed SA considering the non-attainable case (i.e. $\zeta$ can be less than $ 1/2$). Moreover, the requirement on the number of partitions $m <N^{2\zeta+\gamma - 1\over 2\zeta+\gamma}$ to achieve optimal generalization error bounds
is much weaker than that ($m\leq  N^{2\zeta-1 \over 2\zeta+\gamma}$) in \citep{guo2017learning,blanchard2016parallelizing}. 

\subsection{Optimal Rates for Spectral Algorithms on a Single Dataset}

The following results provide  generalization error bounds for classic SA.

\begin{corollary}\label{thm:Fliter}
	Under Assumptions \ref{as:regularity} and \ref{as:eigenvalues},  let $\GLB$ be a filter function with   qualification $\tau \geq (\zeta \vee 1)$, and $\FLESRA$ be given by Algorithm \ref{alg:dSpe} with $\lambda= N^{-{1 \over 1 \vee (2\zeta+\gamma)}}$ and $m=1$. 
	Then 
	\be
	\mE\| \IK \FLESRA- \FR\|_{\rho}^2 \lesssim  \begin{cases}
		N^{-{2\zeta \over 2\zeta + \gamma}},& \quad \mbox{if } 2\zeta + \gamma>1;\\ 
		N^{-{2\zeta}}(1\vee  \log N^{\gamma}),& \quad \mbox{otherwise}. 
	\end{cases}
	\ee
	Here, $a_1\lesssim a_2$ means $a_1 \leq  C a_2$ for some positive constant $C$ which is depending only on  $\kappa, c_{\gamma}, \zeta, M, \sigma, \|\LK\|, E$, and  $F_{\tau}$.
\end{corollary}
The above results assert that SA generalizes optimally if the regularization parameter is well chosen. To the best of our knowledge, the derived result is the first one with optimally capacity-dependent  rates in the non-attainable case for a general SA.
Note that unlike distributed SA, classic SA does not have a saturation effect.

\section{Discussion}\label{sec:discussion}

In this section, we briefly review some of the related results in order to facilitate comparisons.   For ease of comparisons, we summarize some of the results and their computational costs in Table \ref{table:1}.

We first briefly review convergence results on generalization error for KRR, and more generally, SA. 
Statistical results for KRR with different convergence rates have been shown in, e.g., \citep{smale2007learning,caponnetto2007optimal,wu2006learning,steinwart2008support,steinwart2009optimal}. 
Particularly,  \cite{smale2007learning} proved convergence rates of order $O(N^{-{2\zeta \over 1+(2\zeta\vee 1)}})$ with $0<\zeta\leq 1$, without considering the capacity assumption. 
\cite{caponnetto2007optimal} gave optimally capacity-dependent convergence rate of order $O(N^{-{2\zeta \over 2\zeta+\gamma}})$ but only for the case that $1/2 \leq \zeta\leq 1$.  
The above two are based on integral operator approaches.  Using an alternative argument related to covering-number or entropy-numbers, \cite{wu2006learning} provided convergence rate $O(n^{-{2\zeta \over 1 + \gamma}})$, and \cite[Theorem 7.23]{steinwart2008support} providesls convergence rate $O(n^{-{2\zeta \over (2\zeta + \gamma)\vee 1}})$, assuming that $0<\zeta \leq 1/2,$ $\gamma \in (0,1)$ and $|y| \lesssim 1$ almost surely. 
For GM, \cite{yao2007early} derived convergence rate of order $O(N^{-{2\zeta \over 2 \zeta+2}})$ (for $\zeta \in]0,\infty[$), without considering the capacity assumption.  Involving the capacity assumption, \cite{lin2017optimal} derived convergence rate  of order $O(N^{-2\zeta \over 2\zeta +\gamma} \log^2 N)$ if $2\zeta+\gamma > 1$, or $O(N^{-2\zeta} \log^4 N)$ if $2\zeta +\gamma \leq 1$. Note that both proofs from \citep{yao2007early,lin2017optimal} rely on the special separable properties of GM with the square loss.
%
For SA, statistical results on generalization error with different convergence rates have been shown in, e.g., \citep{caponnetto2006,bauer2007regularization,blanchard2016optimal,dicker2017kernel, lin2017distributed}. 
The best convergence rate shown so far (without making any extra unlabeled data as that in \citep{caponnetto2006}) is $O(N^{-{2\zeta \over 2\zeta+\gamma}})$ \citep{blanchard2016optimal,dicker2017kernel, lin2017distributed} but only for the attainable case, i.e., $\zeta\geq 1/2$. These results also apply to GM, as GM can be viewed as a special instance of SA.
Note that some of these results also require the extra assumption that the sample size $N$ is large enough. 
In comparisons, Corollary \ref{thm:Fliter} provides the best convergence rates for SA, considering both the non-attainable and attainable cases and without making any extra assumption. 
Note that our derived error bounds are in expectation, but it is not difficult to derive error bounds in  high probability using our approach, and we will report this result in a future work.

%
%
%

\begin{table}
	\newcommand{\tabincell}[2]{\begin{tabular}{@{}#1@{}}#2\end{tabular}}
	\centering
	\resizebox{\textwidth}{!}{
		\begin{tabular}{ | c | c | c | c | c | c |}
			\hline
			Algorithm & Ass. & $\#$ Processors $m$ & Rate &  \tabincell{c}{Local Memory\\ \& Time} &
			\tabincell{c}{Memory\\ \& Time} \\
			\hline
			KRR \citep{smale2007learning}	& $\zeta\in ]0,1]$, $\gamma=1$ &  1 &  $N^{-{2\zeta \over (2\zeta\vee 1)+1}}$ & $\times$ & $N^2$ \& $N^3$ \\
			\hline
			KRR \citep{caponnetto2007optimal}	& $\zeta\in [{1\over 2},1]$, $\gamma \in ]0,1]$, $N\geq N_{\delta}$ &  1 &  $N^{-{2\zeta \over 2\zeta+\gamma}}$ & $\times$ & - \\
			\hline
			KRR \citep{steinwart2008support}\tablefootnote{The results from \citep{steinwart2008support} are based on entropy-numbers arguments while the other results summarized for KRR in the table are based on integral-operator arguments.}	& $\zeta\in [0, {1\over 2}]$, $\gamma \in ]0,1[$, $|y| \lesssim 1$ &  1 &  $N^{-{(2\zeta \over 2\zeta+\gamma) \vee 1}}$ & $\times$ & - \\
			\hline
			\tabincell{c}{\it KRR  [{Corollary \ref{thm:Fliter}}]}	& $\boldsymbol{\zeta\in ]0,1],2\zeta+\gamma > 1}$ &  1 &  $\boldsymbol{N^{-{2\zeta \over 2\zeta+\gamma}}}$ & $\times$ & - \\
			\hline
			\tabincell{c}{\it KRR [{Corollary \ref{thm:Fliter}}]}	&  $\boldsymbol{2\zeta+\gamma \leq 1}$ &  1 &  $\boldsymbol{N^{-{2\zeta}} \log N^{\gamma}}$ & $\times$ & - \\
			\hline \\
			\hline
			GM \citep{yao2007early}	& $\gamma = 1$ &  1 &  $N^{-{2\zeta \over 2\zeta+2}}$ & $\times$ & $N$ \& $N^2 N^{1\over 2\zeta+2}$ \\
			\hline
			GM \citep{dicker2017kernel}	& $\zeta\in [{1\over 2},\infty[$, $\gamma \in]0,1]$, $N \geq N_0$ &  1 &  $N^{-{2\zeta \over 2\zeta+\gamma}}$ & $\times$ & $N$ \& $N^2 N^{1\over 2\zeta+\gamma}$ \\
			\hline
			GM \citep{lin2017optimal}	&  $2\zeta+\gamma > 1$, $N \geq N_{\delta}$ &  1 &  $N^{-{2\zeta \over 2\zeta+\gamma}} \log^2N$ & $\times$ & $N$ \& $N^2 N^{1\over 2\zeta+\gamma}$ \\
			\hline
			GM \citep{lin2017optimal}	& $2\zeta+\gamma \leq 1$, $N \geq N_{\delta}$ &  1 &  $N^{-{2\zeta}} \log^4N$ & $\times$ & $N$ \& $N^3 $ \\
			\hline
			\tabincell{c}{\it  GM [{Corollary \ref{thm:Fliter}}]} 	&  \boldsymbol{ $2\zeta+\gamma > 1$ } &  1 &   \boldsymbol{$N^{-{2\zeta \over 2\zeta+\gamma}}$} & $\times$ & $N$ \& $N^2 N^{1\over 2\zeta+\gamma}$ \\
			\hline
			\tabincell{c}{\it GM [{Corollary \ref{thm:Fliter}}]} 	& $ \boldsymbol{2\zeta+\gamma \leq 1}$ &  1 &  $ \boldsymbol{N^{-{2\zeta}}\log N^{\gamma}}$ & $\times$ & $N$ \& $N^3 $ \\
			\hline \\
			\hline
			\tabincell{c}{SA \citep{guo2017learning}} 	&  $\zeta \in [{1\over 2}, \tau],\gamma \in ]0,1]$ &  1 &   \boldsymbol{$N^{-{2\zeta \over 2\zeta+\gamma}}$} & $\times$ & $-$\\
			\hline
			\tabincell{c}{\it SA [{Corollary \ref{thm:Fliter}}]} 	&  $\zeta \leq \tau,$  \boldsymbol{ $2\zeta+\gamma > 1$ } &  1 &   \boldsymbol{$N^{-{2\zeta \over 2\zeta+\gamma}}$} & $\times$ & $-$ \\
			\hline
			\tabincell{c}{\it SA [{Corollary \ref{thm:Fliter}}]} 	& $ \zeta \leq \tau, \boldsymbol{2\zeta+\gamma \leq 1}$ &  1 &  $ \boldsymbol{N^{-{2\zeta}}\log N^{\gamma}}$ & $\times$ & $-$ \\
			\hline \\
			\hline
			OL \citep{ying2008online}	& $\gamma =1$ &  1 &  $N^{-{2\zeta \over 2\zeta+1}} \log N$ & $\times$ & $N$ \& $N^2$ \\
			\hline
			AveOL \citep{dieuleveut2016nonparametric}	&  $\zeta\in ]0,1]$, $2\zeta+\gamma >1$ &  1&  $N^{-{2\zeta \over 2\zeta+\gamma}}$ & $\times$ & $N$ \& $N^2$ \\
			\hline
			AveOL \citep{dieuleveut2016nonparametric}	&  $\boldsymbol{2\zeta+\gamma \leq 1}$ &  1&  $\boldsymbol{N^{-{2\zeta}}}$ & $\times$ & $N$ \& $N^2$ \\
			\hline
			SGM	\citep{lin2017optimal} &  $2\zeta+\gamma> 1, N\geq N_{\delta}$ &  1 &  $N^{-{2\zeta \over 2\zeta+\gamma}} \log^2N $ & $\times$ & $N$ \& $N^2 N^{1-\gamma \over 2\zeta+\gamma}$ \\
			\hline
			SGM	\citep{lin2017optimal} &  $2\zeta+\gamma\leq 1, N\geq N_{\delta}$ &  1 &  $N^{-2\zeta} \log^4N $ & $\times$ & $N$ \& $ N^{3-\gamma}$ \\
			\hline
		\it	SGM	[{Corollary \ref{thm:sgm}}] &  $\boldsymbol{2\zeta+\gamma> 1}$ &  1 &  $ \boldsymbol{N^{-{2\zeta \over 2\zeta+\gamma}} }$ & $\times$ & $N$ \& $N^2 N^{1-\gamma \over 2\zeta+\gamma}$ \\
			\hline
		\it	SGM	[{Corollary \ref{thm:sgm}}]  &  $\boldsymbol{ 2\zeta+\gamma\leq 1}$ &  1 &  $\boldsymbol{N^{-2\zeta} \log N^{\gamma} }$ & $\times$ & $N$ \& $ N^{3-\gamma}$ \\
			\hline \\
			\hline
			NyKRR \citep{rudi2015less}	&  $\zeta\in [{1\over 2},1]$, $\gamma \in ]0,1]$, $N\geq N_{\delta}$ &  1&  $N^{-{2\zeta  \over 2\zeta+\gamma}}$ & $\times$ & $N^{{2\zeta + \gamma+1 \over 2\zeta+\gamma}}$ \& $N^{2\zeta+2+\gamma \over 2\zeta+\gamma}$ \\
			\hline
			NySGM \citep{lin2017NySGM}	&  $\zeta\in [{1\over 2},1]$, $\gamma \in ]0,1]$, $N\geq N_{\delta}$ &  1&  $N^{-{2\zeta \over 2\zeta+\gamma}}$ & $\times$ & $N^{{2\over 2\zeta+\gamma}\vee 1}$ \& $N^{2\zeta+2 \over 2\zeta+\gamma}$ \\
			\hline
			DKRR  \& DSA \citep{guo2017learning}	&  $\zeta\in [{1\over 2},1]$, $\gamma \in ]0,1]$ &  $N^{2\zeta - 1 \over 2\zeta + \gamma}$&  $N^{-{2\zeta \over 2\zeta+\gamma}}$ & $N^{2(1+\gamma) \over 2\zeta+\gamma}$ \& $N^{3(1+\gamma) \over 2\zeta+\gamma}$ & $N^{{2 \zeta+2\gamma+1 \over 2\zeta+\gamma} }$ \& $N^{2\zeta+2+3\gamma \over 2\zeta+\gamma}$ \\
			\hline
			\tabincell{c}{
			\it	DKRR \& DSA 
				[{Corollary \ref{cor:DSRA}}]}	&  $\boldsymbol{\zeta \in ]0,1]},\boldsymbol{ 2\zeta+\gamma> 1}$ &  $N^{2\zeta+\gamma - 1 \over 2\zeta + \gamma}$&  $\boldsymbol{N^{-{2\zeta \over 2\zeta+\gamma}}}$ & $N^{2 \over 2\zeta+\gamma}$ \& $N^{3 \over 2\zeta+\gamma}$ & $N^{{2 \zeta+\gamma+1 \over 2\zeta+\gamma} }$ \& $N^{2\zeta+2+\gamma \over 2\zeta+\gamma}$ \\
			\hline
			\tabincell{c}{
			\it	DSGM 
				[{Corollary \ref{cor:fastRatP}.(3)}]}	&   $\boldsymbol{\zeta \in ]0,1], 2\zeta+\gamma > 1}$ &  $N^{2\zeta +\gamma - 1 \over 2\zeta + \gamma}$&  $\boldsymbol{N^{-{2\zeta \over 2\zeta+\gamma}}}$ & $N^{1 \over 2\zeta+\gamma}$ \& ${N^{2 \over 2\zeta+\gamma}}$ & $\boldsymbol{N}$ \& $\boldsymbol{N^{2\zeta+\gamma+1 \over 2\zeta+\gamma}}$ \\
			\hline
	\end{tabular}}
	\caption{}\label{table:1}
	{\it Summary of assumptions and results for distributed SGM (DSGM) and related approaches including KRR,  GM,  SA, one-pass SGM (OL),
		one-pass SGM with averaging (AveOL), SGM, Nystr\"{o}m KRR (NyKRR), Nystr\"{o}m SGM (NySGM),
		distributed KRR (DKRR) and distributed SA (DSA). }
\end{table}

We next briefly review convergence results for SGM. SGM \citep{robbins1951stochastic} has been widely used in convex optimization and machine learning, see e.g. \citep{cesa2004generalization,nemirovski2009robust,bottou2017optimization} and references therein . 
In what follows, we will briefly recall some recent works on generalization error for nonparametric regression on a RKHS considering the square loss. 
We will use the term ``online learning algorithm" (OL)  to mean one-pass SGM, i.e, SGM that each sample can be used only once.   
Different variants of OL, either with or without regularization, have been studied. Most of them take the form 
$$
f_{t+1} = (1- \RegParGD)f_t - \eta_t (f_t(x_{t}) - y_{t}) K_{x_{t}}, t=1\cdots, N. 
$$
Here, the regularization parameter $\RegParGD$ could be zero \citep{zhang2004solving,ying2008online}, or a positive \citep{smale2006online,ying2008online} and possibly time-varying constant \citep{tarres2014online}.  Particularly, \cite{tarres2014online} studied OL with time-varying regularization parameters and convergence rate of order $O(N^{{-2\zeta \over 2\zeta+1}})$ ($\zeta\in [{1\over 2}, 1]$) in high probability was proved. \cite{ying2008online} studied  OL without regularization and convergence rate of order $O(N^{-{2\zeta \over 2\zeta+1}})$ in expectation was shown. Both convergence rates from \citep{ying2008online,tarres2014online} are capacity-independently optimal and they do not take the capacity assumption into account. Considering an averaging step \citep{polyak1992acceleration} and a proof technique motivated by \citep{bach2013non},  \cite{dieuleveut2016nonparametric} proved capacity-dependently optimal rate $O(N^{-{2\zeta \over (2\zeta+\gamma)\vee 1}})$ for OL in the case that $\zeta\leq 1.$  Recently, \cite{lin2017optimal}  studied (multi-pass) SGM, i.e, Algorithm \ref{alg:sgm} with $m=1$. They showed that  SGM with suitable parameter choices, achieves convergence rate of order $O(N^{-{2\alpha \over (2\alpha +\gamma)\vee 1}}\log^{\beta}N)$ with $\beta=2$ when $2\alpha+\gamma>1$ or $\beta=4$ otherwise, 
after some number of iterations. 
In comparisons, the derived results for SGM in Corollary \ref{thm:sgm} are better than  those from \citep{lin2017optimal}, and the convergence rates are the same as those from \citep{dieuleveut2016nonparametric} for averaging OL when $\zeta\leq 1$ and $2\zeta+\gamma > 1$. For the case $2\zeta+\gamma\leq 1,$ the convergence rate  $O(N^{-2\zeta} (1 \vee \log N^{\gamma}))$ for SGM in Corollary \ref{thm:sgm} is worser than $O(N^{-2\zeta})$ in \citep{dieuleveut2016nonparametric} for averaging OL. However, averaging OL saturates for $\zeta > 1$, while SGM does not.


To meet the challenge of large-scale learning, 
a line of research focus on designing learning algorithms with Nystr\"{o}m subsampling, or more generally sketching.  Interestingly, the latter  has also been applied to compressed sensing, low rank matrix recovery and kernel methods, see e.g. \citep{candes2006robust,yurtsever2017sketchy,yang2012nystrom} and references therein.
The basic idea of Nystr\"{o}m subsampling is to replace a standard large matrix with a smaller matrix obtained by subsampling \citep{smola2000sparse,williams2000using}. 
For kernel methods, Nystr\"{o}m subsampling has been successfully combined with KRR \citep{el2014fast,rudi2015less,yang2015randomized}  and SGM \citep{lu2016large,lin2017NySGM}. Generalization error bounds of order $O(N^{-2\zeta \over 2\zeta+\gamma})$ \citep{rudi2015less,lin2017NySGM}
were derived, provided that the subsampling level is suitably chosen, considering the case $\zeta \in [{1 \over 2}, 1].$ 
Computational advantages of these algorithms were highlighted. Here, we summarize their convergence rates and computational costs in Table \ref{table:1}, from which we see that distributed SGM has advantages on both memory and time.

Another line of research for large-scale learning focus on distributed (parallelizing) learning. Distributed learning, based on a divide-and-conquer approach, has been used for, e.g.,  perceptron-based algorithms \citep{mcdonald2009efficient}, parametric smooth convex optimization problems \citep{zhang2012communication}, and sparse regression \citep{lee2015communication}. Recently, this approach has been successfully applied to learning algorithms with kernel methods, such as KRR \citep{zhang2015divide}, and SA \citep{guo2017learning,blanchard2016optimal}.
\cite{zhang2015divide} first studied distributed KRR and showed that distributed KRR retains optimal rates $O(N^{-{2\zeta \over  2\zeta+\gamma}})$ (for $\zeta \in [{1\over2},1]$) provided the partition level is not too large. The number of partition to retain optimal rate shown in \citep{zhang2015divide} for distributed KRR depends on some conditions which may be less well understood and thus potentially leads to a suboptimal partition number. \cite{lin2017distributed} provided an alternative and refined analysis for distributed KRR, leading to a less strict condition on the partition number. \cite{guo2017learning}  extended the analysis to distributed SA, an proved optimal convergence rate for the case $\zeta \geq 1/2$, if the number of partitions $m \leq N^{2\zeta - 1 \over  2\zeta+\gamma}.$ In comparison, the condition on partition number from 
Theorem \ref{thm:mainDFliter} for distributed SA is less strict. Moreover, 
Theorem \ref{thm:mainDFliter} 
shows that distributed SA can retain optimal rate even in the non-attainable case. According to Corollary \ref{cor:fastRatP}, distributed SGM with appropriate choices of parameters  can achieve optimal rate if the partition number is not too large. In comparison of the derived results for distributed SA with those for distributed SGM, we see from Table \ref{table:1} that the latter has advantages on both memory and time.
The most related to our works are \citep{zinkevich2010parallelized,jain2016parallelizing}.  \cite{zinkevich2010parallelized} studied distributed OL for optimization problems over a finite-dimensional domain, and proved convergence results assuming that the objective function is strongly convex.
\cite{jain2016parallelizing} considered distributed OL with averaging for least square regression problems over a finite-dimension space and proved certain convergence results that may depend on the smallest eigenvalue of the covariance matrix.
These results do not apply to our cases, as we consider distributed multi-pass SGM for nonparametric regression over a RKHS and our objective function is not strongly convex.
We finally remark that using a partition approach \citep{thomann2016spatial,tandon2016kernel}, one can also scale up the kernel methods, with a computational advantage similar as those of using distributed learning technique.

We conclude this section with some further questions. First, in this paper, we assume that all parameter choices are given priorly. In practice, these parameters can be possibly tuned by  cross-validation method. Second, the derived rate for SGM and SA when  $2\zeta+\gamma\leq 1$ is $O(N^{-2\zeta}(1 \vee \log N^{\gamma}))$, which is worser than $O(N^{-2\zeta})$ of averaging OL \citep{dieuleveut2016nonparametric}. It would be interesting to improve the rate, or to derive a minimax rate for the case $2\zeta+\gamma\leq 1.$ Third, all results stated in this paper are in expectation,  and it would be interesting to derive high-probability results (possibly by a proof technique from \citep{london2017pac}).

\section{Proofs for Distributed SGM}\label{sec:proof}

In this section, we provide the proofs of our main theorems for distributed SGM.  We begin with some basic notations. For ease of readability, we also make a list of notations in the appendix. 

\subsection{Notations}
$\mE[\xi]$ denotes the expectation of a random variable $\xi.$
$\|\cdot\|_{\infty}$ denotes the supreme norm with
respect to $\rho_X.$
For a given bounded operator $L: H' \to \HK'', $ $\|L\|$ denotes the operator norm of $L$, i.e., $\|L\| = \sup_{f\in \HK', \|f\|_{\HK'}=1} \|Lf\|_{\HK''}$. Here $\HK'$ and $\HK''$ are two separable Hilbert spaces (which could be the same).

We introduce the inclusion operator $\IK: \HK \to \LR$, which is continuous under Assumption \ref{as:basic}. Furthermore, we consider the adjoint operator $\IK^*: \LR \to \HK$, the covariance operator $\TK: \HK \to \HK$ given by $\TK = \IK^* \IK$, and the operator $\LK : \LR \to \LR$ given by $\IK \IK^*.$
It can be easily proved that $ \IK^*f = \int_X K_x f(x) d\rho_X(x)$
and $\TK = \int_X \la \cdot , K_{x} \ra_{\HK}K_x d \rho_X(x).$
The operators $\TK$ and $\LK$ can be proved to be positive trace class operators (and hence compact).
In fact, by \eref{eq:HK},
\be\label{eq:TKBound}
\|\LK\| = \|\TK\| \leq \tr(\TK) = \int_{X} \tr(K_x \otimes K_x)d\rho_X(x) = \int_X \|K_x\|_{\HK}^2 d\rho_X(x) \leq \kappa^2.
\ee
For any function $f \in \HK$,
the $\HK$-norm can be related to the $\LR$-norm by $\sqrt{\TK}:$ \citep{bauer2007regularization}
\be\label{isometry}
\|\IK f\|_{\rho} = \left\|\sqrt{\TK} f\right\|_{\HK},
\ee
and furthermore according to the singular value decomposition of $\IK$,
\be\label{eq:HKtoRHO}
 \|\LK^{-{1\over 2}} \IK f\|_{\rho} \leq \|f\|_{\HK}.
\ee

We define the sampling operator (with respect to any given set $\bx \subseteq X$ of cardinality $n$) $\SX: \HK \to \mR^n$ by $(\SX f)_i = f(x_i) = \la f, K_{x_i} \ra_{\HK},$ $i \in [n]$, where the norm $\|\cdot\|_{\mR^n}$ is the standard Euclidean norm times $1/\sqrt{n}$.
Its adjoint operator $\SX^*: \mR^n \to \HK,$ defined by $\la \SX^*{\by}, f \ra_{\HK} = \la {\bf y}, \SX f\ra_{\mR^n}$ for ${\by} \in \mR^n$ is thus given by
\be\label{eq:sampleOperAdjoint}
\SX^*{\by} = {1 \over n}\sum_{i=1}^{n} y_i K_{x_i}.
\ee
Moreover, we can define the empirical covariance operator (with respect to $\bx$)  $\TX: \HK \to \HK$ such that $\TX = \SX^* \SX$. Obviously,
\bea
\TX = {1 \over n} \sum_{i=1}^{n} \la \cdot, K_{x_i} \ra_{\HK} K_{x_i}.
\eea
By \eref{eq:HK}, similar to \eref{eq:TKBound}, we have
\be\label{eq:TXbound}
\|\TX\| \leq \tr(\TX) \leq \kappa^2.
\ee
For any $\PRegPar>0,$ for notational simplicity, we let $ \TKL = \TK + \PRegPar  $,
$\TXL = \TX + \PRegPar$, and 
$$\mcN(\PRegPar)  = \tr(\LK(\LK+\PRegPar)^{-1}) = \tr(\TK(\TK+\PRegPar)^{-1}).$$ 
For any  $f\in \HK$ and $x\in X$, the following well known reproducing property holds:
\be\label{eq:reproduce}
\la f, K_x\ra_{\HK} = f(x).
\ee
and following from the above, Cauchy-Schwarz inequality and \eqref{eq:HK}, one
can prove that	\be\label{eq:inftyThk}
|f(x)| = |\la f,K_x \ra_{\HK}| \leq \|f\|_{\HK} \|K_x\|_{\HK} \leq \kappa \|f\|_{\HK}\ee

For any $s \in [m]$, we denote the set of random variables $\{j_{s,i}\}_{b(t-1)+1 \leq i \leq bt}$ by $\J_{s,t},$ $\{ j_{s,1},j_{s,2},$ $\cdots,j_{s,bT}\}$  by $\J_s$, and $\{\J_1,\cdots, \J_m\}$ by $\J.$
Note that $j_{s,1},j_{s,2},\cdots,j_{s,bT}$ are conditionally independent given $\bz_s$.

\subsection{Error Decomposition}\label{subsec:error}

The key to our proof is an error decomposition. To introduce the  error decomposition, we need to introduce two auxiliary sequences.

The first auxiliary sequence is generated by distributed GM. 
	For any $s \in [m]$, the GM over the sample set $\bz_s$ is defined by $g_{s,1}=0$ and   
	\be\label{eq:alggm}
	g_{s,t+1}=g_{s,t} - \eta_t \left(\TXS g_{s,t} - \SXS^*\by_s\right),
	\qquad t=1, \ldots, T, 
	\ee
	where $\{\eta_{t}>0\}$ is a step-size sequence given by Algorithm \ref{alg:sgm}.  The average estimator over these local estimators is given by	
	\be\label{eq:DGM}
	\bar{g}_{t} = {1 \over m} \sum_{s=1}^m g_{s,t}.\ee

The second auxiliary sequence is generated by distributed pseudo GM as follows.  
	For any $s \in [m]$, the pseudo GM over the input set $\bx_s$ is defined by $h_{s,1}=0$ and   
	\be\label{eq:algpgm}
	h_{s,t+1}=h_{s,t} - \eta_t  \left(\TXS h_{s,t} - \LXS \FR \right),
	\qquad t=1, \ldots, T. 
	\ee  The average estimator over these local estimators is given by	
	\be\bar{h}_{t} = {1 \over m} \sum_{s=1}^m h_{s,t}.\ee	
In the above, for any given inputs set $\bx \subseteq  X^{|\bx|}$,
 $\LX: \LR \to \HK$ is defined as that for any 
$f \in \LR$ such that $\| f\|_{\infty} < \infty,$
\be\label{eq:LX}
\LX f = {1 \over |\bx|} \sum_{x \in \bx} f(x) K_{x}.
\ee
Note that \eref{eq:algpgm} can not be implemented in practice, as $\FR(x)$ is unknown in general.

We state the error decomposition as follows.
\begin{pro}
	We have that for any $t\in [T],$
	\label{pro:errDecob}
	\be\label{eq:errDecob}
	\mE\|\IK \bar{f}_t - \FR\|_{\rho}^2 = \mE\|\IK\bar{h}_t - \FR \|_{\rho}^2 +  \mE [\|\IK(\bar{g}_t -  \bar{h}_t) \|_{\rho}^2] +
	\mE\|\IK (\bar{f}_t - \bar{g}_t)\|_{\rho}^2  .\ee
\end{pro}

The error decomposition is similar as (but a bit different from)  \citep[Proposition 1]{lin2017optimal} for  classic multi-pass SGM. There are three terms in the right-hand side of \eref{eq:errDecob}. The first term depends on the regularity of the regression function (Assumption \ref{as:regularity}) and it is called as  {\it bias}. The second term depends on the noise level $\sigma^2$ from \eref{noiseExp} and it is called as {\it sample variance}. The last term is caused by the random estimates of the full gradients and it is called as {\it computational variance}. In the following subsections, we will estimate these three terms separately. Total error bounds can be thus derived by substituting these estimates into the error decomposition. 

The proof idea is quite simple. According to Lemmas \ref{lem:fulLocBias}, \ref{lem:samVar1} and \ref{lem:comVar1}, in order to proceed the analysis, we only need to estimate bias, sample and computational variance of a local estimator. In order to estimate local bias and local sample variance, as given in Lemma \ref{lem:gdSRA}, we rewrite $g_{s,t}$ and $h_{s,t}$ as the special forms induced by a filter function $G_t$ of GM. The strategy here for estimating local bias and sample variance is different from that in \citep{lin2017optimal} which relies on the following error decomposition and 
iterative relationship motivated by  \citep{lin2015learning}:
$$
\| g_{s,t}  - \FR \|_{\rho} \leq   \| r_{t}  - \FR \|_{\rho} + \| g_{s,t}  - r_{t} \|_{\rho} \quad {and}
$$
$$
g_{s,t+1}  - r_{t+1} = \sum_{k=1}^t \eta_k \prod_{j = k+1}^t ( I - \eta_j \TXS) (\TK r_{k} - \IK^*\FR - \TX r_{k} + \SXS^*\by_s),
$$
where the population sequence $\{r_t\}_t$ is defined by $r_1=0$ and
\be\label{eq:popSeq}
r_{t+1} = (I - \TK) r_t + \IK^*\FR.
\ee
Instead, in this paper, we use spectral theory from functional analysis to proceed the estimations. Our main novelties lies in a new error bound for $\|\TXL^{-{1\over 2}} \TKL^{1\over 2}\|$, see Lemma \ref{lem:operDifRes} (which allows one to derive optimal rates in the non-attainable cases without requiring the sample size is large enough, and refines the error bounds on $\|\TXL^{-{1\over 2}} \TKL^{1\over 2}\|$
 by involving Assumption \ref{as:eigenvalues} in the logarithmic factor),  and the estimation on local bias, as well as some other refined analysis. For estimating local bias, we introduce a new key error decomposition, in order to cover both the non-attainable case and the unbounded-output case.

All the missing proofs of propositions and lemmas in this section can be found in Appendix \ref{app:sgm}.

%

\subsection{Estimating Bias}\label{subsec:bias}
In this subsection, we estimate bias, i.e., $\mE \|\IK \bar{h}_t - \FR\|_{\rho}^2.$ We first give the following lemma, which asserts that the bias term can be estimated in terms of the bias of a local estimator.

\begin{lemma}\label{lem:fulLocBias} For any $t\in [T],$ we have
	$$
	\mE \|\IK \bar{h}_t -\FR\|_{\rho}^2 \leq \mE \|\IK h_{1,t} - \FR \|_{\rho}^2.
	$$
\end{lemma}

To estimate the bias of the local estimator, $\mE \|\IK h_{1,t} - \FR \|_{\rho}^2$, we next introduce some preliminary notations and lemmas. 

$\Pi_{t+1}^T(L) = \prod_{k=t+1}^T (I - \eta_k L)$ for $t \in [T-1]$ and $\Pi_{T+1}^T(L) = I,$ for any operator $L: H \to H,$
where $H$ is a Hilbert space and $I$ denotes the identity operator on $H$.
Let $k, t \in \mN.$ We use the following conventional notations:
$1/0=+\infty,$ 
$\prod_{k}^t = 1$ and $\sum_{k}^t = 0$ whenever $k >t$.
$\Sigma_{k}^t =\sum_{i=k}^t \eta_i$, $\lambda_{k:t} = (\Sigma_k^t)^{-1}$,
and specially $\lambda_{1:t}$ is abbreviated as $\RegParGD.$
Define the function $G_t: \mR \to \mR$ by
\be\label{eq:regFunc}
G_t(u) = \sum_{k=1}^t \eta_k \prod_{k=t+1}^T (I - \eta_k u).
\ee

Throughout this paper, we assume that the step-size sequence satisfies $\eta_t \in ]0,\kappa^{-2}]$ for all $t \in \mN.$ Thus, $G_t(u)$ and $\Pi_k^t(u)$ are non-negative on $]0,\kappa^2].$ 
For notational simplicity, throughout the rest of this subsection, we will drop the index $s=1$ for the first local estimator whenever it shows up, i.e, we abbreviate $h_{1,t}$ as $h_{t}$, $\bz_1$ as $\bz$, and $\TK_{\bx_1}$  as $\TK_{\bx}$, etc.

The key idea for our estimation on bias is that $\{h_{t}\}_t$ can be well approximated by the population sequence $\{r_t\}_t$, a deterministic sequence depending on the regression function $\FR$.

We first have the following observations.
\begin{lemma}\label{lem:gdSRA}
	The sequence $\{r_{t}\}_t$ defined by \eref{eq:popSeq} can be rewritten as
	\be\label{eq:popSeqEq}
	r_{t+1} = G_t (\TK) \IK^* \FR.
	\ee
	Similarly, for any $s \in [m]$,  the sequences $\{g_{s, t}\}_t$ and  $\{h_{s, t}\}_t$ defined by
	\eref{eq:alggm} and \eref{eq:algpgm} can be rewritten as
	$$
	g_{s,t+1} = \GL(\TXS) \SXS^* \by_s,
	$$
	and 
	$$
	h_{s,t+1} = \GL(\TXS) \LXS \FR.
	$$
\end{lemma}
\begin{proof}
	Using the relationship \eref{eq:popSeq} iteratively, introducing with $r_1 =0,$ one can prove the first conclusion.   
\end{proof}

According to the above lemma, we know that GM can be rewritten as a form of  SA with filter function $\GLB(\cdot) = \GL(\cdot).$ In the next lemma, we further develop some basic properties for this filter function.
\begin{lemma}\label{lem:regFuncIneq}
	For all $u \in [0,\kappa^{2}]$, \\
	1) $u^{\alpha} G_t(u) \leq \RegParGD^{\alpha-1}, \,\,\forall \alpha \in [0,1].$\\
	2) $(1 - u G_t(u))u^{\alpha} = \Pi_1^t(u) u^{\alpha} \leq (\alpha/\mathrm{e})^{\alpha}\RegParGD^{\alpha}, \quad \forall \alpha \in [0,\infty[.$ \\
	3) $\Pi_{k}^t(u)u^{\alpha} \leq (\alpha /\mathrm{e})^{\alpha} \lambda_{k:t}^{\alpha},\ \ \  \forall t,k\in \mN$.
\end{lemma}

According to Lemma \ref{lem:regFuncIneq}, $\GL(\cdot)$ is a filter function indexed with regularization parameter $\lambda = \RegParGD$, and the qualification $\tau$ can be any positive number, and $E = 1,$ $F_{\tau} = (\tau/\mathrm{e})^{\tau}.$ 
Using Lemma \ref{lem:regFuncIneq} and the spectral theorem, one can get the following results.
\begin{lemma}
	\label{lem:operatorNorm}
	Let $L$ be a compact, positive operator on a separable Hilbert space $H$ such that $ \|L\| \leq \kappa^2$. Then for any $\PRegPar \geq 0,$ \\
	1) $\|(L + \PRegPar)^{\alpha} G_t(L)\| \leq \RegParGD^{\alpha-1} (1+ (\PRegPar/\lambda_t)^\alpha), \, \, \,\forall \alpha \in [0,1].$\\
	2) $\|(I - L G_t(L)) (L + \PRegPar)^{\alpha}\| = \|\Pi_1^t(L) (L+\PRegPar)^{\alpha}\| \leq 2^{(\alpha-1)_+}( (\alpha/\mathrm{e})^{\alpha} + (\PRegPar/\RegParGD)^{\alpha})\RegParGD^{\alpha}, \quad \forall \alpha \in [0,\infty[.$ \\
	3)
	$
	\| \Pi_{k+1}^t(L) L^{\alpha}\| \leq \left( \alpha/ \mathrm{e}\right)^{\alpha} \lambda_{k:t}^{\alpha},\, \, \,\forall k, t \in \mN.
	$
\end{lemma}

To proceed the proof, we introduce the following basic lemmas on operators.
\begin{lemma} \cite[Cordes inequality]{fujii1993norm}
	\label{lem:operProd}
	Let $A$ and $B$ be two positive bounded linear operators on a separable Hilbert space. Then
	\bea
	\|A^s B^s\| \leq \|AB\|^s, \quad\mbox{when } 0\leq s\leq 1.
	\eea
\end{lemma}

\begin{lemma}\label{lem:sss}
	Let $H_1,H_2$ be two separable Hilbert spaces and  $\mcS: H_1 \to H_2$ a compact operator. Then for any function $f:[0,\|\mcS\|] \to [0,\infty[$,
	$$
	f(\mcS\mcS^*)\mcS = \mcS f(\mcS^*\mcS).
	$$
\end{lemma}
\begin{proof}
	The result can be proved using singular value decomposition of a compact operator.
\end{proof}

\begin{lemma}\label{lem:operDiff}
	Let $A$ and $B$ be two non-negative bounded linear operators on a separable Hilbert space with $\max(\|A\|,\|B\|) \leq \kappa^2$ for some non-negative $\kappa^2.$
	Then for any $\zeta>0,$
	\be
	\|A^\zeta - B^\zeta\| \leq C_{\zeta,\kappa} \|A -B\|^{\zeta \wedge 1},
	\ee
	where
	\be
	C_{\zeta,\kappa} = \begin{cases}
		1   & \mbox{when } \zeta \leq 1,\\
		2	\zeta \kappa^{2\zeta - 2} &\mbox{when } \zeta >1.
	\end{cases}
	\ee
\end{lemma}
\begin{proof}
	The proof is based on the fact that $u^{\zeta}$ is operator monotone if $0<\zeta\leq 1$. While for
	$\zeta \geq 1$,  the proof can be found in, e.g., \citep{dicker2017kernel}.
\end{proof}

Using Lemma \ref{lem:operatorNorm}, one can prove the following results, which give some basic properties for the population sequence $\{r_{t}\}_t$.
\begin{lemma}\label{lem:detmiticSeq}
	Let $a \in \mR.$ Under Assumption \ref{as:regularity}, the following results hold. \\
	1) For any $a \leq \zeta,$ we have
	\bea
	\|\LK^{-a}\left(\IK r_{t+1} - \FR\right)\|_{\rho} \leq \left( {(\zeta-a )/ \mathrm{e}}\right)^{\zeta-a}R \RegParGD^{\zeta-a}. 
	\eea
	2) 	We have
	\be
	\|\TK^{a-1/2} r_{t+1}\|_{\HK} \leq R\cdot  \begin{cases}
		\RegParGD^{\zeta+a -1}, & \text{if } -\zeta \leq a\leq 1-\zeta, \\
		\kappa^{2(\zeta +a-1)},& \text{if } \ a\geq 1 -\zeta.
	\end{cases}
	\ee
\end{lemma}
\begin{proof}
	1) Using  Lemma \ref{lem:sss},
	$$\IK \GL(\TK) \IK^* = \IK \GL(\IK^* \IK) \IK^* =  \GL(\IK\IK^* ) \IK\IK^* = \GL(\LK ) \LK,$$
and by \eref{eq:popSeqEq},	we have
	\bea
	\LK^{-a}(\IK \FL - \FR) = \LK^{-a}\left(G_t(\LK)\LK - I\right) \FR.
	\eea
	Taking the $\rho$-norm, applying Assumption \ref{as:regularity}, we have
	\bea
	\|\LK^{-a}(\IK \FL - \FR)\|_{\rho} \leq \| \LK^{\zeta-a}(G_t(\LK)\LK - I)\| R =  \| \LK^{\zeta-a}\Pi_1^t(\LK)\|R.
	\eea
	Note that the condition \eref{eq:HK} implies
	\eref{eq:TKBound}.  Applying Part 2) of Lemma \ref{lem:operatorNorm},  one can  prove the first desired result.
	
	2) By \eref{eq:popSeqEq} and Assumption \ref{as:regularity}, 
	$$
	\|\TK^{a-1/2} r_{t+1}\|_{\HK} = 	\|\TK^{a-1/2} G_{t}(\TK) \IK^* \FR \|_{\HK} \leq  \|\TK^{a-1/2} G_{t}(\TK) \IK^* \LK^{\zeta} \|R.
	$$
	Noting that
	\begin{align*}
	\|\TK^{a-1/2} G_{t}(\TK) \IK^* \LK^{\zeta} \| & = \|\TK^{a-1/2} G_{t}(\TK) \IK^* \LK^{2\zeta} \IK G_t(\TK) \TK^{a-1/2} \|^{1/2}  \\
	&= \| G_t^2(\TK) \TK^{2\zeta+2a}\|^{1/2} = \|G_t(\TK) \TK^{\zeta+a}\|,
	\end{align*}
	we thus have
	$$
	\|\TK^{a-1/2} r_{t+1}\|_{\HK}  \leq \|G_t(\TK) \TK^{\zeta+a}\| R.
	$$
	If $0\leq \zeta+a \leq 1,$ i.e., $ -\zeta  \leq a \leq 1 -\zeta$, then by using 1) of Lemma \ref{lem:operatorNorm}, we get  
	$$
	\|\TK^{a-1/2} r_{t+1}\|_{\HK}  \leq \RegParGD^{\zeta + a-1} R.
	$$
	Similarly, when $a \geq  1 -\zeta,$ we have
	$$
	\|\TK^{a-1/2} r_{t+1}\|_{\HK}  \leq \|G_t(\TK) \TK\|\|\TK\|^{\zeta+a-1} R \leq  \kappa^{2(\zeta+a -1)} R,
	$$
	where for the last inequality we used 1) of
	Lemma \ref{lem:operatorNorm} and \eref{eq:TKBound}. This thus proves the second desired result.
\end{proof}

With the above lemmas, we can prove the the following analytic result, which enables us  to estimate the bias term in terms of several random quantities.

\begin{lemma}\label{lem:anaResult}
	Under Assumption \ref{as:regularity}, let $\PRegPar>0$,
	\bea
	\DZF = \|\TKL^{1/2} \TXL^{-1/2}\|^2 \vee 1, \qquad \DZT = \|\TK- \TX\|
	\eea
	and
	\bea
	\DZS = \| \LX f_{\rho} - \IK^* f_{\rho} - \TX \FL + \TK \FL\|_{\HK}.
	\eea
	Then the following results hold.\\
	1) For $0<\zeta \leq 1,$
	\be\label{eq:anaResA}
	\|\IK h_{t+1} - \FR\|_{\rho} \leq  \left( 1 \vee \left({\PRegPar \over \RegParGD}\right)^{\zeta \vee {1 \over 2}}\right) (C_1 R (\DZF)^{\zeta \vee {1\over 2}} \RegParGD^{\zeta}+ 2 \sqrt{\DZF} \RegParGD^{-{1\over 2}} \DZS) .
	\ee
	2) For $\zeta >1,$
	\be\label{eq:anaResB}
	\|\IK h_{t+1} - \FR\|_{\rho} \leq  \sqrt{\DZF} \left(1 \vee \left({\PRegPar\over \RegParGD} \right)^{\zeta}\right) (C_2 R\RegParGD^{\zeta} +  2 \RegParGD^{-{1\over 2}} \DZS + C_{3} R\RegParGD^{1 \over 2} (\DZT)^{(\zeta-{1\over 2})\wedge 1}) .
	\ee
	Here, $C_1$, $C_2$ and $C_3$ are positive constants depending only on $\zeta$ and $\kappa$.
\end{lemma}

\begin{proof}
	Using Lemma \ref{lem:gdSRA} with $s=1$, 
	we can estimate $\|\IK h_{t+1} - \FR\|_{\rho}$ as
	\begin{align}
	\|\IK G_t(\TX)\LX \FR -\FR\|_{\rho} \leq & \|\underbrace{\IK G_t(\TX)[\LX \FR - \IK^* \FR - \TX r_{t+1} + \TK r_{t+1}]}_{\text{\bf Bias.1}}\|_{\rho} \nonumber\\
	&+ \| \underbrace{\IK G_t(\TX)[\IK^* f_{\rho} - \TK r_{t+1}]}_{\text{\bf Bias.2}}\|_{\rho} \nonumber\\
	&+ \|\underbrace{\IK [I - G_t(\TX)\TX] r_{t+1}}_{\text{\bf Bias.3}}\|_{\rho}\nonumber\\
	&+  \|\underbrace{\IK \FL - \FR}_{\text{\bf Bias.4}}\|_{\rho}. \label{eq:biasDecom}
	\end{align}
	In the rest of the proof, we will estimate the four terms of the r.h.s separately.\\
	{\bf Estimating Bias.4}\\
	Using 1) of Lemma \ref{lem:detmiticSeq} with $a=0$, we get
	\be\label{eq:bias4B}
	\|{\bf Bias.4} \|_{\rho} \leq (\zeta/\mathrm{e})^{\zeta} \RegParGD^{\zeta}R.
	\ee 
	{\bf Estimating Bias.1} \\
	By a simple calculation, we know that for any $f\in \HK,$
	\bea
	\| \IK \GL(\TX)f\|_{\rho} \leq \|\IK \TKL^{-1/2}\| \|\TKL^{1/2} \TXL^{-1/2}\| \|\TXL^{1/2} \GL(\TX)\| \| f\|_{\HK}.
	\eea
	Note that
	\be\label{eq:7}  \| \IK \TKL^{-1/2}\| = \sqrt{\|\IK \TKL^{-1} \IK^*\| } = \sqrt{\|\LK \LKL^{-1}\|} \leq 1,
	\ee
	and that applying 1) of Lemma \ref{lem:operatorNorm}, with  \eref{eq:TXbound}, we have
	$$
	\|\TXL^{1/2} \GL(\TX)\|   \leq (1+\sqrt{\PRegPar/\RegParGD})/\sqrt{\RegParGD}. 
	$$
	Thus for any $f\in \HK,$ we have
	\be\label{eq:6}
	\|\IK \GL(\TX)f\|_{\rho} \leq (1+\sqrt{\PRegPar/\RegParGD})\RegParGD^{-{1 \over 2}}
	\sqrt{\DZF} \| f\|_{\HK}.
	\ee
	Therefore,
	\begin{align}\label{eq:bias1B}
	\|{ \bf Bias.1}\|_{\rho} \leq (1+\sqrt{\PRegPar/\RegParGD})\RegParGD^{-{1 \over 2}}
	\sqrt{\DZF} \DZS.
	\end{align}
	{\bf 
		Estimating Bias.2}\\
	By \eref{eq:6}, we have 
	$$
	\|{ \bf Bias.2}\|_{\rho} \leq
	(1+\sqrt{\PRegPar/\RegParGD})\RegParGD^{-{1 \over 2}}
	\sqrt{\DZF} \|\TK\FL -\IK^* \FR\|_{\HK}.
	$$
	Using (with $\TK = \IK^* \IK$ and $\LK = \IK \IK^*$)
	$$
	\|\TK\FL -\IK^* \FR\|_{\HK} = \|\IK^*(\IK\FL - \FR) \|_{\HK} = \|\LK^{1/2}(\IK\FL - \FR) \|_{\rho},
	$$
	and applying 1) of Lemma \ref{lem:detmiticSeq} with $a=-1/2$,  we get 
	\begin{align}\label{eq:bias2B}
	\|{ \bf Bias.2}\|_{\rho} \leq ((\zeta+1/2)/\mathrm{e})^{\zeta+1/2} (1+\sqrt{\PRegPar/\RegParGD}) 
	\sqrt{\DZF} \RegParGD^{\zeta} R	.
	\end{align}
	{\bf Estimating Bias.3}\\
	By 2) of Lemma \ref{lem:regFuncIneq}, 
	$$
	{\bf Bias.3} = \IK \Pi_1^t(\TX) r_{t+1}.
	$$
	When $\zeta \leq 1/2,$ 	by a simple calculation, we have
	\begin{align*}
	\|{\bf Bias.3}\|_{\rho} \leq& \|\IK\TKL^{-1/2}\| \|\TKL^{1/2} \TXL^{-1/2}\| \|  \TXL^{1/2}\Pi_1^t (\TX)\| \| r_{t+1}\|_{\HK} \\ 
	\leq & \sqrt{\DZF} \|  \TXL^{1/2}\Pi_1^t (\TX)\| \| r_{t+1}\|_{\HK},
	\end{align*}
	where for the last inequality, we used \eref{eq:7}.
	By 2) of Lemma \ref{lem:operatorNorm}, with  \eref{eq:TXbound},
	\be\label{eq:8}
	\|  \TXL^{1/2}\Pi_1^t (\TX)\| \leq   \sqrt{\RegParGD} (1/\sqrt{2\mathrm{e}}  + \sqrt{\PRegPar/\RegParGD}),
	\ee
	and by 2) of Lemma \ref{lem:detmiticSeq}, 
	$$
	\|r_{t+1}\|_{\HK} \leq R \RegParGD^{\zeta-1/2} .
	$$
	It thus follows that
	$$
	\|{\bf Bias.3}\|_{\rho} \leq  \sqrt{\DZF} (\sqrt{\PRegPar/\RegParGD} + 1/\sqrt{2\mathrm{e}}) R \RegParGD^{\zeta}. 
	$$
	When $1/2 < \zeta \leq 1,$
	by a simple computation, we have 
	\begin{align*}
	\|{\bf Bias.3}\|_{\rho} \leq 
	\|\IK\TKL^{-1/2}\|  \|\TKL^{1/2} \TXL^{-1/2}\|  \| \TXL^{1/2}\Pi_1^t (\TX)\TXL^{\zeta - 1/2}\| \|
	\TXL^{1/2-\zeta} \TKL^{\zeta - 1/2}\|  \|\TKL^{1/2-\zeta} r_{t+1}\|_{\HK}.
	\end{align*}
	Applying \eref{eq:7} and 2) of Lemma \ref{lem:detmiticSeq}, we have 
	\begin{align*}
	\|{\bf Bias.3}\|_{\rho} \leq 
	\sqrt{\DZF} \| \TXL^{1/2}\Pi_1^t (\TX)\TXL^{\zeta - 1/2}\| \|
	\TXL^{1/2-\zeta} \TKL^{\zeta - 1/2}\|   R.
	\end{align*}
	By 2) of Lemma \ref{lem:operatorNorm}, 
	$$
	\| \TXL^{1/2}\Pi_1^t (\TX)\TXL^{\zeta - 1/2}\| = \| \TXL^{\zeta}\Pi_1^t (\TX)\| \leq  ( (\zeta/\mathrm{e})^{\zeta} +  (\PRegPar/\RegParGD)^{\zeta})\RegParGD^{\zeta} .
	$$
	Besides, by $\zeta \leq 1$ and Lemma \ref{lem:operProd}, 
	$$
	\|	\TXL^{1/2-\zeta} \TKL^{\zeta - 1/2}\| =   \|\TXL^{-{1\over 2}(2\zeta-1)} \TKL^{{1 \over 2}(2\zeta - 1)}\| \leq  \|\TXL^{-{1\over 2}} \TKL^{{1 \over 2}}\|^{2\zeta - 1} \leq (\DZF)^{\zeta -{1 \over 2}}.
	$$
	It thus follows that
	$$
	\|{\bf Bias.3}\|_{\rho} \leq  (\DZF)^{\zeta} ((\PRegPar/\RegParGD)^{\zeta} + (\zeta/\mathrm{e})^{\zeta}) R \RegParGD^{\zeta}.
	$$
	When $\zeta>1,$ we  rewrite {\bf Bias.3} as
	\begin{align*}
	\IK\TKL^{-1/2} \cdot \TKL^{1/2} \TXL^{-1/2}\cdot  \TXL^{1/2}\Pi_1^t (\TX) (\TX^{\zeta-1/2} + \TK^{\zeta-1/2} - \TX^{\zeta-1/2}) \TK^{1/2 - \zeta} r_{t+1}.
	\end{align*}
	By a simple calculation, we can  upper bound $\|{\bf Bias.3}\|_{\rho}$ by
	$$
	\leq \|\IK\TKL^{-1/2}\|  \|\TKL^{1/2} \TXL^{-1/2}\| ( \|\TXL^{1/2} \Pi_1^t(\TX) \TX^{\zeta-1/2}\|+ \|\TXL^{1/2}\Pi_1^t(\TX)\| \|\TK^{\zeta-1/2} - \TX^{\zeta-1/2}\|) \|\TK^{1/2 - \zeta} r_{t+1}\|.
	$$
	Introducing with \eref{eq:7} and \eref{eq:8}, and applying 2) of Lemma \ref{lem:detmiticSeq}, 
	$$
	\|{\bf Bias.3}\|_{\rho} \leq \sqrt{\DZF} ( \|\TXL^{1/2} \Pi_1^t(\TX) \TX^{\zeta-1/2}\|+ (1/\sqrt{2\mathrm{e}} + \sqrt{\PRegPar/\RegParGD})\sqrt{\RegParGD} \|\TK^{\zeta-1/2} - \TX^{\zeta-1/2}\|) R.
	$$
	By  2) of
	Lemma \ref{lem:operatorNorm},
	$$
	\|\TXL^{1/2} \Pi_1^t(\TX) \TX^{\zeta-1/2}\| \leq \|\TXL^{\zeta} \Pi_1^t(\TX) \| \leq  2^{\zeta-1}  ((\zeta/\mathrm{e})^{\zeta} + (\PRegPar/\RegParGD)^{\zeta})\RegParGD^{\zeta}.
	$$
	Moreover, by Lemma \ref{lem:operDiff} and $\max(\|\TK\|,\|\TX\|) \leq \kappa^2$,
	$$\|\TK^{\zeta-1/2} - \TX^{\zeta-1/2}\|  \leq (2\zeta \kappa^{2\zeta-3})^{\mathbf{1}_{\{2\zeta\geq 3\}}} \|\TK - \TX\|^{(\zeta-1/2) \wedge 1}.
	$$ 
	Therefore, when $\zeta >1$, {\bf Bias.3} can be estimated as
	\begin{align*}
	&\|{\bf Bias.3}\|_{\rho} \\
	\leq&   \sqrt{\DZF} \left(2^{\zeta-1}  ((\zeta/\mathrm{e})^{\zeta} + (\PRegPar/\RegParGD)^{\zeta})\RegParGD^{\zeta}+ 
	(2\zeta \kappa^{2\zeta-3})^{{\mathbf{1}_{\{2\zeta\geq 3\}}}} 
	(1/\sqrt{2\mathrm{e}} + \sqrt{\PRegPar/\RegParGD})\sqrt{\RegParGD} (\DZT)^{(\zeta-1/2) \wedge 1}\right) R.
	\end{align*}
	From the above analysis,  we know that $\|{\bf Bias.3}\|_{\rho}$ can be upper bounded by 
	\be\label{eq:bias3B}
	\begin{cases} 
		\sqrt{\DZF} (\sqrt{\PRegPar/\RegParGD} + 1/\sqrt{2\mathrm{e}}) R \RegParGD^{\zeta},& \mbox{if } \zeta \in ]0,1/2], \\
		(\DZF)^{\zeta} (\left({\PRegPar/ \RegParGD}\right)^{\zeta} + (\zeta/\mathrm{e})^{\zeta}) R \RegParGD^{\zeta}, & \mbox{if } \zeta \in ]1/2,1], \\ 
		\sqrt{\DZF} \left(2^{\zeta-1}  (\left({\zeta \over \mathrm{e}}\right)^{\zeta} + ({\PRegPar\over \RegParGD})^{\zeta})\RegParGD^{\zeta}+ 
		(2\zeta \kappa^{2\zeta-3})^{{\mathbf{1}_{\{2\zeta\geq 3\}}}} 
		({1\over \sqrt{2\mathrm{e}}} + \sqrt{{\PRegPar\over \RegParGD}})\sqrt{\RegParGD} (\DZT)^{(\zeta- {1\over 2}) \wedge 1}\right) R,& \mbox{if } \zeta\in ]1,\infty[. 
	\end{cases}
	\ee
	
	Introducing \eref{eq:bias4B},  \eref{eq:bias1B},
	\eref{eq:bias2B} and \eref{eq:bias3B} into \eref{eq:biasDecom}, 
	and by a simple calculation, one can prove the desired results with 
	$$C_1 = 
	(\zeta /\mathrm{e})^{\zeta} + 2((\zeta+{1 \over 2})/\mathrm{e})^{\zeta+{1 \over 2}} + ((\zeta \vee {1 \over 2})/\mathrm{e})^{\zeta \vee {1 \over 2}}+1, 
	$$
	$$C_2 = 
	(2^{\zeta-1} + 1)
	(\zeta /\mathrm{e})^{\zeta} + 2((\zeta+{1 \over 2})/\mathrm{e})^{\zeta+{1 \over 2}} 
	+ 2^{\zeta-1},
	$$
	$$
	\mbox{and }\quad C_3 = (2\zeta \kappa^{2\zeta-3})^{\mathbf{1}_{\{2\zeta\geq 3\}}}  (1/\sqrt{2\mathrm{e}} + 1).$$	
\end{proof}

The upper bounds  in \eref{eq:anaResA} and \eref{eq:anaResB} depend on three random quantities, $\DZF$, $\DZT$ and $\DZS$. To derive error bounds for the bias term from Lemma \ref{lem:anaResult}, it is necessary to estimate these three random quantities. We thus introduce the following lemmas.

\begin{lemma}\label{lem:statEstim}
	Let $f: X\to Y$ be a measurable function such that $\|f\|_{\infty} < \infty,$ then with probability at least $1-\delta$ ($0<\delta<1/2$),
	\bea
	\| \LX f - \LK f\|_{\HK} \leq 2\kappa\left( {2\|f\|_{\infty} \over | {\bf x}|} + {\|f\|_{\rho} \over \sqrt{|\bf x|}}  \right)\log {2 \over \delta}.
	\eea
\end{lemma}

\begin{lemma}\label{lem:statEstiOper}
	Let $0<\delta<1/2.$ It holds with probability at least $1-\delta:$
	\bea
	\	\|\TK - \TX\|_{HS} \leq { 6 \kappa^2 \over \sqrt{{|\bx|}}} \log {2\over \delta}.
	\eea
	Here, $\|\cdot\|_{HS}$ denotes the Hilbert-Schmidt norm.
\end{lemma}

\begin{lemma}\label{lem:operDifEff}
	Let $0<\delta <1$ and $\lambda>0$. With probability at least $1-\delta,$ the following holds:
	\bea
	\left\| (\TK+\lambda)^{-1/2}(\TK - \TX)(\TK+\lambda)^{-1/2}  \right\| \leq {4\kappa^2 \beta  \over 3 {|\bx| }\lambda} + \sqrt{2\kappa^2 \beta  \over {|\bx|}\lambda}, \quad \beta = \log {4\kappa^2( \mcN(\lambda)+1) \over \delta \|\TK\|}.
	\eea
\end{lemma}

The proofs of Lemmas \ref{lem:statEstim} and \ref{lem:statEstiOper} are based on concentration result for Hilbert space valued random variable from \citep{pinelis1986remarks}, while the proof of Lemma \ref{lem:operDifEff} is based on the concentration inequality for norms of self-adjoint operators on a Hilbert space from \citep{tropp2012user,minsker2011some}. For completeness, we give the proofs in the appendix.

We will use Lemmas \ref{lem:statEstim} and \ref{lem:detmiticSeq}  to estimate the quantity $\DZS$. The quantity $\DZT$ can be estimated by Lemma \ref{lem:statEstiOper} directly, as $\|\TK - \TX\|\leq \|\TK - \TX\|_{HS}.$ The quantity $\DZF$ can be estimated by the following lemma, whose proof is based on Lemma \ref{lem:operDifEff}.

\begin{lemma}\label{lem:operDifRes}
	Under Assumption \ref{as:eigenvalues},
	let $c,\delta\in(0,1)$, $\lambda= |{\bf x}|^{-\theta}$ for some $\theta\geq 0$, and
	\be\label{eq:defa}
	a_{|\bx|,\delta,\gamma}(c,\theta) =  {32\kappa^2 \over (\sqrt{9+24c} - 3)^2 }  \left(\log{ {4\kappa^2(c_{\gamma}+1) }\over \delta \|\TK\|} + \theta \gamma \min\left({1 \over \mathrm{e}(1-\theta)_+},\log |\bx|\right)\right).
	\ee
	Then with probability at least $1-\delta,$
	\bea
	\| (\TK+\lambda)^{-1/2}(\TX+\lambda)^{1/2}\|^2 \leq (1+c) a_{|\bx|,\delta,\gamma}(c,\theta) (1 \vee |\bx|^{\theta-1}), \mbox{ and}\eea
	\bea
	\| (\TK+\lambda)^{1/2}(\TX+\lambda)^{-1/2}\|^2 \leq (1-c)^{-1} a_{|\bx|,\delta,\gamma}(c,\theta) (1 \vee |\bx|^{\theta-1}).
	\eea
\end{lemma}
\begin{rem}	
	Typically, we will choose $c=2/3.$ In this case,
	\be\label{eq:aa}
	a_{|\bx|,\delta,\gamma}({2/3},\theta) =  8\kappa^2 \left(\log{ {4\kappa^2(c_{\gamma}+1) }\over \delta \|\TK\|} + \theta \gamma \min\left({1 \over \mathrm{e}(1-\theta)_+},\log |\bx|\right)\right).
	\ee
	We have with probability at least $1-\delta,$
	\bea
	\| (\TK+\lambda)^{1/2}(\TX+\lambda)^{-1/2}\|^2 \leq 3 a_{|\bx|,\delta,\gamma}(2/3,\theta) (1 \vee |\bx|^{\theta-1}).
	\eea
\end{rem}
\begin{proof}
	We use Lemma \ref{lem:operDifEff} to prove the result. Let
	$c\in(0,1].$ By a simple calculation, we have that if $ 0\leq u  \leq {\sqrt{9+24c} - 3 \over 4},$ then
	$2u^2/3 + u\leq c.$
	Letting $\sqrt{2\kappa^2\beta  \over |\bx|\lambda'}=u,$ and combining with Lemma \ref{lem:operDifEff},
	we know that if
	\bea
	\sqrt{2\kappa^2\beta \over |\bx|\lambda'}   \leq {\sqrt{9+24c} - 3 \over 4},
	\eea
	which is equivalent to
	\be\label{eq:num1}
	|\bx| \geq {32\kappa^2 \beta \over (\sqrt{9+24c} - 3)^2 \lambda'}, \quad \beta  = \log{ 4\kappa^2( 1 + \mcN(\lambda')) \over \delta \|\TK\|},
	\ee
	then
	with probability at least $1-\delta,$
	\be\label{eq:1}
	\left\| \TK_{\lambda'}^{-1/2}(\TK - \TX)\TK_{\lambda'}^{-1/2}  \right\| \leq c.
	\ee
	Note that from \eref{eq:1}, we can prove
	\be\label{eq:2b}
	\| \TK_{\lambda'}^{-1/2}\TK_{\bx \lambda'}^{1/2}\|^2 \leq c+1, \quad \| \TK_{\lambda'}^{1/2}\TK_{\bx\lambda'}^{-1/2}\|^2 \leq (1-c)^{-1}.
	\ee
	Indeed, by simple calculations,
	\bea
	&&\| \TK_{\lambda'}^{-1/2}\TK_{\bx \lambda'}^{1/2}\|^2 = \| \TK_{\lambda'}^{-1/2}\TK_{\bx \lambda'}\TK_{\lambda'}^{-1/2}\| = \| \TK_{\lambda'}^{-1/2}(\TK - \TX)\TK_{\lambda'}^{-1/2}+ I\| \\
	&&\leq \| \TK_{\lambda'}^{-1/2}(\TK - \TX)\TK_{\lambda'}^{-1/2}\|+ \|I\| \leq c+1,
	\eea
	and \citep{caponnetto2007optimal}
	\bea
	\| \TK_{\lambda'}^{1/2}\TK_{\bx\lambda'}^{-1/2}\|^2 = \| \TK_{\lambda'}^{1/2}\TK_{\bx\lambda'}^{-1}\TK_{\lambda'}^{1/2}\|
	=\|(I - \TK_{\lambda'}^{-1/2}(\TK - \TX)\TK_{\lambda'}^{-1/2})^{-1} \| \leq (1-c)^{-1}.
	\eea
	From the above analysis, we know that for any fixed $\lambda'>0$ such that \eref{eq:num1}, then with probability at least $1-\delta,$
	\eref{eq:2b} hold.
	
	Now let $\lambda' = a \lambda $ when $\theta\in[0,1)$ and $\lambda'= a |\bx|^{-1}$ when $\theta \geq 1,$ where for notational simplicity,
	we denote  $a_{|\bx|,\delta,\gamma}(c,\theta)$ by $a$. We will prove that the choice on $\lambda'$ ensures the condition \eref{eq:num1} is satisfied, as thus with probability at least $1-\delta,$
	\eref{eq:2b} holds. Obviously, one can easily prove that $a\geq 1$, using $\kappa^2 \geq 1$ and \eqref{eq:TKBound}. Therefore, $\lambda' \geq \lambda,$ and
	\bea
	\| \TK_{\lambda}^{1/2}\TK_{\bx\lambda}^{-1/2}\| \leq \| \TK_{\lambda}^{1/2} \TK_{\lambda'}^{-1/2}\|
	\|\TK_{\lambda'}^{1/2} \TK_{\bx\lambda'}^{-1/2}\| \| \TK_{\bx \lambda'}^{1/2} \TK_{\bx\lambda}^{-1/2}\| \leq \|\TK_{\lambda'}^{1/2} \TK_{\bx\lambda'}^{-1/2}\|\sqrt{\lambda'/\lambda},
	\eea
	where for the last inequality, we used $\| \TK_{\lambda}^{1/2} \TK_{\lambda'}^{-1/2}\|^2\leq \sup_{u\geq 0} {u + \lambda \over u+\lambda'} \leq 1$ and $\| \TK_{\bx \lambda'}^{1/2} \TK_{\bx\lambda}^{-1/2}\|^2 \leq \sup_{u\geq 0} {u + \lambda' \over u+\lambda} \leq \lambda'/\lambda.$
	Similarly, $$\| \TK_{\lambda}^{-1/2}\TK_{\bx \lambda}^{1/2}\|  \leq \| \TK_{\lambda'}^{-1/2}\TK_{\bx \lambda'}^{1/2}\|\sqrt{\lambda'/\lambda}.$$ Combining with \eref{eq:2b}, and
	by a simple calculation, one can prove the desired bounds.
	What remains is to prove that the condition \eref{eq:num1} is satisfied.
	By Assumption \ref{as:eigenvalues} and $a\geq 1,$
	\bea
	\beta \leq \log{ 4\kappa^2( 1 + c_{\gamma} a^{-\gamma} |\bx|^{(\theta\wedge 1)\gamma}) \over \delta \|\TK\|} \leq  \log{ 4\kappa^2( 1 + c_{\gamma}) |\bx|^{\theta\gamma} \over \delta \|\TK\|} = \log{ 4\kappa^2( 1 + c_{\gamma})  \over \delta \|\TK\|} + {\theta\gamma}\log |\bx|.
	\eea
	If $\theta \geq 1$, or  $\theta\gamma=0$, or $\log |\bx| \leq {1 \over (1-\theta)_+ \mathrm{e}},$
	then the condition \eref{eq:num1}  follows trivially.
	Now consider the case $\theta \in (0,1),$ $\theta\gamma \neq 0$ and $\log |\bx| \geq {1 \over (1-\theta)_+ \mathrm{e}}.$ In this case,
		we
	apply \eqref{expx} to get ${\theta\gamma \over 1-\theta}\log |\bx|^{1-\theta} \leq {\theta\gamma \over 1-\theta} {|\bx|^{1-\theta} \over \mathrm{e}} $, and thus
	\bea
	\beta \leq \log{ 4\kappa^2(1+c_{\gamma}) \over \delta \|\TK\|} + {\theta\gamma \over 1-\theta} {|\bx|^{1-\theta} \over \mathrm{e}} .
	\eea
	Therefore, a sufficient condition for \eref{eq:num1} is
	\bea
	{|\bx|^{1-\theta}a \over g(c)} \geq  \log{ 4\kappa^2(1+c_{\gamma}) \over \delta \|\TK\|} + {\theta\gamma \over \mathrm{e}(1-\theta)} |\bx|^{1-\theta}, \quad g(c) = { 32\kappa^2 \over (\sqrt{9+24c} - 3)^2 }.
	\eea
	From the definition of $a$ in \eqref{eq:defa}, 
	$$
	a =  g(c) \left(\log{ {4\kappa^2(c_{\gamma}+1) }\over \delta \|\TK\|} + {\theta \gamma  \over \mathrm{e}(1-\theta)_+}\right),
	$$
	 and by a direct calculation, one can prove that the condition \eref{eq:num1} is satisfied.
	The proof is complete.
\end{proof}

We also need the following lemma, which enables one to derive convergence results in expectation from convergence results in high probability.
\begin{lemma} \label{lem:hipbExp}
	Let $F: ]0, 1] \to \mR_+$ be a monotone non-increasing, continuous function,  and
	$\xi$ a nonnegative real random variable such that
	$$ \Pr[\xi > F(t)] \leq  t ,\quad  \forall t \in (0, 1].$$
	Then
	$$\mE [\xi] \leq \int_{0}^{1} F(t) dt.
	$$
\end{lemma}
The proof of the above lemma can be found in, e.g., \citep{blanchard2016optimal}. 
Now we are ready to state and prove the following result for the local bias.

\begin{pro}\label{pro:localBias}
	Under Assumptions \ref{as:regularity} and \ref{as:eigenvalues}, we
	let $\PRegPar=n^{-1+\theta}$ for some $\theta \in [0,1]$. Then for any $t \in [T]$, the following results hold.\\
	1) For $0<\zeta \leq 1,$	
	$$
	\mE \|\IK h_{t+1} - \FR\|_{\rho}^2 \leq C_5(R + {\bf{1}}_{\{\zeta< 1/2 \}} \|\FR\|_\infty)^2 \left(1 \vee  {\PRegPar^2 \over \RegParGD^2} \vee [\gamma(\theta^{-1}\wedge \log n)]^{2\zeta \vee 1} \right)\RegParGD^{2\zeta}.
	$$
	2)For $\zeta>1,$
	$$
	\mE \|\IK h_{t+1} - \FR\|_{\rho}^2 \leq C_6 R^2 \left(1 \vee {\PRegPar^{2\zeta}\over \RegParGD^{2\zeta}}  \vee {\RegParGD^{1-2\zeta}} \left({1\over n}\right)^{(\zeta-{1\over 2})\wedge 1} \vee [\gamma (\theta^{-1} \wedge \log n)] \right) \RegParGD^{2\zeta}.
	$$
	Here, $C_5$ and $C_6$ are positive constants depending only on $\kappa,$ $\zeta$ and can be given explicitly in the proof.
\end{pro}
\begin{rem}
	In this paper, we did not try to optimize the constants  from the error bounds. But one should keep in mind that the constants can be further improved using an alternative proof for some special case, e.g., $\gamma=0$ \citep{hsu2014random}, or $\zeta \geq 1/2$ \citep{caponnetto2007optimal}, or $|y| \leq M$. Furthermore, by assuming that $|y| \leq M$ almost surely (which is not even satisfied with linear measurement model with Gaussian noise), the proof can be further simplified.
	 Note also that, the constants from  our error bounds appear to be larger than those from \citep{hsu2014random,caponnetto2007optimal},
	but our results do not require the extra assumption
	that the sample size is large enough as those in \citep{hsu2014random,caponnetto2007optimal}.
\end{rem}
\begin{proof}
	We will use Lemma \ref{lem:anaResult} to prove the results. To do so, we need to estimate $\DZF,$ $\DZS$ and $\DZT$.
	
	By Lemma \ref{lem:operDifRes}, we have that with probability at least $1-\delta,$
	\be\label{eq:4}
	\DZF \leq 3 a_{n,\delta,\gamma}(1-\theta) \leq (1 \vee \gamma[\theta^{-1}\wedge \log n])24 \kappa^2 \log {4\kappa^2 \mathrm{e}(c_{\gamma}+1) \over \delta \|\TK\|},
	\ee
	where $a_{n,\delta,\gamma}(1-\theta) = a_{n,\delta,\gamma}(2/3,1-\theta)$, given by \eref{eq:aa}.
	By Lemma \ref{lem:statEstim}, we have that with probability at least $1-\delta,$
	$$\DZS \leq 2\kappa \left( {2 \|r_{t+1} - \FR\|_{\infty} \over n } + {\|\IK r_{t+1} - \FR\|_{\rho} \over \sqrt{n}}\right)\log {2 \over \delta} .
	$$
	Applying Part 1) of Lemma \ref{lem:detmiticSeq} with $a=0$ to estimate $\|\IK r_{t+1} - \FR\|_{\rho}$,
	we get that with probability at least $1-\delta$,	
	$$\DZS  \leq 2\kappa \left( {2 \|r_{t+1} - \FR\|_{\infty} / n } + (\zeta/\mathrm{e})^{\zeta} R \RegParGD^{\zeta}/\sqrt{n}\right) \log {2 \over \delta}.
	$$
	When $\zeta \geq 1/2,$ we know that there exists a $\FH \in \HK$ such that
	$\IK \FH = \FR$ \cite[Section 4.5]{steinwart2008support}.
In fact,  letting $g = \LK^{-\zeta} \FR$, for $\zeta \geq 1/2,$ $\FR$ can be written as
$$
\FR = \LK^{\zeta} g = (\IK \IK^* )^{\zeta} g =  \IK (\IK^* \IK)^{\zeta - {1\over 2}} 
(\IK^* \IK)^{ - {1\over 2}} \IK^* g = \IK \TK^{\zeta -1/2} (\IK^* \IK)^{ - {1\over 2}}  \IK^* g.  
$$
Choosing $\FH = \TK^{\zeta - {1\over 2}} (\IK^* \IK)^{ - {1\over 2}} \IK^* g,$ as $(\IK^* \IK)^{ - {1\over 2}} \IK^*$ is partial isometric from $\LR$ to $\HK$ and $\zeta \geq {1/2}$, $\FH$ is well defined. Moreover, $\IK \FH = \FR$ and 
$$
\|r_{t+1} - \FH \|_{\HK} = \| \GL(\TK) \IK^*\FR - \FH\|_{\HK} = \| \GL(\TK) \IK^*\IK\FH - \FH\|_{\HK} =   \| (\GL(\TK) \TK - I)\FH\|_{\HK},
$$
where we used \eqref{eq:popSeqEq} for the first equality. 
Introducing with  $\FH = \TK^{\zeta - 1} \IK^* g,$ with $\|g\|_{\rho} \leq R$ by Assumption \ref{as:regularity}, 
$$
\|r_{t+1} - \FH \|_{\HK}  \leq  \| (\GL(\TK) \TK - I)\TK^{\zeta - 1} \IK^*\| \|g\|_{\rho}  \leq  \| (\GL(\TK) \TK - I)\TK^{\zeta - 1/2}\| R. 
$$
Using Lemma \ref{lem:operatorNorm} with \eqref{eq:TKBound}, we get
$$
\|r_{t+1} - \FH \|_{\HK}  \leq ((\zeta-1/2)/\mathrm{e})^{\zeta-1/2}  \RegParGD^{\zeta-1/2} R. 
$$
Combing with  \eqref{eq:inftyThk}, 
	\begin{align*}
	\|r_{t+1} - \FR\|_{\infty} =  & \|r_{t+1} - \FH\|_{\infty} \leq   \kappa \|r_{t+1} - \FH\|_{\HK} 
	\leq  \kappa ((\zeta-1/2)/\mathrm{e})^{\zeta-1/2} R \RegParGD^{\zeta-1/2}.
	\end{align*}
	When $\zeta < 1/2,$ by Part 2) of Lemma \ref{lem:detmiticSeq},
	$\|r_{t+1}\|_{\HK} \leq R \RegParGD^{\zeta-1/2}$. Combining with \eqref{eq:inftyThk}, we have
	$$
	\|r_{t+1} - \FR\|_{\infty} \leq \kappa\|r_{t+1}\|_{\HK} + \|\FR\|_{\infty} \leq \kappa \RegParGD^{\zeta-1/2} R + \|\FR\|_{\infty}   .
	$$
	From the above analysis, we get that with probability at least $1-\delta,$ 
	\bea
	\DZS \leq \log{2 \over \delta}
	\begin{cases}
		2\kappa R \big(2\kappa((\zeta-1/2)/\mathrm{e})^{\zeta-1/2}/ {(\RegParGD n)} + (\zeta/\mathrm{e})^{\zeta}/\sqrt{\RegParGD n} \big) \RegParGD^{\zeta+1/2}, & \mbox {if } \zeta \geq 1/2, \\
		2\kappa  \big(2\kappa R/ (\RegParGD n) + {2\|\FR\|_{\infty} (n\RegParGD)^{-\zeta-1/2}} +  (\zeta/\mathrm{e})^{\zeta} R/\sqrt{n\RegParGD}\big) \RegParGD^{\zeta+1/2}, & \mbox {if } \zeta < 1/2,
	\end{cases}
	\eea
	which can be further relaxed as 
	\be\label{eq:2}
	\DZS \leq C_4 \tilde{R} (1 \vee (\RegParGD n)^{-1}) \RegParGD^{\zeta+1/2} \log {2 \over \delta}, \quad \widetilde{R} = R + {\bf{1}}_{\{\zeta< 1/2 \}} \|\FR\|_{\infty} . 
	\ee 
	where 
	\bea
	C_4 \leq 
	\begin{cases}
		2\kappa  \big(2\kappa((\zeta-1/2)/\mathrm{e})^{\zeta-1/2}+ (\zeta/\mathrm{e})^{\zeta} \big), & \mbox {if } \zeta \geq 1/2, \\
		2\kappa  \big(2\kappa + {2 +  (\zeta/\mathrm{e})^{\zeta} }\big), & \mbox {if } \zeta < 1/2.
	\end{cases}
	\eea
	Applying Lemma \ref{lem:statEstiOper}, and combining with the fact that $\|\TK - \TX \| \leq \|\TK - \TX\|_{HS},$ 
	we have that with probability at least $1- \delta,$
	\be\label{eq:3}
	\DZT  \leq {6\kappa^2 \over \sqrt{n}} \log {2 \over \delta}.
	\ee

	For $0<\zeta\leq 1,$ by Pat 1) of Lemma \ref{lem:anaResult}, \eref{eq:4} and \eref{eq:2}, we have that with probability at least $1-2\delta,$
	$$
	\| \IK h_{t+1} - \FR\|_{\rho} \leq  \left(3^{\zeta \vee {1\over 2}}C_1 R a_{n,\delta,\gamma}^{\zeta \vee {1 \over 2}} (1-\theta) + 2\sqrt{3}C_{4} \widetilde{R} a_{n,\delta,\gamma}^{ {1 \over 2}} (1-\theta)\log{2 \over \delta} \right) \left(1 \vee \left({\PRegPar \over \RegParGD}\right)^{\zeta \vee {1\over 2}}\vee {1\over n\RegParGD }\right)\RegParGD^{\zeta}.
	$$
	Rescaling $\delta$, and then  combining with Lemma \ref{lem:hipbExp}, 
	we get 
	\begin{align*}
	&	\mE	\| \IK h_{t+1} - \FR\|_{\rho}^2  \nonumber\\
	\leq&  \int_{0}^1\left(3^{\zeta \vee {1\over 2}}C_1 a_{n,\delta/2,\gamma}^{\zeta \vee {1 \over 2}} (1-\theta) + 2\sqrt{3}C_{4} a_{n,\delta/2,\gamma}^{ {1 \over 2}} (1-\theta)\log{4 \over \delta} \right)^2 d \delta \left(1 \vee \left({\PRegPar \over \RegParGD}\right)^{2\zeta \vee 1}\vee {1\over n^2\RegParGD^2 }\right)\RegParGD^{2\zeta} \widetilde{R}^2.  
	\end{align*}
	By a direct computation,  noting that since $\PRegPar \geq  n^{-1}$ and $2\zeta  \leq 2,$
	$$
	1 \vee \left({\PRegPar \over \RegParGD}\right)^{2\zeta \vee 1}\vee {1\over n^2\RegParGD^2 }\leq 1 \vee \left({\PRegPar \over \RegParGD}\right)^{2},
	$$
	and that for all $b\in \mR_+,$
	\be\label{eq:gamFunc}
	\int_{0}^1 \log^b {1 \over t} dt =  \Gamma(b+1),
	\ee	
	one can prove the first desired result with
	\be \label{eq:const5}
	C_5= 2 [ C_1^2(48\kappa^2)^{2\zeta \vee 1}  (A^{2\zeta \vee 1} + 2) + 192\kappa^2 C_4^2 (A (\log^2 4 +2+2 \log 4) +  \log^2 4 + 4\log 4  +6)],\ A = \log {8\kappa^2 (c_{\gamma}+1)\mathrm{e} \over \|\TK\|}.
	\ee
	For $\zeta >1,$ by Part 2) of Lemma \ref{lem:anaResult}, \eref{eq:4}, \eref{eq:2} and \eref{eq:3}, we know that with probability at least $1-3\delta,$
	\bea
	&&\|\IK h_{t+1} - \FR\|_{\rho}\\
	&& \leq \sqrt{3} R (C_2 + 2C_4 + 6\kappa^2 C_3 )a_{n,\delta,\gamma}^{1\over 2} (1-\theta) \log{2 \over \delta} \left(1 \vee {\PRegPar^{\zeta}\over \RegParGD^{\zeta}} \vee {1\over n\RegParGD} \vee \RegParGD^{{1\over 2}-\zeta}  \left({1\over n}\right)^{{(\zeta-{1\over 2})\wedge 1 \over 2}}\right) \RegParGD^{\zeta} .
	\eea
	Rescaling $\delta$, and applying Lemma \ref{lem:hipbExp}, we get
	\begin{align*}
	&\mE\|\IK h_{t+1} - \FR\|_{\rho}^2\\
	& \leq 3 (C_2 + 2C_4 + 6\kappa^2 C_3 )^2 R^2 \int_0^1 a_{n,\delta/3,\gamma} (1-\theta) \log^2{6 \over \delta} d\delta \left(1 \vee {\PRegPar^{2\zeta}\over \RegParGD^{2\zeta}} \vee {1\over n^2\RegParGD^2} \vee \RegParGD^{1-2\zeta}  \left({1\over n}\right)^{{(\zeta-{1\over 2})\wedge 1}}\right) \RegParGD^{2\zeta}.
	\end{align*}
	This leads to the second desired result with 
	\begin{align}\label{eq:const6}
	C_6 =  24\kappa^2 (C_2 + 2C_4 + 6\kappa^2 C_3 )^2 ((A + 1)\log^2 6 + 2 (A+2) \log 6 + 2A + 6), 
	 \ A= \log {12 \kappa^2 (c_{\gamma}+1)\mathrm{e} \over \|\TK\|}.
\end{align}
	by noting that $n^{-1}\leq \PRegPar.$
	The proof is complete.
\end{proof}

Combining Proposition \ref{pro:localBias} with Lemma \ref{lem:fulLocBias}, we get the following results for the bias of the fully averaged estimator.

\begin{pro}\label{pro:fulBias}
	Under Assumptions \ref{as:regularity} and \ref{as:eigenvalues}, for any $\PRegPar=n^{-1+\theta}$ with $\theta \in [0,1]$ and any $t \in [T]$, the following results hold. \\
	1) For $0<\zeta \leq 1,$ 
	\be\label{eq:fullBias}
	\mE \|\IK \bar{h}_{t+1} - \FR\|_{\rho}^2 \leq C_5(R + {\bf{1}}_{\{\zeta< 1/2 \}} \|\FR\|_\infty)^2 \left(1 \vee  {\PRegPar^2 \over \RegParGD^2} \vee [\gamma(\theta^{-1}\wedge \log n)]^{2\zeta \vee 1}\right)\RegParGD^{2\zeta}.
	\ee 
	2) For $\zeta>1,$ 
	\be\label{eq:fullBiasB}
	\mE \|\IK \bar{h}_{t+1} - \FR\|_{\rho}^2 \leq C_6 R^2 \left(1 \vee {\PRegPar^{2\zeta}\over \RegParGD^{2\zeta}}  \vee {\RegParGD^{1-2\zeta}} \left({1\over n}\right)^{(\zeta-{1\over 2})\wedge 1} \vee [\gamma (\theta^{-1} \wedge \log n)] \right) \RegParGD^{2\zeta}.
	\ee 
	Here, $C_5$ and $C_6$ are given by Proposition \ref{pro:localBias}.
\end{pro}

\subsection{Estimating Sample Variance}	\label{subsec:samVar}
In this section, we estimate sample variance $\|\IK (\bar{g}_{t} - \bar{h}_{t})\|_{\rho}.$ We first introduce the following lemma.

\begin{lemma}\label{lem:samVar1}
	For any $t \in [T],$ we have 
	\be
	\mE\|\IK (\bar{g}_{t} - \bar{h}_{t})\|_{\rho}^2 = {1\over m}  \mE\|\IK({g}_{1, t} - h_{1,t}) \|_{\rho}^2.
	\ee
\end{lemma}

According to Lemma \ref{lem:samVar1}, we know that the sample variance of the averaging over $m$ local estimators can be well controlled in terms of the  sample variance of a local estimator. In what follows, we will estimate the local sample variance, $\mE\|\IK({g}_{1, t} - h_{1,t}) \|_{\rho}^2$. Throughout the rest of this subsection, we shall drop the index $s=1$ for the first local estimator whenever it shows up, i.e., we rewrite $g_{1,t}$ as $g_{t}$, $\bz_1$ as $\bz$, etc.

\begin{pro}\label{pro:localSamVar}
	Under Assumption \ref{as:eigenvalues},
	let $\PRegPar=n^{\theta - 1}$ for some $\theta \in [0,1].$ Then for any $t\in [T],$
	\bea
	\mE	\|\IK (g_{t+1} - h_{t+1})\|_{\rho}^2 \leq
	C_8  {\sigma^2 \over n\PRegPar^{\gamma}}\left(1 \vee {\PRegPar \over \RegParGD} \vee[ \gamma(\theta^{-1} \wedge \log n])\right).  
	\eea
	Here, $C_8$ is a positive constant depending only on $ \kappa,c_{\gamma},\|\TK\|$ and will be given explicitly in the proof.
\end{pro}
\begin{proof}
	Following from Lemma \ref{lem:gdSRA},
	$$
	g_{t+1} - h_{t+1} = G_{t}(\TX) (\SX^* \by - \LX \FR).
	$$
	For notational simplicity, we
	let $\epsilon_i = y_i - \FR(x_i)$ for all $i\in[n]$ and $\bfep = (\epsilon_i)_{ 1\leq i\leq n}$. Then the above can be written as
	$$
	g_{t+1} - h_{t+1} = G_t(\TX) \SX^* \bfep.
	$$
	Using the above relationship and the isometric property \eqref{isometry}, we have 
	\begin{align*}
	\mE_{\by}	\|\IK (g_{t+1} - h_{t+1})\|_{\rho}^2 
	& = \mE_{\by}	\|\IK G_t(\TX) \SX^* \bfep \|_{\rho}^2 \\
	& = \mE_{\by} \|\TK^{1/2} G_t(\TX) \SX^* \bfep \|_{\HK}^2  \\
	&= {1 \over n^2} \sum_{l,k=1}^n\mE_{\by} [\epsilon_l \epsilon_k] \tr\left(G_t(\TX)  \TK G_t(\TX) K_{x_l} \otimes K_{x_k}\right).
	\end{align*}
	Here, $\mE_{\by}$ denotes the expectation with respect to $\by$ conditional on $\bx.$
	From the definition of $\FR$ and the independence of $z_l$ and $z_k$ when $l\neq k,$ we know that $\mE_{\by} [\epsilon_l \epsilon_k] = 0$ whenever $l\neq k.$ Therefore,
	$$	\mE_{\by}	\|\IK (g_{t+1} - h_{t+1})\|_{\rho}^2 
	= {1 \over n^2} \sum_{k=1}^n\mE_{\by} [\epsilon_k^2] \tr\left( G_t(\TX) \TK G_t(\TX) K_{x_k} \otimes K_{x_k}\right). 
	$$
	Using the condition \eqref{noiseExp},
	\begin{align*}
	\mE_{\by}	\|\IK (g_{t+1} - h_{t+1})\|_{\rho}^2 
	\leq&  {\sigma^2 \over n^2} \sum_{k=1}^n \tr\left(G_t(\TX) \TK G_t(\TX) K_{x_k} \otimes K_{x_k}\right) \\
	=& {\sigma^2 \over n} \tr\left(  \TK (G_t(\TX))^2 \TX \right) \\
	=& {\sigma^2 \over n} \tr\left( \TKL^{-1/2} \TK \TKL^{-1/2} \TKL^{1/2} (G_t(\TX))^2 \TX \TKL^{1/2} \right) \\
	\leq & {\sigma^2 \over n} \tr(\TKL^{-1/2} \TK \TKL^{-1/2}) \|\TKL^{1/2} G_t(\TX)^2 \TX \TKL^{1/2}\| \\
	\leq & {\sigma^2 \mcN(\PRegPar)\over n} \|\TKL^{1/2} \TXL^{-1/2}\|  \|\TXL^{1/2} G_t(\TX)^2 \TX \TXL^{1/2}\|\| \TXL^{-1/2} \TKL^{1/2}\|  \\
	\leq & {\sigma^2 \mcN(\PRegPar) \over n} \DZF \|G_t(\TX)\TX\| \|G_t(\TX)\TXL\| \\
	\leq& {\sigma^2 \mcN(\PRegPar) \over n} \DZF (1 + \PRegPar/\RegParGD),
	\end{align*}
	where $\DZF$ is given by Lemma \ref{lem:anaResult} and  we used 1) of Lemma \ref{lem:operatorNorm} for the last inequality.
	Taking the expectation with respect to $\bx$, this leads to 
	\bea
	\mE	\|\IK (g_{t+1} - h_{t+1})\|_{\rho}^2 \leq
	{\sigma^2 \mcN(\PRegPar) \over n}(1 + \PRegPar/\RegParGD)  \mE[\DZF].  
	\eea
	Applying Lemmas \ref{lem:operDifRes} and \ref{lem:hipbExp},  we get
	\begin{align*}
	\mE	\|\IK (g_{t+1} - h_{t+1})\|_{\rho}^2 &\leq
	6 {\sigma^2 \mcN(\PRegPar) \over n}(1 \vee (\PRegPar/\RegParGD)) \int_0^1 a_{n,\delta,\gamma}(2/3, 1-\theta) d\delta  \\
	\leq& C_7 {\sigma^2 \mcN(\PRegPar) \over n}(1 \vee (\PRegPar/\RegParGD) \vee[ \gamma(\theta^{-1} \wedge \log n])),
	\end{align*}
	where $C_7 = 48\kappa^2 \log {4\kappa^2 (c_{\gamma}+1)\mathrm{e} \over \|\TK\|}.$  Using Assumption \ref{as:eigenvalues}, we get the desired result with 
	\be \label{eq:const8} 
	C_8 = c_{\gamma} 48\kappa^2 \log {4\kappa^2 (c_{\gamma}+1)\mathrm{e} \over \|\TK\|}.\ee
\end{proof}

Using the above proposition and Lemma \ref{lem:samVar1}, we derive the following results for sample variance.
\begin{pro}\label{pro:samVar}
	Under Assumption \ref{as:eigenvalues}, let $\PRegPar=n^{\theta - 1}$ for some $\theta \in [0,1].$ Then for any $t\in [T]$,
	\be\label{eq:samVar}
	\mE	\|\IK (\bar{g}_{t+1} - \bar{h}_{t+1})\|_{\rho}^2 \leq
	C_8 {\sigma^2 \over N \PRegPar^{\gamma}}\left(1\vee \left({\PRegPar \over \RegParGD}\right)  \vee [ \gamma(\theta^{-1} \wedge \log n)]\right).
	\ee
	Here, $C_8 $ is the positive constant given by Proposition \ref{pro:localSamVar}.
\end{pro}

\subsection{Estimating Computational Variance}
In this section, we estimate computational variance, $\mE[\|\IK(\bar{f}_{t} - \bar{h}_{t})\|_{\rho}^2].$
We begin with the following lemma, from which we can see that the global computational variance can be estimated in terms of local computational variances.

\begin{lemma}\label{lem:comVar1}
	For any $t \in [T],$ we have 
	\be
	\mE\|\IK (\bar{f}_{t} - \bar{g}_{t})\|_{\rho}^2 =  {1 \over m^2}  \sum_{s=1}^m  \mE\|\IK({f}_{s, t} - g_{s,t}) \|_{\rho}^2.
	\ee
\end{lemma}

In what follows, we will estimate the local computational variance, i.e., $\mE\|\IK({f}_{s, t} - g_{s,t}) \|_{\rho}^2.$ As in Subsections \ref{subsec:bias} and \ref{subsec:samVar}, we will drop the index $s$ for the $s$-th local estimator whenever it shows up.
We first introduce the following two lemmas, see \cite[Lemmas 20 and 24]{lin2017optimal}. The empirical risk $\mcE_{\bz}(f)$ of a function $f$ with respect to the samples $\bz$ is defined as
$$
\mcE_{\bz}(f) = {1\over n} \sum_{(x,y) \in \bz} (f(x) - y)^2. 
$$

\begin{lemma}\label{lem:empriskB}
	Assume that for all $t \in [T]$ with $t \geq 2,$
	\be\label{empriskBCon}
	{1 \over \eta_t}  \sum_{k=1}^{t-1}{1 \over k(k+1)}  \sum_{i=t-k}^{t-1} {\eta_i^2} \leq {1 \over 4\kappa^2}.
	\ee
	Then for all $t \in[T],$
	\be\label{empiricalBConse}
	\sup_{k\in [t]} \mE_{\J}[\mcE_{\bf z}(f_k)]  \leq {8 \mcE_{\bf z}(0)\Sigma_1^t \over \eta_t t}.
	\ee
\end{lemma}

\begin{lemma}\label{lem:cul_err} For any $t\in[T]$, we have
	\be\label{eq:cul_err}
	\mE_{\bf J}\|\IK f_{t+1} - \IK g_{t+1}\|_{\rho}^2
	\leq {\kappa^2 \over b} \sum_{k=1}^t \eta_{k}^2 \left\|\TK^{1\over 2}\Pi^t_{k+1}(\TX)\right\|^2  \mE_{\J}[\mcE_{\bf z}(f_k)].
	\ee
	Here, $\mE_{\J}$ denotes the expectation with respect to $\J$ conditional on $\bz.$
\end{lemma}

Now, we are ready to state and prove the result for local computational variance as follows.

\begin{pro}\label{pro:localComVar}
	Assume that \eref{empriskBCon} holds for any $t\in [T]$ with $t\geq 2.$
	Let $\PRegPar = n^{-\theta+1}$ for some $\theta\in [0,1]$. For any $t\in [T]$,
	$$
	\mE\|\IK f_{t+1} - \IK g_{t+1}\|_{\rho}^2 
	\leq
	C_{9} M^2 (1 \vee [\gamma(\theta^{-1} \wedge \log n)]) b^{-1}
	\sup_{k \in [t]} \left\{ {\Sigma_1^k \over \eta_k k}
	\right\} \left( \sum_{k=1}^{t-1} \eta_{k}^2 (\PRegPar + \lambda_{k+1:t}\mathrm{e}^{-1}) + \eta_t^2 \right).
	$$
	Here, $C_9$ is a positive constant depending only on 
	$\kappa,c_{\gamma},\|\TK\|$ and can be given explicitly in the proof.
\end{pro}
\begin{proof}
	Following from Lemmas \ref{lem:cul_err} and \ref{lem:empriskB}, we have that,
	$$
	\mE_{\bf J}\|\IK f_{t+1} - \IK g_{t+1}\|_{\rho}^2
	\leq {8\kappa^2 \mcE_{\bf z}(0)   \over b }   \sum_{k=1}^t \eta_{k}^2 \left\|\TK^{1\over 2}\Pi^t_{k+1}(\TX)\right\|^2  \sup_{k \in [t]} \left\{ {\Sigma_1^k \over \eta_k k}\right\}.
	$$
	Taking the expectation with respect to $\by$ conditional on $\bx$, and then with respect to $\bx,$ noting that $\int_{Y} y^2 d \rho(y|x) \leq M^2,$ we get 
	$$
	\mE\|\IK f_{t+1} - \IK g_{t+1}\|_{\rho}^2
	\leq {8\kappa^2 M^2 \over b}  \sup_{k \in [t]} \left\{ {\Sigma_1^k \over \eta_k k}
	\right\} \sum_{k=1}^t \eta_{k}^2 \mE \left\|\TK^{1\over 2}\Pi^t_{k+1}(\TX)\right\|^2.
	$$
	Note that
	\begin{align*}
	&\left\|\TK^{1\over 2}\Pi^t_{k}(\TX)\right\|^2  \leq\|\TK^{1\over 2}\TXL^{-1/2}\|^2 \| \TXL^{1/2} \Pi^t_{k}(\TX)\|^2 \leq \DZF  \| \TXL (\Pi^t_{k}(\TX))^2\| \\ &\leq  \DZF ( \| \TX \Pi^t_{k}(\TX) \| + \PRegPar \|\Pi^t_{k}(\TX)\|) \|\Pi^t_{k}(\TX)\|  \leq 
	\DZF ( \lambda_{k:t}\mathrm{e}^{-1} + \PRegPar),
	\end{align*}
	where $\DZF$ is given by Lemma \ref{lem:anaResult} and
	 for the last inequality we used Part 2) of Lemma \ref{lem:operatorNorm}.
	Therefore, 
	$$
	\mE\|\IK f_{t+1} - \IK g_{t+1}\|_{\rho}^2 
	\leq \mE[\DZF]
	{8\kappa^2 M^2 \over b}  \sup_{k \in [t]} \left\{ {\Sigma_1^k \over \eta_k k}
	\right\} \left( \sum_{k=1}^{t-1} \eta_{k}^2 (\PRegPar + \lambda_{k+1:t}\mathrm{e}^{-1}) + \eta_t^2 \right).
	$$
	Using Lemmas \ref{lem:operDifRes} and \ref{lem:hipbExp}, and by a simple calculation, one can upper bound $\mE[\DZF]$ and consequently
	prove the desired result with $C_9$ given by 
	$$
	C_9	= 192\kappa^4 \log {4\kappa^2 (c_{\gamma}+1)\mathrm{e} \over \|\TK\|}.
	$$
	The proof is complete.
\end{proof}

Combining Lemma \ref{lem:comVar1} with Proposition \ref{pro:localComVar}, we have the following error bounds for computational variance.

\begin{pro}\label{pro:ComVar}
	Assume that \eref{empriskBCon} holds for any $t\in [T]$ with $t\geq 2.$
	Let $\PRegPar = n^{-\theta+1}$ for some $\theta\in [0,1]$. For any $t\in [T],$
	\be\label{eq:comVar}
	\mE\|\IK (\bar{f}_{t+1} - \bar{g}_{t+1}\|_{\rho}^2 
	\leq
	C_{9} M^2 (1 \vee [\gamma(\theta^{-1} \wedge \log n)]) {1 \over mb}
	\sup_{k \in [t]} \left\{ {\Sigma_1^k \over \eta_k k}
	\right\} \left( \sum_{k=1}^{t-1} \eta_{k}^2 (\PRegPar + \lambda_{k+1:t}\mathrm{e}^{-1}) + \eta_t^2 \right).
	\ee
	Here, $C_9$ is the positive constant from Proposition \ref{pro:localComVar}.
\end{pro}

\subsection{Deriving Total Errors}

We are now ready to derive total error bounds for (distributed) SGM and to 
prove the main theorems for (distributed) SGM of this paper.  

\medskip

\par {\noindent \bf Proof	of Theorem \ref{thm:main}\ }
	We will use Propositions \ref{pro:errDecob}, \ref{pro:fulBias}, \ref{pro:samVar} and \ref{pro:ComVar} to prove the result.
	
	We first show that the condition \eref{etaRestri} implies \eref{empriskBCon}. Indeed, when $\eta_t = \eta,$ for any $t\in[T]$
	$$
	{1 \over \eta_t}  \sum_{k=1}^{t-1}{1 \over k(k+1)}  \sum_{i=t-k}^{t-1} {\eta_i^2}  = \eta \sum_{k=2}^{t} {1 \over k} \leq \eta \sum_{k=2}^{t}\int_{k-1}^k {1 \over x} dx = \eta \log t \leq  {1 \over 4\kappa^2}
	$$
	where for the last inequality, we used the condition \eref{etaRestri}.  
	Thus, by Proposition \ref{pro:ComVar}, \eref{eq:comVar} holds. Note also that $\lambda_{k+1:t} = {1 \over \eta (t-k)}$ and $\RegParGD = {1 \over \eta t}$ as $\eta_t = \eta.$
	It thus follows from \eref{eq:comVar} that  
	\bea
	\mE\|\IK (\bar{f}_{t+1} - \bar{g}_{t+1}\|_{\rho}^2 
	\leq
	C_{9} M^2 (1 \vee [\gamma(\theta^{-1} \wedge \log n)]) {\eta \over mb}
	\left( \PRegPar \eta(t-1) + \sum_{k=1}^{t-1}{1 \over \mathrm{e}(t-k)} + \eta \right).
	\eea
	Applying 
	$$
	\sum_{k=1}^{t-1}{1 \over t-k} = \sum_{k=1}^{t-1}{1 \over k} \leq 1 + \sum_{k=2}^{t-1} \int_{k-1}^k {1 \over x} dx \leq 1 + \log t,
	$$
	and \eref{etaRestri}, 
	we get 
	\bea
	\mE\|\IK (\bar{f}_{t+1} - \bar{g}_{t+1}\|_{\rho}^2 
	\leq
	C_{9} M^2 (1 \vee [\gamma(\theta^{-1} \wedge \log n)] \vee \PRegPar \eta t \vee  \log t) {\eta \over mb}
	\left( 2 + {1\over 4\kappa^2} \right).
	\eea
	Introducing the above inequality, \eref{eq:fullBias} (or \eqref{eq:fullBiasB}), and \eref{eq:samVar} into the error decomposition \eref{eq:errDecob}, by a direct calculation, one can prove the desired results with 
	\be\label{eq:const10}
	C_{10} = C_{9} \left( 2 + {1\over 4\kappa^2} \right) =  192\kappa^4 \log {4\kappa^2 (c_{\gamma}+1)\mathrm{e} \over \|\TK\|} \left( 2 + {1\over 4\kappa^2} \right) .
	\ee
\hfill \BlackBox

\medskip
\par {\noindent \bf Proof of Corollary \ref{cor:fastRatP}\ }
	In Theorem \ref{thm:main}, we
	let $\PRegPar = N^{-{1 \over 2\zeta+\gamma}}$. In this case, with Condition \eref{eq:partNum}, it is easy to show that
	$$
	1\geq 	\theta = {\log \PRegPar \over \log n} + 1 = {\log \PRegPar \over \log N - \log m} + 1 \geq -{1\over 2\zeta+\gamma} {\log N \over \log N - \beta \log N} + 1> 0.
	$$
	The proof can be done by simply applying Theorem \ref{thm:main} 
	and plugging with the specific choices of $\eta_t$, $b$, and $T_*$.
\hfill \BlackBox

\medskip 

\par {\bf \noindent Proof of Corollary \ref{cor:simpCa} }
	Since $\FR \in \HK,$ we know from \eqref{eq:HKtoRHO} that Assumption \ref{as:regularity} holds with $\zeta = {1 \over 2}$ and $R \leq \|\FR\|_{\HK}$. As noted in comments after Assumption \ref{as:eigenvalues}, \eref{eigenvalue_decay} trivially holds with $\gamma=1$ and $c_{\gamma} = \kappa^2.$ Applying Corollary \ref{cor:fastRatP}, one can prove the desired results.
\hfill \BlackBox

\medskip

\par { \noindent \bf Proof of Corollary \ref{thm:sgm}\ }
	In Theorem \ref{thm:main}, we let $m=1$ and $n = N$ and $\PRegPar = N^{\theta -1}$ with $\theta = 1-\alpha.$ Then it is easy to see that
	$$
	\gamma(\theta^{-1} \wedge \log N)\leq \begin{cases}
	{\gamma(2\zeta+\gamma ) \over 2\zeta+\gamma -1}, & \mbox{if } 2\zeta +\gamma >1, \\
	\gamma \log N, & \mbox{if } 2\zeta+\gamma \leq 1.
	\end{cases}
	$$
	Following from \eqref{eq:mainErr} or \eqref{eq:mainErr2}, and plugging with the specific choices on $\eta_t, T_*, b$, one can prove the desired error bounds.
\hfill \BlackBox

\newpage

\acks{The authors thank Yen-Huan Li for his comments.
	This work was sponsored by the Department of the Navy, Office of Naval Research (ONR) under a grant number N62909-17-1-2111.  It has also received funding from Hasler Foundation Program: Cyber Human Systems (project number 16066), and from the European Research Council (ERC) under the European Union\rq s Horizon 2020 research and innovation program (grant agreement n 725594-time-data).}


\bibliography{distributed_alg_sgm_sra}
\bibliographystyle{abbrv}

\newpage

\appendix
\begin{center}
 { \bf \Large Appendix }  
 \end{center}
 \begin{itemize}
\item  In {\bf Appendix \ref{sec:notations}}, we provide a list of notations commonly used in this paper.  
\item In {\bf Appendix \ref{app:sgm}}, we prove some of the lemmas and propositions from Section \ref{sec:proof}.
\item  In {\bf Appendix \ref{app:proof}}, we prove our main results for distributed SA. We first introduce an error decomposition,  which decomposes total errors into bias and sample variance. We then estimate these two terms in the following two subsequent subsections. Plugging the two estimates into the error decomposition, we prove the desired results.
\end{itemize}  

{\section{List of Notations} \label{sec:notations}}
{\footnotesize
	\begin{center}
		\begin{longtable}{ c | l  }
			\hline			
			Notation &  Meaning  \\ \hline
			$\HK$ & the hypothesis space, RKHS\\
			
			$X, Y, Z$& the input space, the output space and the sample space ($Z= X \times Y$)  \\
			
			$\rho$, $\rho_X$ & the fixed  probability measure on $Z$, the induced marginal measure of $\rho$ on $X$
			\\
			
			$\rho(\cdot | x)$ &  the conditional probability measure on $Y$ w.r.t. $x\in X$ and $\rho$ \\
			
			$N,n,m$ & the total sample size, the local sample size, the number of partition ($N=nm$)\\
			
			$\Samples$ & the whole samples $\{{z}_i\}_{i=1}^N$, where each $z_i$ is i.i.d. according to $\rho$.\\

			$\bz_s$ & the samples $\{z_{s,i}=(x_{s,i}, y_{s,i})\}_{i=1}^n$ for the $s$-th local machine, $s\in [m]$\\

			$\mcE$ & the expected risk defined by \eqref{generalization_error}\\
			
			
			$\kappa^2$ & the constant from the bounded assumption \eqref{eq:HK} on the hypothesis space $\HK$ \\
			
			$\{f_{s,t}\}_t$ & the sequence generated by SGM over the local sample $\bz_s$, given by \eqref{eq:algsgm}\\
			
			$\{\bar{f}_t\}$ & the sequence generated by distributed SGM, i.e.,  $\bar{f}_t = {1\over m}\sum_{s=1}^m f_{s,t}$ \\
			
			$b$ &  the minibatch size of SGM \\
			
			$T$ & the maximal number of iterations for SGM \\
			
			$j_{s,i} $ ($j_{s,t}$ etc.)& the random index from the uniform distribution on $[n]$ for SGM performing on the $s$-th local sample  set $\bz_s$ \\
			
			$\J_{s,t}$ & the set of random indices at $t$-th iteration of SGM performing on the $s$-th local sample set $\bz_s$\\
			
			$\J_s$ &
			the set of all random indices for SGM performing on the $s$-th local sample set $\bz_s$ after $T$ iterations \\
			
			$\J$ &
			the set of all random indices for distributed SGM after $T$ iterations \\
			
			$\mE_{\J_s}$ & the expectation with respect to the random variables $\J_s$ (conditional on $\bz_s$)\\
			
			$\mE_{\J}$ & the expectation with respect to the random variables $\J$ (conditional on $\Samples$)\\
			
			$\mE_{\by}$ & the expectation with respect to the random variables $\by$ (conditional on $\bx$)\\
			
			$\{\eta_t\}_t$ & the sequence of step-sizes \\
			
			$M, \sigma$ & the positive constants from 
			Assumption \ref{as:noiseExp} \\
			
			$\LR$ & the Hilbert space of square integral functions from $X$ to $\mR$ with respect to $\rho_X$ \\
			
			$f_{\rho}$ & the regression function defined  \eref{regressionfunc} \\
			
			$\zeta,R$& the parameters related to the `regularity' of $\FR$ (see Assumption \ref{as:regularity}) \\
			
			$\gamma, c_{\gamma}$ & the parameters related to the effective dimension (see Assumption \ref{as:eigenvalues}) \\
			
			
			$\{g_{s,t}\}_t$  & the sequence generated by  GM \eqref{eq:alggm} with respect to the $s$-th local sample set $\bz_s$\\
			
			$\{\bar{g}_{t}\}_t$ & the sequence generated by  distributed GM \eqref{eq:DGM}\\
			
			$\{h_{s,t}\}_t$ & the sequence generated by  pseudo GM \eqref{eq:algpgm} over the $s$-th local sample set $\bz_s$\\
			
			$\{\bar{h}_{t}\}_t$ & the sequence generated by distributed pseudo GM\\
			
			$\{r_{t}\}_t$ & the sequence generated by  population GM  \eqref{eq:popSeq}\\
			
			$\IK$ & the inclusion map from $\HK \to \LR$ \\
			
			$\IK^*$ & the adjoint operator of $\IK$, $\IK^* f = \int_{X} f(x) K_x d\rho_{X}(x) $ \\
			
			$\LK$ & the operator from $\LR $ to $\LR$, $\LK(f) = \IK \IK^*f =\int_{X}  f(x) K_x \rho_{X}(x)$\\
			
			$\TK$ & the covariance operator from $\HK$ to $\HK$, $\TK = \IK^* \IK = \int_{X} \la \cdot, K_x \ra_{\HK} K_x d\rho_{X}(x)$ \\
			
			$\SX$ & the sampling operator from $\HK$ to $\mR^{|\bx|}$, $(\SX f)_i = f(x_i), x_i \in \bx$ \\
			
			$\SX^*$ & the adjoint operator of  $\SX$, $\SX^* \mathbf{y} = {1 \over |\bx|} \sum_{i=1}^{|\bx|} y_i K_{x_i}$\\
			
			$\TX$ & the empirical covariance operator, $\TX = \SX^* \SX = {1 \over |\bx|} \sum_{i=1}^{|\bx|} \la \cdot, K_{x_i}\ra_{\HK} K_{x_i}$\\
			
			$\Pi_{t+1}^T(L)$ & $= \Pi_{k=t+1}(I - \eta_k L)$ when $t\in [T-1]$ and $\Pi_{t+1}^{T} = I$ if $t\geq T$ \\
			
			
			$\PRegPar$ & a pseudo regularization parameter, $\PRegPar>0$\\
			
			$\TKL$, &  $\TKL = \TK + \PRegPar$\\
			
			$\TXL$, &  $\TXL = \TX  +\PRegPar$\\ 
			
			$\GL(\cdot)$ &  the filter function of GM, \eqref{eq:regFunc} \\
			
			$\GLB(\cdot)$ &  a general filter function \\
			
			$\lambda$ & a regularization parameter $\lambda>0$ \\
			
			$[t]$ & the set $\{1,\cdots ,t\}$ \\
			
			$b_1 \lesssim  b_2$  & $b_1 \leq C b_2$ for some universal constant  $C>0$ \\
			
			$b_1 \simeq b_2$  &  $b_2 \lesssim b_1 \lesssim b_2$\\
			
			$\DZF,\DZS,\DZT$ & the random quantities defined in Lemma \ref{lem:anaResult} (or Lemma \ref{lem:anaResultSRA})\\

			$\LX$ & an operator defined by \eqref{eq:LX} \\ 
			
			$\Sigma_{k}^t$ & $ = \sum_{i=k}^t \eta_i$ ($= 0$ if $k>t$) \\
			
			$\RegParGD$ & the regularization parameter of GM (= $(\RegPar)^{-1}$) \\
			
			$\lambda_{k:t}$ & $= (\Sigma_{k}^t)^{-1}$ ( = $\infty$ if $k>t$)\\
			
			$a_{|\bx|,\delta,\gamma}(c,\theta)$ & the quantity defined by \eqref{eq:defa}\\

			$\LESRA$ & the estimator defined by SA over the $s$-th local sample set $\bz_s$, see Algorithm \ref{alg:dSpe} \\
			
			$\EDSRA$ & the estimator defined by  distributed SA, see Algorithm \ref{alg:dSpe}\\
			
			$\LEPSRA$ & the estimator defined by pseudo SA over the $s$-th local sample set $\bz_s$,
			\eqref{eq:LEPSRA} \\
			
			$\EDPSRA$ & the estimator defined by  distributed pseudo SA, \eqref{eq:EDPSRA} \\
			
			$\PFSRA$ &   the function defined by population SA, \eqref{eq:popSeqSRA}\\

			\hline  
		\end{longtable}
\end{center}}

\section{Proofs for Section \ref{sec: distributed SGM}} \label{app:sgm}

In this section, we provide the missing proofs of lemmas and propositions from Section \ref{sec: distributed SGM}.
\subsection{Proof of Proposition \ref{pro:errDecob}}
	For any $s \in [m]$,
	using an inductive argument, one can prove that \citep{lin2017optimal}
	\be\label{eq:sgm_gm}
	\mE_{\J_
		s| \bz_s} [f_{s,t}] = g_{s,t}.
	\ee
	Here $\mE_{\J_
		s| \bz_s}$  (or abbreviated as $\mE_{\J_
		s}$) denotes the conditional expectation with respect to $\J_s$ given $\bz_s.$ Indeed, taking the conditional expectation with respect to $\J_{s,t}$ (given $\bz_s$) on both sides of \eref{eq:algsgm}, and noting that 
	$f_{s,t}$ depends only on $\J_{s,1},\cdots,\J_{s,t-1}$ (given $\bz_s$), one has
	$$
	\mE_{\J_{s,t}}[f_{s, t+1}] = f_{s,t} - \eta_t {1 \over n}\sum_{i=1}^n(  f_{s,t}(x_{s,i}) - y_{s, i}) K_{x_{s,i}}, $$
	and thus,
	$$
	\mE_{\J_s}[f_{s, t+1}] = \mE_{\J_s}[f_{s,t}] - \eta_t {1 \over n}\sum_{i=1}^n( \mE_{\J_s}[f_{s,t}](x_{s,i}) - y_{s,i}) K_{x_{s,i}} , \quad t=1, \ldots, T, $$
	which satisfies the iterative relationship given in \eref{eq:alggm}. Similarly, using the definition of the regression function \eref{regressionfunc} and an inductive argument, one can also prove that 
	\be\label{eq:gmn_gmns}
	\mE_{\by_s} [g_{s,t}] = h_{s,t}.
	\ee
	Here, $\mE_{\by_s}$ denotes the conditional expectation with respect to $\by_s$ given $\bx_s.$
	
	We have
	$$
	\|\IK \bar{f}_t - \FR\|_{\rho}^2 = \|\IK \bar{f}_t - \IK \bar{g}_t\|_{\rho}^2 + \|\IK\bar{g}_t - \FR\|_{\rho}^2 + 2\la \IK \bar{f}_t - \IK \bar{g}_t, \IK\bar{g}_t - \FR \ra.
	$$
	Taking the conditional expectation  with respect to $\J$	(given $\bz$) on both sides, using \eref{eq:sgm_gm} which implies
	$$
	\mE_{\J} \IK (\bar{f}_t -  \bar{g}_t) = {1 \over m} \sum_{s=1}^{m} \IK \mE_{\J_s} [f_{s,t} - g_{s,t}] = 0,
	$$
	we thus have 
	$$
	\mE_{\J}\|\IK \bar{f}_t - \FR\|_{\rho}^2 = \mE_{\J}\|\IK \bar{f}_t - \IK \bar{g}_t\|_{\rho}^2 + \|\IK\bar{g}_t - \FR\|_{\rho}^2.
	$$
	Taking the conditional expectation with respect to $\Outputs = \{\by_1,\cdots,\by_m\}$ (given $\Inputs =\{ \bx_1,\cdots,\bx_m\}$), noting that 
	$$
	\mE_{\Outputs}	\|\IK\bar{g}_t - \FR\|_{\rho}^2 = \mE_{\Outputs} [\|\IK(\bar{g}_t -  \bar{h}_t) \|_{\rho}^2] + \|\IK\bar{h}_t - \FR \|_{\rho}^2 + 2 \la  \IK\mE_{\Outputs}[\bar{g}_t - \bar{h}_t], \IK\bar{h}_t - \FR \ra_{\rho}
	$$
	and that from \eref{eq:gmn_gmns},
	$$
	\la  \IK\mE_{\Outputs}[\bar{g}_t - \bar{h}_t], \IK\bar{h}_t - \FR \ra_{\rho} = {1\over m} \sum_{s=1}^m  \la \IK \mE_{\by_s}(g_{s,t} - h_{s,t}), \IK \bar{h}_t - \FR \ra_{\rho} = 0,
	$$
	we know that 
	$$
	\mE_{\Outputs}\mE_{\J}\|\IK \bar{f}_t - \FR\|_{\rho}^2 = \mE_{\Outputs}\mE_{\J}\|\IK \bar{f}_t - \IK \bar{g}_t\|_{\rho}^2 +  \mE_{\Outputs} [\|\IK(\bar{g}_t -  \bar{h}_t) \|_{\rho}^2] + \|\IK\bar{h}_t - \FR \|_{\rho}^2,
	$$
	which leads to the desired result. 
	\hfill\BlackBox

\subsection{Proof of Lemma \ref{lem:fulLocBias}}
By Jensen's inequality, we can prove the desired result:
$$
\mE \|\IK \bar{h}_t -\FR\|_{\rho}^2 = \mE \left\| {1\over m}\sum_{s=1}^m (\IK{h}_{s,t} -\FR) \right\|_{\rho}^2
\leq {1\over m}   \mE \sum_{s=1}^m \left\|\IK {h}_{s,t} -\FR \right\|_{\rho}^2 = \mE \|\IK h_{1,t} - \FR \|_{\rho}^2.
$$
\hfill \BlackBox

\subsection{Proof of Lemma \ref{lem:regFuncIneq}}
1). 
For $\alpha = 0$ or $1$, the proof is  straightforward and can be found in \citep{yao2007early}. Indeed, for all $u \in [0,\kappa^2],$ $\Pi_{k+1}^t(u) \leq 1$ and thus
$G_t(u) \leq \sum_{k=1}^t \eta_k = \RegParGD^{-1}.$ Moreover, 
writing $\eta_k u = 1 - (1 - \eta_k u),$ we have
\be\label{eq:5}
uG_{t}(u) = \sum_{k=1}^t (\eta_k u) \Pi_{k+1}^t(u) =  \sum_{k=1}^t (\Pi_{k+1}^t(u) - \Pi_{k}^t(u)) = 1- \Pi_1^t(u) \leq 1. 
\ee
Now we consider the case $0<\alpha<1$. We have
\bea
u^{\alpha}G_t(u) = |u G_t(u)|^{\alpha} |G_t(u)|^{1-\alpha} \leq \RegParGD^{\alpha-1},
\eea
where we used $uG_t(u) \leq 1$ and $G_t(u) \leq \RegParGD^{-1}$ in the above. 

2) By \eref{eq:5}, we have $(1 - u G_t(u))u^{\alpha} = \Pi_1^t(u) u^{\alpha}$. Then
the desired result is a direct consequence of Conclusion 3). 

3)  The proof can be also found, e.g., in \cite[Page 17]{lin2017optimal}.
Using the basic inequality
\be\label{expx}
1 + x \leq \mathrm{e}^{x} \qquad \mbox{for all } x \geq -1,
\ee
with $\eta_l\kappa^2 \leq 1$, we get
\bea
\Pi_{k+1}^t(u) u^{\alpha} \leq  \exp\left\{ - u \Sigma_{k+1}^t\right\} u^{\alpha}.
\eea
The maximum of the function $g(u) = \mathrm{e}^{-cu}u^{\alpha}$ (with $c>0$) over $ \mR_+ $ is achieved at $u_{\max}= \alpha/c,$ and thus
\be\label{exppoly1}
\sup_{u \geq 0} \mathrm{e}^{-cu} u^{\alpha} =  \left({\alpha \over \mathrm{e}c} \right)^{\alpha}.
\ee
Using this inequality with $c = \Sigma_{k+1}^t$, one can prove the desired result. 
\hfill\BlackBox

\subsection{Proof of Lemma \ref{lem:operatorNorm}}
1) \ 	Following from the spectral theorem, one has
$$
\|(L+\PRegPar)^{\alpha} G_t(L)\| \leq \sup_{u \in [0,\kappa^2]} (u+\PRegPar)^{\alpha} G_t(u) \leq \sup_{u \in [0,\kappa^2]} (u^{\alpha} +\PRegPar^{\alpha}) G_t(u) .
$$
Using Part 1) of Lemma \ref{lem:regFuncIneq} to the above, one can prove the first conclusion. \\
2) \ Using the spectral theorem, 
$$
\|\Pi_1^t(L) (L+\PRegPar)^{\alpha}\| \leq \sup_{u \in [0,\kappa^2]} (u + \PRegPar)^{\alpha} \Pi_1^t(u).
$$
When $\alpha \leq 1,$ 
$$\sup_{u \in [0,\kappa^2]} (u + \PRegPar)^{\alpha} \Pi_1^t(u) \leq \sup_{u \in [0,\kappa^2]} (u^{\alpha} + \PRegPar^{\alpha}) \Pi_1^t(u) \leq (\alpha/\mathrm{e})^{\alpha} \lambda_t^{\alpha} + \PRegPar^{\alpha},$$
where for the last inequality, we used Part 2) of Lemma \ref{lem:regFuncIneq}. Similarly, when $\alpha >1,$ by H\"{o}lder's inequality, and Part 2) of Lemma \ref{lem:regFuncIneq},
$$\sup_{u \in [0,\kappa^2]} (u + \PRegPar)^{\alpha} \Pi_1^t(u) \leq 2^{\alpha-1}\sup_{u \in [0,\kappa^2]} (u^{\alpha} + \PRegPar^{\alpha}) \Pi_1^t(u) \leq 2^{\alpha-1}((\alpha/\mathrm{e})^{\alpha} \lambda_t^{\alpha} + \PRegPar^{\alpha}).$$
From the above analysis, one can prove the second conclusion.\\
3) \ Simply applying the spectral theorem and 3) of Lemma \ref{lem:regFuncIneq}, one can prove the third conclusion.

\hfill \BlackBox

\subsection{Proof of Lemma \ref{lem:statEstim}}
We first introduce the following concentration result for Hilbert space valued random variable
used in \citep{caponnetto2007optimal} and based on the results in \citep{pinelis1986remarks}.

\begin{lemma}
	\label{lem:Bernstein}
	Let $w_1,\cdots,w_m$ be i.i.d random variables in a separable Hilbert space with norm $\|\cdot\|$. Suppose that
	there are two positive constants $B$ and $\sigma^2$ such that
	\be\label{bernsteinCondition}
	\mE [\|w_1 - \mE[w_1]\|^l] \leq {1 \over 2} l! B^{l-2} \sigma^2, \quad \forall l \geq 2.
	\ee
	Then for any $0< \delta <1/2$, the following holds with probability at least $1-\delta$,
	$$ \left\| {1 \over m} \sum_{k=1}^m w_m - \mE[w_1] \right\| \leq 2\left( {B \over m} + {\sigma \over \sqrt{ m }} \right) \log {2 \over \delta} .$$
	In particular, \eref{bernsteinCondition} holds if
	\be\label{bernsteinConditionB}
	\|w_1\| \leq B/2 \ \mbox{ a.s.}, \quad \mbox{and } \quad \mE [\|w_1\|^2] \leq \sigma^2.
	\ee
\end{lemma}

Lemmas \ref{lem:statEstim} and \ref{lem:statEstiOper} 
can be proved by simply applying the above lemma.

\medskip
{\par \noindent{\bf Proof of Lemma \ref{lem:statEstim}\ }
	Let $\xi_i= f(x_i) K_{x_i}$ for $i=1,\cdots, |{\bf x}|.$ Obviously,
	\bea
	\LX f - \LK f = {1 \over |{\bx}|} \sum_{i=1}^{|{\bf x}|} (\xi_i - \mE[\xi_i]),
	\eea
	and by Assumption \eref{eq:HK}, we have
	$$\|\xi\|_{\HK} \leq \|f\|_{\infty}\|K_{x}\|_{\HK} \leq \kappa\|f\|_{\infty}$$ and
	\bea
	\mE\|\xi\|_{\HK}^2 \leq \kappa^2\|f\|_{\rho}^2.
	\eea
	Applying Lemma \ref{lem:Bernstein} with $B'= 2\kappa \|f\|_{\infty} $ and $\sigma = \kappa \|f\|_{\rho} ,$ one can prove the desired result.
\hfill \BlackBox

\subsection{Proof of Lemma \ref{lem:statEstiOper}}
	Let $\xi_i = K_{x_i} \otimes K_{x_i},$ for all $i \in [|\bx|].$ Obviously,
	\bea
	\TK - \TX = {1 \over |\bx|} \sum_{i=1}^{|\bx|} (\mE[\xi_i] - \xi_i),
	\eea
	and by Assumption \eref{eq:HK},
	$\|\xi_i\|_{HS} = \|K_{x_i}\|_{\HK}^2 \leq \kappa^2.$
	Applying Lemma \ref{lem:Bernstein} with $B' = 2\kappa^2$ and $\sigma' = \kappa^2$, one can prove the desire result.

\hfill \BlackBox

\subsection{Proof of Lemma \ref{lem:operDifEff}}

In order to prove Lemma \ref{lem:operDifEff}, we introduce the following concentration inequality for norms of self-adjoint operators on a Hilbert space.
\begin{lemma}
	\label{lem:concentrSelfAdjoint}
	Let $\mcX_1, \cdots, \mcX_m$ be a sequence of independently and identically distributed self-adjoint Hilbert-Schmidt operators on a separable Hilbert space.
	Assume that $\mE [\mcX_1] = 0,$ and $\|\mcX_1\| \leq B$ almost surely for some $B>0$. Let $\mathcal{V}$ be a positive trace-class operator such that $\mE[\mcX_1^2] \preccurlyeq \mathcal{V}.$
	Then with probability at least $1-\delta,$ ($\delta \in ]0,1[$), there holds
	\bea
	\left\| {1 \over m} \sum_{i=1}^m \mcX_i \right\| \leq {2B \beta \over 3m} + \sqrt{2\|\mathcal{V}\|\beta \over m }, \qquad \beta = \log {4 \tr \mathcal{V} \over \|\mathcal{V}\|\delta}.
	\eea
\end{lemma}
\begin{proof}
	The proof can be found in, e.g., \citep{rudi2015less,dicker2017kernel}.
	Following from the argument in \cite[Section 4]{minsker2011some}, we can generalize \cite[Theorem 7.3.1]{tropp2012user} from a sequence of self-adjoint matrices to a sequence of self-adjoint Hilbert-Schmidt operators on a separable Hilbert space, and get that for any $t \geq \sqrt{{\|\mcV\|\over m}} + {B\over 3m},$
	\be\label{eq:conf}
	\Pr\left( \left\| {1 \over m} \sum_{i=1}^m \mcX_i \right\| \geq t \right) \leq  {4\tr \mcV \over \|\mcV\|} \exp\left( {- mt^2 \over 2\|\mcV\| + 2Bt/3 }\right).
	\ee
	Rewriting
	\bea
	{4\tr \mcV \over \|\mcV\|} \exp\left( {- mt^2 \over 2\|\mcV\| + 2Bt/3 }\right) = \delta,
	\eea
	as a quadratic equation with respect to the variable $t$, and then solving the quadratic equation, we get
	\bea
	t_0 = {B\beta \over 3m} + \sqrt{\left({B \beta \over 3m}\right)^2 + {2\beta \|\mcV\| \over m}} \leq {2B\beta \over 3m} + \sqrt{{2\beta \|\mcV\| \over m}} := t^*,
	\eea
	where we used $\sqrt{a+b} \leq \sqrt{a} + \sqrt{b},\forall a,b>0.$ Note that $\beta>1$, and thus $t_0\geq \sqrt{{\|\mcV\|\over m}} + {B\over 3m}.$ By
	\bea
	\Pr\left( \left\| {1 \over m} \sum_{i=1}^m \mcX_i \right\| \geq t_* \right) \leq \Pr\left( \left\| {1 \over m} \sum_{i=1}^m \mcX_i \right\| \geq t_0 \right),
	\eea
	and applying \eref{eq:conf} to bound the left-hand side, one can get the desire result.
\end{proof}

Applying the above lemma, one can prove 
Lemma \ref{lem:operDifEff} as follows.
\medskip 
{\par \noindent \bf Proof of Lemma \ref{lem:operDifEff} \ }
	The proof can be also found in \citep{rudi2015less,dicker2017kernel,hsu2014random}. Unlike the result
	 in \citep{rudi2015less} which requires the condition $\lambda \leq \|\TK\|,$ our results hold for any $\lambda>0.$
	We will use Lemma \ref{lem:concentrSelfAdjoint} to prove the result.
	Let $|\bx| = m$ and $\mcX_i = \TKL^{-1/2} (\TK - \TK_{x_i}) \TKL^{-1/2},$ for all $i\in[m].$ Then $\TKL^{-1/2} (\TK - \TX) \TKL^{-1/2} = {1\over m}\sum_{i=1}^m \mcX_i.$  Obviously, for any $\mcX= \mcX_i$, $\mE[\mcX]=0,$ and
	\bea
	\|\mcX\| \leq \mE\left[\|\TKL^{-1/2} \TK_{x} \TKL^{-1/2} \|\right] + \|\TKL^{-1/2} \TK_{x} \TKL^{-1/2}\| \leq 2\kappa^2/ \PRegPar,
	\eea
	where for the last inequality, we used Assumption \eref{eq:HK} which implies
	\bea
	\|\TKL^{-1/2} \TK_{x} \TKL^{-1/2}\| \leq \tr(\TKL^{-1/2} \TK_{x} \TKL^{-1/2}) = \tr(\TKL^{-1} \TK_{x}) = \la \TKL^{-1} K_{x}, K_{x} \ra_{\HK} \leq \kappa^2/\PRegPar.
	\eea
	Also, by $\mE(A - \mE A)^2 \preccurlyeq \mE A^2,$
	\bea
	\mE\mcX^2 &\preccurlyeq& \mE(\TKL^{-1/2} \TK_{\tilde{x}} \TKL^{-1/2})^2 = \mE [\la \TKL^{-1} K_{x}, K_{x} \ra_{\HK} \TKL^{-1/2} K_{x} \otimes  K_{x} \TKL^{-1/2}] \\
	&\preccurlyeq & {\kappa^2 \over \PRegPar} \mE [ \TKL^{-1/2} K_{x} \otimes  K_{x} \TKL^{-1/2}] = {\kappa^2 \over \PRegPar} \TKL^{-1} \TK = \mcV,
	\eea
	Note that $\|\TKL^{-1}\TK\| = {\|\TK\| \over \|\TK\|+ \PRegPar} \leq 1.$ Therefore,
	$\|\mcV\| \leq {\kappa^2 \over \PRegPar}$ and
	$${\tr(\mcV) \over \|\mcV\|} = {\mcN(\PRegPar)\|\TK\| + \tr(\TKL^{-1}\TK) \PRegPar \over \|\TK\|} \leq {\mcN(\PRegPar)\|\TK\| + \tr(\TK)\over \|\TK\|} \leq {\kappa^2(\mcN(\PRegPar)+  1)\over \|\TK\|},$$
	where for the last inequality we used \eref{eq:TKBound}.
	Now, the result can be proved by applying Lemma \ref{lem:concentrSelfAdjoint}.
\hfill \BlackBox

\subsection{Proof of Lemma \ref{lem:samVar1}}
Note that from the independence of $\bz_1,\cdots, \bz_m$ and \eref{eq:gmn_gmns}, we have
\begin{align*}
\mE_{\Outputs}\|\IK (\bar{g}_{t} - \bar{h}_{t})\|_{\rho} = {1 \over m^2} \sum_{s,l=1}^m \mE_{\Outputs}\la \IK({g}_{s, t} - h_{s,t}),  \IK({g}_{l, t} - h_{l,t}) \ra_{\rho} = {1 \over m^2}  \sum_{s=1}^m  \mE_{\by_s}\|\IK({g}_{s, t} - h_{s,t}) \|_{\rho}^2.
\end{align*}
Taking the expectation with respect to $\Inputs,$ we get
\begin{align*}
\mE\|\IK (\bar{g}_{t} - \bar{h}_{t})\|_{\rho} = {1\over m^2} \sum_{s=1}^m  \mE\|\IK({g}_{s, t} - h_{s,t}) \|_{\rho}^2 = {1\over m}  \mE\|\IK({g}_{1, t} - h_{1,t}) \|_{\rho}^2.
\end{align*}
The proof is complete.
\hfill \BlackBox

\subsection{Proof of Lemma \ref{lem:comVar1}}
Note that by \eref{eq:sgm_gm} and from the conditional independence of $\J_{s},\cdots \J_{m}$ (given $\Samples$), we have
\begin{align*}
\mE_{\J}\|\IK (\bar{f}_{t} - \bar{g}_{t})\|_{\rho} = {1 \over m^2} \sum_{s,l=1}^m \mE_{\J}\la \IK({f}_{s, t} - g_{s,t}),  \IK({f}_{l, t} - g_{l,t}) \ra_{\rho} = {1 \over m^2}  \sum_{s=1}^m  \mE_{\J_s}\|\IK({f}_{s, t} - g_{s,t}) \|_{\rho}^2.
\end{align*}
Taking the expectation with respect to $\Samples,$ we thus prove the desired result.
\hfill \BlackBox

\section{Proofs for Distributed Spectral Algorithms}\label{app:proof}
The proof for distributed SGM in Section \ref{sec: distributed SGM} involves the analysis 
for distributed GM.
In this section, we will extend our analysis for distributed GM to distributed SA. 
The proof almost follows along the same lines as the proof for distributed GM in Subsections \ref{subsec:bias} and \ref{subsec:samVar}, but some of them need some delicate modifications, the reason for which lies in that the qualification $\tau$ for GM can be any positive number while it is a fixed constant for a general SA.

\subsection{Error Decomposition}
We begin with an error decomposition. To introduce the error decomposition, we define an auxiliary function, generated by pseudo-SA as follows.
	Given a spectral function $\GLB,$
	for any $s \in [m],$ the function $\LEPSRA$ generated by the pseudo spectral algorithm over $\bx_s$ is given by 
	\be\label{eq:LEPSRA}
	\LEPSRA = \GLB(\TXS)\LXS \FR. 
	\ee
	The estimator generated by distributed pseudo-spectral  algorithm is the averaging over these local estimators,
	\be\label{eq:EDPSRA}
	\EDPSRA = {1 \over m}\sum_{s=1}^m \LEPSRA.
	\ee

We note that the above algorithm can not be implemented in practice as the regression function $\FR$ is unknown. 
From the definition of the regression function, similar to \eref{eq:gmn_gmns}, we can prove that
\be\label{eq:sra_psra}
\mE_{\by_s} [\LESRA] =  \LEPSRA,
\ee
and thus
$$
\mE_{\Outputs} [\EDSRA] =  \EDPSRA,
$$
Using these basic properties, analogous to Proposition \ref{pro:errDecob}, we have the following error decomposition for distributed SA.

\begin{pro}\label{pro:errDecSRA}
	We have 
	$$	\mE \| \IK \EDSRA - \FR \|_{\rho}^2 = \mE\| \EDPSRA - \FR\|_{\rho}^2 + \mE\| \IK \EDSRA - \EDPSRA\|_{\rho}^2.$$
\end{pro}
The right-hand side is composed of two terms. The first term is called as bias, and the second term is called as sample variance. In what follows, we will estimate these two terms separably.

\subsection{Estimating Bias}

Analogous to Lemma \ref{lem:fulLocBias}, we can show that the bias term 
$\mE\| \EDPSRA - \FR\|_{\rho}^2$ can be upper bounded in terms of the local bias $\mE\| \FLEPSRA - \FR\|_{\rho}^2$. 
\begin{lemma}\label{lem:fulLocBiasSRA}
	We have 
	$\mE\| \EDPSRA - \FR\|_{\rho}^2 \leq \mE\| \FLEPSRA - \FR\|_{\rho}^2 .$
\end{lemma}
\begin{proof}
	The proof is the same as that in Lemma \ref{lem:fulLocBias} by using H\"{o}lder's inequality.
\end{proof}

In what follows, we  will estimate the local bias $\mE\| \FLEPSRA - \FR\|_{\rho}^2$. Throughout the rest of this subsection, we shall drop the index $s=1$ for the first local estimator whenever it shows up, i.e., we rewrite $\FLEPSRA$ as $\EPSRA$, $\bz_1$ as $\bz$, etc.
To do so, we need to introduce a population function defined by 
\be\label{eq:popSeqSRA}
\PFSRA = \GLB(\TK) \IK^* \FR. 
\ee 
The function $\PFSRA$ is deterministic and it is independent from the samples.
Since $\GLB(\cdot) $ is a filter function with qualification $\tau>0$ and constants $E,F_{\tau},$ similar to Lemma \ref{lem:operatorNorm}, we have the following results for a filter function according to the spectral theorem.

\begin{lemma}
	\label{lem:operatorNormSRA}
	Let $L$ be a compact, positive operator on a separable Hilbert space $H$ such that $ \|L\| \leq \kappa^2$. Then for any $\PRegPar \geq 0$,\\
	1) $\|(L+\PRegPar)^{\alpha} \GLB(L)\| \leq E\lambda^{\alpha-1} (1 + (\PRegPar/\lambda)^{\alpha} ), \, \, \,\forall \alpha \in [0,1].$\\
	2) $\|(I - L \GLB(L))(L+\PRegPar)^{\alpha}\| \leq F_{\tau}2^{(\alpha-1)_+}\lambda^{\alpha} (1 + (\PRegPar/\lambda)^{\alpha}), \quad \forall \alpha \in [0,\tau].$ 
\end{lemma}

With the above lemma, analogous to Lemma \ref{lem:detmiticSeq}, we have the following properties for the population function.

\begin{lemma}\label{lem:detmiticSeqSRA}
	Under Assumption \ref{as:regularity}, the following results hold. \\
	1) For any $\zeta - \tau \leq a \leq \zeta,$ we have
	\bea
	\|\LK^{-a}\left(\IK \PFSRA - \FR\right)\|_{\rho} \leq F_{\tau} R  \lambda^{\zeta-a}. 
	\eea
	2) 	We have
	\be
	\|\TK^{a-1/2} \PFSRA \|_{\HK} \leq E R\cdot  \begin{cases}
		\lambda^{\zeta+a -1}, & \text{if } -\zeta \leq a\leq 1-\zeta, \\
		\kappa^{2(\zeta +a-1)},& \text{if } \ a\geq 1 -\zeta.
	\end{cases}
	\ee
\end{lemma}
Note that there is a subtle difference between Lemma \ref{lem:detmiticSeq}.(1) and Lemma \ref{lem:detmiticSeqSRA}.(1).
The latter requires $a \geq \zeta - \tau$ while the former does not, the reason for which is that, the qualification $\tau$ is fixed in the latter while it can be any positive constant in the former. 
This difference makes the proof for SA slightly different to the one for GM, when estimating the bias. 

\begin{proof}
	1) 
	According to the spectral theory,
	$$\IK \GLB(\TK) \IK^* = \IK \GLB(\IK^* \IK) \IK^* =  \GLB(\IK\IK^* ) \IK\IK^* = \GLB(\LK ) \LK.$$
	Combining with \eref{eq:popSeqSRA}, we thus have
	\bea
	\LK^{-a}(\IK \PFSRA - \FR) = \LK^{-a}\left(\GLB(\LK)\LK - I\right) \FR.
	\eea
	Taking the $\rho$-norm, and applying Assumption \ref{as:regularity}, we have
	\bea
	\|\LK^{-a}(\IK \PFSRA - \FR)\|_{\rho} \leq \| \LK^{\zeta-a}(\GLB(\LK)\LK - I)\| R.
	\eea
	Note that the condition \eref{eq:HK} implies
	\eref{eq:TKBound}.  By a similar argument as that for 2) of Lemma \ref{lem:operatorNormSRA},  one can  prove the first desired result.
	
	2) By \eref{eq:popSeqSRA} and Assumption \ref{as:regularity}, 
	$$
	\|\TK^{a-1/2} \PFSRA\|_{\HK} = 	\|\TK^{a-1/2} \GLB(\TK) \IK^* \FR \|_{\HK} \leq  \|\TK^{a-1/2} \GLB(\TK) \IK^* \LK^{\zeta} \|R.
	$$
	Noting that
	\begin{align*}
	\|\TK^{a-1/2} \GLB(\TK) \IK^* \LK^{\zeta} \| & = \|\TK^{a-1/2} \GLB(\TK) \IK^* \LK^{2\zeta} \IK \GLB(\TK) \TK^{a-1/2} \|^{1/2}  \\
	&= \| \GLB^2(\TK) \TK^{2\zeta+2a}\|^{1/2} = \|\GLB(\TK) \TK^{\zeta+a}\|,
	\end{align*}
	we thus have
	$$
	\|\TK^{a-1/2} \PFSRA \|_{\HK}  \leq \|\GLB(\TK) \TK^{\zeta+a}\| R.
	$$
	If $0\leq \zeta+a \leq 1,$ i.e., $ -\zeta  \leq a \leq 1 -\zeta$, then  using 1) of Lemma \ref{lem:operatorNormSRA}, we get  
	$$
	\|\TK^{a-1/2} \PFSRA \|_{\HK}  \leq \lambda^{\zeta + a-1} E R.
	$$
	Similarly, when $a \geq  1 -\zeta,$ we have
	$$
	\|\TK^{a-1/2} \PFSRA\|_{\HK}  \leq \|\GLB(\TK) \TK\|\|\TK\|^{\zeta+a-1}  R \leq  \kappa^{2(\zeta+a -1)} E R,
	$$
	where for the last inequality we used 1) of
	Lemma \ref{lem:operatorNormSRA} and \eref{eq:TKBound}. This thus proves the second desired result.
\end{proof}

With the above lemmas, similar to Lemma \ref{lem:anaResult}, we have the following analytic result, which enables us  to estimate the bias term in terms of several random quantities.

\begin{lemma}\label{lem:anaResultSRA}
	Under Assumption \ref{as:regularity}, let
	\bea
	\DZF = \|\TKLP^{1/2} \TXLP^{-1/2}\|^2 \vee 1, \qquad \DZT = \|\TK- \TX\|
	\eea
	and
	\bea
	\DZS = \| \LX f_{\rho} - \IK^* f_{\rho} - \TX \PFSRA + \TK \PFSRA \|_{\HK}.
	\eea
	Then the following results hold for any $\PRegPar>0$.\\
	1) For $0<\zeta \leq 1$
	\be\label{eq:anaResASRA}
	\|\IK \EPSRA - \FR\|_{\rho} \leq  \left( 1 \vee \left({\PRegPar \over \lambda}\right)^{\zeta \vee {1\over 2}}\right) (C_1' R (\DZF)^{\zeta \vee {1\over 2}} \lambda^{\zeta} + 2 E \sqrt{\DZF} \lambda^{-{1\over 2}} \DZS) .
	\ee
	2) For $\zeta >1,$
	\be\label{eq:anaResBSRA}
	\|\IK \EPSRA - \FR\|_{\rho} \leq   \DZF\left(1 \vee \left({\PRegPar\over \lambda} \right)^{\zeta}\right) (C_2' R \lambda^{\zeta} +  2 E \lambda^{-{1\over 2}} \DZS + C_{3}' R\lambda^{1 \over 2} (\DZT)^{(\zeta-{1\over 2})\wedge 1} ) .
	\ee
	Here, $C_1'$, $C_2'$ and $C_3'$ are positive constants depending only on $\zeta,\kappa,E$, and $F_{\tau}$.
\end{lemma}

The upper bound from \eqref{eq:anaResBSRA} is a bit worser than the one in \eqref{eq:anaResB}.

\begin{proof}
	We can estimate $\|\IK \EPSRA - \FR\|_{\rho}$ as
	\begin{align}
	\|\IK \GLB(\TX)\LX \FR -\FR\|_{\rho} \leq & \|\underbrace{\IK \GLB(\TX)[\LX \FR - \IK^* \FR - \TX \PFSRA + \TK \PFSRA]}_{\text{\bf Bias.1}}\|_{\rho} \nonumber\\
	&+ \| \underbrace{\IK \GLB(\TX)[\IK^* f_{\rho} - \TK \PFSRA]}_{\text{\bf Bias.2}}\|_{\rho} \nonumber\\
	&+ \|\underbrace{\IK [I - \GLB(\TX)\TX] \PFSRA}_{\text{\bf Bias.3}}\|_{\rho}\nonumber\\
	&+  \|\underbrace{\IK \PFSRA - \FR}_{\text{\bf Bias.4}}\|_{\rho}. \label{eq:biasDecomSRA}
	\end{align}
	In the rest of the proof, we will estimate the four terms of the r.h.s separately.\\
	{\bf Estimating Bias.4}\\
	Using 1) of Lemma \ref{lem:detmiticSeqSRA} with $a=0$, we get
	\be\label{eq:bias4BSRA}
	\|{\bf Bias.4} \|_{\rho} \leq F_{\tau} R \lambda^{\zeta}.
	\ee 
	{\bf Estimating Bias.1} \\
		By a simple calculation and \eref{eq:7}, we know that for any $f\in \HK$ and any $b \in [0,{1\over 2}],$
	\bea
	\| \IK \GLB(\TX)f\|_{\rho} \leq \|\TKLP^{1/2} \TXLP^{-1/2}\| \|\TXLP^{1/2} \GLB(\TX)\TXLP^{b}\| \|\TXLP^{-b}\ \TKLP^{b} \| \| \TKLP^{-b} f\|_{\HK}.
	\eea
	Note that by 1) of Lemma \ref{lem:operatorNormSRA}, with \eref{eq:TXbound},
	$$
	\|\TXLP^{1/2} \GLB(\TX)\TXLP^{b}\| \leq E(1 + (\PRegPar/\lambda)^{b+{1\over 2}})\lambda^{b-{1\over 2}},
	$$
	and by Lemma \ref{lem:operProd}, we get
	$$
	\|\TXLP^{-b}\ \TKLP^{b} \| \leq \|\TXLP^{-{1\over 2}}\ \TKLP^{1\over 2} \|^{2b}.   
	$$
	Therefore, for any $f\in \HK$ and any $b \in [0,{1\over 2}],$ 
		\be\label{eq:6SRA}
	\| \IK \GLB(\TX)f\|_{\rho} \leq (\DZF)^{b+{1\over 2}} E(1 + (\PRegPar/\lambda)^{b+{1\over 2}})\lambda^{b-{1\over 2}}  \| \TKLP^{-b} f\|_{\HK}.
	\ee
	Letting $f = \LX \FR - \IK^* \FR - \TX \PFSRA + \TK \PFSRA $ and $b={1\over 2}$ in the above, we get
	\begin{align}\label{eq:bias1BSRA}
	\|{ \bf Bias.1}\|_{\rho} \leq E(1+\sqrt{\PRegPar/\lambda})\lambda^{-{1 \over 2}}
	\sqrt{\DZF} \DZS.
	\end{align}
	{\bf 
		Estimating Bias.2}\\
	Thus, letting $f = \TK\PFSRA -\IK^* \FR,$ in \eqref{eq:6SRA}, we have
	\begin{align*}
	\| \mbox{\bf Bias.2}\|_{\rho} &\leq E \|\TKLP^{1\over 2} \TXLP^{-{1\over2}}\|^{2b+1} (1 + (\PRegPar/\lambda)^{b+{1\over 2}})\lambda^{b-{1\over 2}} \| \TKLP^{-b} [\TK\PFSRA -\IK^* \FR]\|_{\HK} \\
	&\leq E (\DZF)^{b+{1 \over 2}} (1 + (\PRegPar/\lambda)^{b+{1\over 2}}) \lambda^{b-{1\over 2}}\| \LKL^{-b+{1\over 2}} [\IK\PFSRA - \FR]\|_{\HK}.
	\end{align*}
	When $\zeta \leq {1\over 2}$, we have $\tau - \zeta \geq {1\over 2}$ since $\tau \geq 1.$
	Letting $b=0$, and  applying Lemma \ref{lem:detmiticSeqSRA}.(1) with $a=-{1\over 2}$ , we get 
	\begin{align*}
	\| \IK \GLB(\TX)[\TK\PFSRA -\IK^* \FR]\|_{\rho} 
	&\leq E F_{\tau}R (\DZF)^{{1 \over 2}} (1 + (\PRegPar/\lambda)^{{1\over 2}}) \lambda^{\zeta}.
	\end{align*}
	Similarly, when ${1\over 2} \leq \zeta \leq 1,$ we choose $b=  \zeta - {1\over 2},$ and applying Lemma \ref{lem:detmiticSeqSRA}.(1) with $a=\zeta -1$, we get 
	\begin{align*}
	\| \IK \GLB(\TX)[\TK\PFSRA -\IK^* \FR]]\|_{\rho} 
	&\leq E F_{\tau}R (\DZF)^{\zeta} (1 + (\PRegPar/\lambda)^{\zeta}) \lambda^{\zeta}.
	\end{align*}
	When 	
	$ \zeta \geq 1,$ we choose $b=  {1\over 2},$ and applying Lemma \ref{lem:detmiticSeqSRA}.(1) with $a=0$, we get 
	\begin{align*}
	\| \IK \GLB(\TX)[\TK\PFSRA -\IK^* \FR]]\|_{\rho} 
	&\leq E F_{\tau}R \DZF (1 + (\PRegPar/\lambda)) \lambda^{\zeta}.
	\end{align*}
	From the above estimate, we get
	\begin{align}\label{eq:bias2BSRA} 
	\|{ \bf Bias.2}\|_{\rho} \leq EF_{\tau} R \lambda^{\zeta} \times
	\begin{cases}
	(1+(\PRegPar/\lambda)^{1/2}) 
	(\DZF)^{1/2} & \mbox{if } 0<\zeta \leq 1/2, \\
	(1+(\PRegPar/\lambda)^{\zeta})  
	(\DZF)^{\zeta}  & \mbox{if } 1/2<\zeta \leq 1,\\
	(1+\PRegPar/\lambda) 
	\DZF & \mbox{if } \zeta > 1.
	\end{cases}
	\end{align}
	{\bf Estimating Bias.3}\\
	When $\zeta \leq 1/2,$ 	by a simple calculation and \eref{eq:7}, we have
	\begin{align*}
	\|{\bf Bias.3}\|_{\rho} \leq& \|\IK\TKLP^{-1/2}\| \|\TKLP^{1/2} \TXLP^{-1/2}\| \|  \TXLP^{1/2}(I - \GLB(\TX)\TX)\| \| \PFSRA \|_{\HK} \\ 
	\leq & \sqrt{\DZF} \|  \TXLP^{1/2}(I - \GLB(\TX)\TX)\|  \| \PFSRA \|_{\HK},
	\end{align*}
	By 2) of Lemma \ref{lem:operatorNormSRA}, with \eref{eq:TXbound}, 
	\be\label{eq:8SRA}
	\|  \TXLP^{1/2}(I - \GLB(\TX)\TX) \| \leq F_{\tau}( 1 + \sqrt{\PRegPar/\lambda}) \sqrt{\lambda},
	\ee
	and by 2) of Lemma \ref{lem:detmiticSeqSRA}, 
	$
	\|\PFSRA\|_{\HK} \leq E R \lambda^{\zeta-1/2} .
	$
	It thus follows that
	$$
	\|{\bf Bias.3}\|_{\rho} \leq  \sqrt{\DZF} (1 +  \sqrt{\PRegPar/\lambda}) E F_{\tau} R \lambda^{\zeta}. 
	$$
	When $1/2 < \zeta \leq 1,$
	by a simple computation, we have 
	\begin{align*}
	\|{\bf Bias.3}\|_{\rho} \leq 
	\|\IK\TKLP^{-{1\over 2}}\|  \|\TKLP^{1\over 2} \TXLP^{-{1\over 2}}\|  \| \TXLP^{{1\over 2}}(I - \GLB(\TX) \TX)\TXLP^{\zeta - {1\over 2}}\| \|
	\TXLP^{{1\over 2}-\zeta} \TKLP^{\zeta - {1\over 2}}\|  \|\TKLP^{{1\over 2}-\zeta} \PFSRA\|_{\HK}.
	\end{align*}
	Applying \eref{eq:7} and 2) of Lemma \ref{lem:detmiticSeqSRA}, we have 
	\begin{align*}
	\|{\bf Bias.3}\|_{\rho} \leq 
	\sqrt{\DZF} \| \TXLP^{1\over 2}(I - \GLB(\TX) \TX)\TXLP^{\zeta - {1\over 2}}\| \|
	\TXLP^{{1\over 2}-\zeta} \TKLP^{\zeta - {1\over 2}}\|   E R.
	\end{align*}
	By 2) of Lemma \ref{lem:operatorNormSRA}, 
	$$
	\| \TXLP^{1\over 2}(I - \GLB(\TX) \TX)\TXLP^{\zeta - {1\over 2}}\| \leq  F_{\tau} ( 1+  (\PRegPar/\lambda)^{\zeta})\lambda^{\zeta}.
	$$
	Besides, by $\zeta \leq 1$ and Lemma \ref{lem:operProd}, 
	$$
	\|
	\TXLP^{{1\over 2}-\zeta} \TKLP^{\zeta - {1\over 2}}\|  =   \|\TXLP^{-{1\over 2}(2\zeta-1)} \TKLP^{{1 \over 2}(2\zeta - 1)}\| \leq  \|\TXLP^{-{1\over 2}} \TKLP^{{1 \over 2}}\|^{2\zeta - 1} \leq (\DZF)^{\zeta -{1 \over 2}}.
	$$
	It thus follows that
	$$
	\|{\bf Bias.3}\|_{\rho} \leq  (\DZF)^{\zeta} (1+ (\PRegPar/\lambda)^{\zeta}) E F_{\tau} R \lambda^{\zeta}.
	$$
	When $\zeta>1,$ we  rewrite {\bf Bias.3} as
	\begin{align*}
	\IK\TKLP^{-{1\over 2}} \cdot \TKLP^{1\over 2} \TXLP^{-{1\over 2}}\cdot  \TXLP^{1\over 2}(I - \GLB(\TX)\TX) (\TX^{\zeta-{1\over 2}} + \TK^{\zeta-{1\over 2}} - \TX^{\zeta-{1\over 2}}) \TK^{{1\over 2} - \zeta} \PFSRA.
	\end{align*}
	By a simple calculation and \eref{eq:7}, we can  upper bound $\|{\bf Bias.3}\|_{\rho}$ by
	$$
	\leq \|\TKLP^{1\over 2} \TXLP^{-{1\over 2} }\| ( \|\TXLP^{1\over 2} (I - \GLB(\TX)\TX) \TX^{\zeta-{1\over 2}}\|+ \|\TXLP^{1\over 2}(I - \GLB(\TX)\TX) \| \|\TK^{\zeta-{1\over 2}} - \TX^{\zeta-{1\over 2}}\|) \|\TK^{{1\over 2} - \zeta} \PFSRA\|_{\HK}.
	$$
	Introducing with \eref{eq:8SRA}, and applying 2) of Lemma \ref{lem:detmiticSeqSRA}, 
	$$
	\|{\bf Bias.3}\|_{\rho} \leq \sqrt{\DZF} ( \|\TXLP^{1\over 2} (I - \GLB(\TX)\TX) \TX^{\zeta-1/2}\|+ F_{\tau}(\sqrt{\PRegPar/\lambda} + 1)\sqrt{\lambda} \|\TK^{\zeta-1/2} - \TX^{\zeta-1/2}\|) ER.
	$$
	By  2) of Lemma \ref{lem:operatorNormSRA},
	\begin{align*}
	\|\TXLP^{1\over 2} (I - \GLB(\TX)\TX) \TX^{\zeta-{1\over 2}}\| \leq& 	\|\TXLP^{\zeta} (I - \GLB(\TX)\TX)\| \\
	\leq& 2^{\zeta-1} F_{\tau} ( 1+ (\PRegPar/\lambda)^{\zeta})\lambda^{\zeta}.
	\end{align*}
	Moreover, by Lemma \ref{lem:operDiff} and $\max(\|\TK\|,\|\TX\|) \leq \kappa^2$,
	$$\|\TK^{\zeta-{1\over 2}} - \TX^{\zeta-{1\over2}}\|  \leq (2\zeta \kappa^{2\zeta-3})^{\mathbf{1}_{\{2\zeta\geq3\}}} \|\TK - \TX\|^{(\zeta-{1\over 2}) \wedge 1}.
	$$ 
	Therefore, when $\zeta >1$, {\bf Bias.3} can be estimated as
	\begin{align*}
	\|{\bf Bias.3}\|_{\rho} \leq&   \sqrt{\DZF} \left(2^{\zeta-1}  (1 + (\PRegPar/\lambda)^{\zeta}) \lambda^{\zeta}+ 
	(2\zeta \kappa^{2\zeta-3})^{\mathbf{1}_{\{2\zeta\geq3\}}}  (\sqrt{\PRegPar/\lambda} + 1)\sqrt{\lambda}(\DZT)^{(\zeta-{1\over 2}) \wedge 1}\right) E F_{\tau} R.
	\end{align*}
	From the above analysis,  we know that $\|{\bf Bias.3}\|_{\rho}$ can be upper bounded by 
	\be\label{eq:bias3BSRA}EF_{\tau}R
	\begin{cases} 
		\sqrt{\DZF} (\sqrt{\PRegPar/\lambda} + 1)  \lambda^{\zeta},& \mbox{if } \zeta \in ]0,1/2], \\
		(\DZF)^{\zeta} ((\PRegPar/\lambda)^{\zeta} + 1)  \lambda^{\zeta}, & \mbox{if } \zeta \in ]1/2,1], \\ 
		\sqrt{\DZF} \left(2^{\zeta-1} (1 + (\PRegPar/\lambda)^{\zeta}) \lambda^{\zeta} + 
		(2\zeta \kappa^{2\zeta-3})^{\mathbf{1}_{\{2\zeta\geq3\}}}  (\sqrt{\PRegPar/\lambda} + 1)\sqrt{\lambda}(\DZT)^{(\zeta-{1\over 2}) \wedge 1}\right) ,& \mbox{if } \zeta\in ]1,\infty[. 
	\end{cases}
	\ee
	
	Introducing \eref{eq:bias4BSRA},  \eref{eq:bias1BSRA}
	\eref{eq:bias2BSRA} and \eref{eq:bias3BSRA} into \eref{eq:biasDecomSRA}, 
	and by a simple calculation, one can prove the desired results with 
	$$C_1' = F_{\tau}(1 + 4E), 
	$$
	$$C_2' =F_{\tau}
	 \left( 1 + 2 E  +2^{\zeta} E \right),
	$$
	$$
	\mbox{and }\quad C_3' = 2EF_{\tau} (2\zeta \kappa^{2\zeta-3})^{\mathbf{1}_{\{2\zeta\geq3 \}}} .$$	
\end{proof} 

The rest of the proofs parallelize as those for distributed GM.

\begin{pro}\label{pro:localBiasSRA}
	Under Assumptions \ref{as:regularity} and \ref{as:eigenvalues}, we
	let $\PRegPar=n^{-1+\theta}$ for some $\theta \in [0,1]$.  Then  the following results hold.\\
	1) For $0<\zeta \leq 1,$	
	$$
	\mE \|\IK \EPSRA - \FR\|_{\rho}^2 \leq C_5' \left( R + {\bf 1}_{\{2\zeta <1 \}} \|\FR\|_{\infty} \right)^2 \left(1 \vee  [\gamma(\theta^{-1}\wedge \log n)]^{2\zeta \vee 1} \vee  {\PRegPar^2 \over \lambda^2 }\right)\lambda^{2\zeta} 
	$$
	2)For $\zeta>1,$
	$$
	\mE \|\IK \EPSRA - \FR\|_{\rho}^2 \leq C_6' R^2 \left(1 \vee {\PRegPar^{2\zeta}\over \lambda^{2\zeta}}  \vee \lambda^{1-2\zeta} \left({1\over n}\right)^{(\zeta-{1\over 2})\wedge 1} \vee [\gamma (\theta^{-1} \wedge \log n)]^2\right) \lambda^{2\zeta}.
	$$
	Here, $C_5'$ and $C_6'$ are positive constants depending only on $\kappa,\zeta, E, F_{\tau},c_{\gamma}, \|\TK\|$ and can be given explicitly in the proof.
\end{pro}
\begin{proof}
	We will use Lemma \ref{lem:anaResultSRA} to prove the results. To do so, we need to estimate $\DZF,$ $\DZS$ and $\DZT$.
	
	By Lemma \ref{lem:operDifRes}, we have that with probability at least $1-\delta,$ \eref{eq:4} holds,
	where $a_{n,\delta,\gamma}(1-\theta) = a_{n,\delta,\gamma}(2/3,1-\theta)$ is given by \eref{eq:aa}.
	By Lemma \ref{lem:statEstim}, we have that with probability at least $1-\delta,$
	$$\DZS \leq 2\kappa \left( {2 \|\PFSRA  - \FR\|_{\infty} \over n } + {\|\IK \PFSRA - \FR\|_{\rho} \over \sqrt{n}}\right)\log {2 \over \delta} .
	$$
	Applying Lemma \ref{lem:detmiticSeqSRA} with $a=0$ to estimate $\|\IK \PFSRA - \FR\|_{\rho}$,
	we get that with probability at least $1-\delta$,	
	$$\DZS  \leq 2\kappa \left( {2 \|\PFSRA - \FR\|_{\infty} / n } + F_{\tau}R \lambda^{\zeta}/\sqrt{n}\right)\log{2 \over \delta}.
	$$
	When $\zeta \geq 1/2,$ we know that there exists some $\FH = \TK^{\zeta-1} \IK^* \LK^{-\zeta} \FR\in \HK$ such that
	$\IK \FH = \FR$ \citep{steinwart2008support} and 
	$$\|\PFSRA - \FR\|_{\infty} \leq \kappa \|\PFSRA - \FH\|_{\HK} \leq \kappa F_{\tau} R \lambda^{\zeta-1/2},$$
	where for the last inequality, we used Lemma \ref{lem:operatorNormSRA}.
	When $\zeta < 1/2,$
	 by 2) of Lemma \ref{lem:detmiticSeqSRA},
	$\|\PFSRA\|_{\HK} \leq ER \lambda^{\zeta-1/2}$, which thus lead to 
	$$
	\|\PFSRA - \FR\|_{\infty} \leq \kappa\|\PFSRA\|_{\HK} + \|\FR\|_{\infty} \leq \kappa ER \lambda^{\zeta-1/2} + \|\FR\|_{\infty}.
	$$
	From the above analysis, we get that with probability at least $1-\delta,$ 
	\bea
	\DZS \leq \log{2 \over \delta}
	\begin{cases}
		2\kappa F_{\tau} R \big(2\kappa/ {(\lambda n)} + 1/\sqrt{\lambda n} \big) \lambda^{\zeta+1/2}, & \mbox {if } \zeta \geq 1/2, \\
		2\kappa  \big(2\kappa E R/ (\lambda n) + {2 \|\FR\|_{\infty} (n\lambda)^{-\zeta-1/2}} +  F_{\tau} R/\sqrt{n\lambda}\big) \lambda^{\zeta+1/2}, & \mbox {if } \zeta < 1/2,
	\end{cases}
	\eea
	which can be further relaxed as 
	\be\label{eq:2SRA}
	\DZS \leq C_4' \widetilde{R} (1 \vee (\lambda n)^{-1}) \lambda^{\zeta+1/2} \log {2 \over \delta}, \quad \widetilde{R} = R + {\bf 1}_{\{2\zeta <1 \}} \|\FR\|_{\infty},
	\ee 
	where 
	\bea
	C_4' \leq 
	\begin{cases}
		2\kappa F_{\tau} \big(2\kappa + 1\big), & \mbox {if } \zeta \geq 1/2, \\
		2\kappa  \big(2\kappa E + {2 +  F_{\tau} }\big), & \mbox {if } \zeta < 1/2.
	\end{cases}
	\eea
	Applying Lemma \ref{lem:statEstiOper}, 
	we have that with probability at least $1- \delta,$ \eref{eq:3} holds.

	For $0<\zeta\leq 1,$ by Lemma \ref{lem:anaResultSRA}, \eref{eq:4} and \eref{eq:2SRA}, we have that with probability at least $1-2\delta,$
	$$
	\| \IK \EPSRA - \FR\|_{\rho} \leq  \left(3^{\zeta \vee {1\over 2}}C_1'R a_{n,\delta,\gamma}^{\zeta \vee {1 \over 2}} (1-\theta) + 2\sqrt{3}E C_{4}' \widetilde{R} a_{n,\delta,\gamma}^{ {1 \over 2}} (1-\theta)\log{2 \over \delta} \right) \left(1 \vee \left({\PRegPar \over \lambda}\right)^{\zeta \vee {1\over 2}}\vee {1\over n\lambda }\right)\lambda^{\zeta}.
	$$
	Rescaling $\delta$, and then  combining with Lemma \ref{lem:hipbExp},
	we get 
	\begin{align*}
	&	\mE	\| \IK h_{t+1} - \FR\|_{\rho}^2  \nonumber\\
	\leq&  \widetilde{R}^2\int_{0}^1\left(3^{\zeta \vee {1\over 2}} C_1' a_{n,\delta/2,\gamma}^{\zeta \vee {1\over 2}} (1-\theta) + 2\sqrt{3}E C_{4}' a_{n,\delta/2,\gamma}^{ {1 \over 2}} (1-\theta)\log{4 \over \delta} \right)^2 d \delta \left(1 \vee \left({\PRegPar \over \lambda}\right)^{2\zeta \vee 1}\vee {1\over n^2\lambda^2 }\right)\lambda^{2\zeta}.  
	\end{align*}
	By a direct computation and noting that $\PRegPar \geq n^{-1}$ and $\zeta\leq 1,$ 
	one can prove the first desired result with $\ A = \log {8\kappa^2 (c_{\gamma}+1)\mathrm{e} \over \|\TK\|},$ and
	$$
	C_5' = 2 [ C_1'^2(48\kappa^2)^{2\zeta \vee 1}  (A^{2\zeta \vee 1} + \Gamma(3)) + 192\kappa^2 C_4'^2 E^2 (A (\log^2 4 +2 +2\log 4) +  \log^2 4+ 4\log 4  +6)].
	$$
	For $\zeta >1,$ by Lemma \ref{lem:anaResultSRA}, \eref{eq:4}, \eref{eq:2SRA} and \eref{eq:3}, we know that with probability at least $1-3\delta,$
	\bea
	&&\|\IK \EPSRA - \FR\|_{\rho}\\
	&& \leq 3 R (C_2' + 2E C_4' + 6\kappa^2 C_3' )a_{n,\delta,\gamma} (1-\theta) \log{2 \over \delta} \left(1 \vee {\PRegPar^{\zeta}\over \lambda^{\zeta}} \vee {1\over n\lambda} \vee {\lambda^{{1\over 2}-\zeta} }  \left({1\over n}\right)^{{(\zeta-{1\over 2})\wedge 1 \over 2}}\right) \lambda^{\zeta} .
	\eea
	Rescaling $\delta$, and applying Lemma \ref{lem:hipbExp}, we get
	\begin{align*}
	&\mE\|\IK \EPSRA - \FR\|_{\rho}^2\\
	& \leq 9 R^2 (C_2' + 2EC_4' + 6\kappa^2 C_3' )^2 \int_0^1 a_{n,\delta/3,\gamma}^2 (1-\theta) \log^2{6 \over \delta} d\delta \left(1 \vee {\PRegPar^{2\zeta}\over \lambda^{2\zeta}} \vee {1\over n^2\lambda^2} \vee { \lambda^{1-2\zeta}}  \left({1\over n}\right)^{{(\zeta-{1\over 2})\wedge 1}}\right) \lambda^{2\zeta} ,
	\end{align*}
	which leads to the second desired result with 
	$$
	C_6' = 24^3\kappa^4 (C_2' + 2EC_4' + 6\kappa^2 C_3')^2 (A + 1)^2(\log 6+1)^2, \ A= \log {12 \kappa^2 (c_{\gamma}+1)\mathrm{e} \over \|\TK\|},
	$$
	by noting that $\PRegPar \geq n^{-1}$ and $\zeta \geq 1.$
	The proof is complete.
\end{proof}

Combining Proposition \ref{pro:localBiasSRA} with Lemma \ref{lem:fulLocBiasSRA}, we get the following results for the bias of the fully averaged estimators.

\begin{pro}\label{pro:fulBiasSRA}
	Under Assumptions \ref{as:regularity} and \ref{as:eigenvalues}, for any $\PRegPar=n^{-1+\theta}$ with $\theta \in [0,1]$, the following results hold. \\
	1) For $\zeta\leq 1,$
	\be\label{eq:fullBiasSRA}
	\mE \|\IK \EPSRA - \FR\|_{\rho}^2 \leq C_5' \left(R + {\bf 1}_{2\zeta<1} \|\FR\|_{\infty}\right)^2  \left(1 \vee  [\gamma(\theta^{-1}\wedge \log n)]^{2\zeta \vee 1} \vee  {\PRegPar^2 \over  \lambda^2 }\right)\lambda^{2\zeta}.
	\ee
	2) For $1<\zeta \leq \tau,$
	\be\label{eq:fullBiasSRA2}
	\mE \|\IK \EPSRA - \FR\|_{\rho}^2
	\leq C_6' R^2 \left(1 \vee {\PRegPar^{2\zeta}\over \lambda^{2\zeta}}  \vee \lambda^{1-2\zeta} \left({1\over n}\right)^{(\zeta-{1\over 2})\wedge 1} \vee [\gamma (\theta^{-1} \wedge \log n)]^2\right) \lambda^{2\zeta}
	\ee
	Here, $C_5'$ and $C_6'$ are given by Proposition \ref{pro:localBiasSRA}.
\end{pro}

\subsection{Estimating Sample Variance}	\label{subsec:samVarSRA}
In this section, we estimate sample variance $\|\IK (\EDSRA - \EDPSRA)\|_{\rho}.$ We first introduce the following lemma.

\begin{lemma}\label{lem:samVar1SRA}
We have 
	\be
	\mE\|\IK (\EDSRA - \EDPSRA )\|_{\rho} = {1\over m}  \mE\|\IK( \FLESRA - \FLEPSRA) \|_{\rho}^2.
	\ee
\end{lemma}
\begin{proof}
	The proof is the same as that in Lemma \ref{lem:samVar1} by applying \eref{eq:sra_psra}.
\end{proof}

According to Lemma \ref{lem:samVar1SRA}, we know that the sample variance of the averaging over $m$ local estimators can be well controlled in terms of the  sample variance of a local estimator. In what follows, we will estimate the local sample variance, $ \mE\|\IK( \FLESRA - \FLEPSRA) \|_{\rho}^2$. Throughout the rest of this subsection, we shall drop the index $s=1$ and write $\bz_1$ as $\bz$, $\bx_1$ as $\bx$.

\begin{pro}\label{pro:localSamVarSRA}
	Under Assumption \ref{as:eigenvalues},
	let $\PRegPar =n^{\theta - 1}$ for some $\theta \in [0,1].$ Then 
	\bea
	\mE\|\IK( \ESRA - \EPSRA) \|_{\rho}^2 \leq
	C_8' {\sigma^2 \over n \PRegPar^{\gamma}}\left(1 \vee {\PRegPar \over \lambda} \vee[ \gamma(\theta^{-1} \wedge \log n])\right).  
	\eea
	Here, $C_8'$ is a positive constant depending only on $\kappa,c_{\gamma},\|\TK\|, E$ and will be given explicitly in the proof.
\end{pro}
\begin{proof}
	For notational simplicity, we
	let $\epsilon_i = y_i - \FR(x_i)$ for all $i\in[n]$ and $\bfep = (\epsilon_i)_{ 1\leq i\leq n}$. Then from the definitions of $\LEPSRA$ and $\LESRA$
	$$
	\ESRA - \EPSRA = \GLB (\TX) \SX^* \bfep.
	$$
	Using the above relationship and the isometric property \eref{isometry}, we have 
	\begin{align*}
	\mE_{\by}	\|\IK (g_{t+1} - h_{t+1})\|_{\rho}^2 
	& = \mE_{\by}	\|\IK \GLB(\TX) \SX^* \bfep \|_{\rho}^2 \\
	& = \mE_{\by} \|\TK^{1/2} \GLB(\TX) \SX^* \bfep \|_{\HK}^2  \\
	&= {1 \over n^2} \sum_{l,k=1}^n\mE_{\by} [\epsilon_l \epsilon_k] \tr\left(\GLB(\TX)  \TK \GLB(\TX) K_{x_l} \otimes K_{x_k}\right).
	\end{align*}
	From the definition of $\FR$ and the independence of $z_l$ and $z_k$ when $l\neq k,$ we know that $\mE_{\by} [\epsilon_l \epsilon_k] = 0$ whenever $l\neq k.$ Therefore,
	$$	\mE_{\by}	\|\IK (g_{t+1} - h_{t+1})\|_{\rho}^2 
	= {1 \over n^2} \sum_{k=1}^n\mE_{\by} [\epsilon_k^2] \tr\left( \GLB(\TX) \TK \GLB(\TX) K_{x_k} \otimes K_{x_k}\right). 
	$$
	Using Assumption \ref{as:basic},
	\begin{align*}
	\mE\|\IK( \ESRA - \EPSRA) \|_{\rho}^2 
	\leq&  {\sigma^2 \over n^2} \sum_{k=1}^n \tr\left(\GLB(\TX) \TK \GLB(\TX) K_{x_k} \otimes K_{x_k}\right) \\
	=& {\sigma^2 \over n} \tr\left(  \TK (\GLB(\TX))^2 \TX \right) \\
	\leq & {\sigma^2 \over n} \tr(\TKL^{-1/2} \TK \TKL^{1/2}) \|\TKL^{1/2} \GLB(\TX)^2 \TX \TKL^{1/2}\| \\
	\leq & {\sigma^2 \mcN(\PRegPar)\over n} \DZF   \|\TXL^{1/2} \GLB(\TX)^2 \TX \TXL^{1/2}\| \\
	\leq & {\sigma^2 \mcN(\PRegPar) \over n} \DZF \|\GLB(\TX)\TX\| (\|\GLB(\TX)\TX\| +\PRegPar \|\GLB(\TX)\|) \\
	\leq& E^2 {\sigma^2 \mcN(\PRegPar) \over n} \DZF (1 + \PRegPar/\lambda),
	\end{align*}
	where for the last inequality, we used 1) of Lemma \ref{lem:operatorNormSRA}.
	Taking the expectation with respect to $\bx$, this leads to 
	\bea
	\mE\|\IK( \ESRA - \EPSRA) \|_{\rho}^2 \leq
	E^2	{\sigma^2 \mcN(\PRegPar) \over n}(1 + \PRegPar/\lambda)  \mE[\DZF].  
	\eea
	Applying Lemmas \ref{lem:operDifRes} and \ref{lem:hipbExp},  we get
	\begin{align*}
	\mE\|\IK( \ESRA - \EPSRA) \|_{\rho}^2 &\leq
	6 E^2 {\sigma^2 \mcN(\PRegPar) \over n}(1 \vee (\PRegPar/\lambda)) \int_0^1 a_{n,\delta,\gamma}(2/3, 1-\theta) d\delta  \\
	\leq& C_7' {\sigma^2 \mcN(\PRegPar) \over n}(1 \vee (\PRegPar/\lambda) \vee[ \gamma(\theta^{-1} \wedge \log n])),
	\end{align*}
	where $C_7' = 48E^2\kappa^2 \log {4\kappa^2 (c_{\gamma}+1)\mathrm{e}^2 \over \|\TK\|}.$  Using Assumption \ref{as:eigenvalues}, we get the desired result with $$C_8' = c_{\gamma} C_7.$$
\end{proof}

Using the above proposition and Lemma \ref{lem:samVar1SRA}, we derive the following results for sample variance.
\begin{pro}\label{pro:samVarSRA}
	Under Assumption \ref{as:eigenvalues}, let $\PRegPar=n^{\theta - 1}$ for some $\theta \in [0,1].$ Then for any $t\in [T]$,
	\be\label{eq:samVarSRA}
	\mE\|\IK( \EDSRA - \EDPSRA) \|_{\rho}^2 \leq
	C_8' {\sigma^2 \over N \PRegPar^{\gamma}}\left(1\vee \left({\PRegPar \over \lambda}\right)  \vee [ \gamma(\theta^{-1} \wedge \log n)]\right),
	\ee
	where $C_8'$ is given by Proposition \ref{pro:localSamVarSRA}.
\end{pro}

\subsection{Deriving Total Error Bounds}
\medskip
{\par \noindent \bf Proof of Theorem \ref{thm:mainDFliter} \ }
	The proof can be finished by simply applying Propositions \ref{pro:samVarSRA} and \ref{pro:fulBiasSRA} into Proposition \ref{pro:errDecSRA}.
\hfill \BlackBox

\medskip
{\par \noindent \bf Proof of Corollary \ref{thm:Fliter} \ }
	The results are direct consequences of Theorem \ref{thm:mainDFliter}.
\hfill \BlackBox

%

\end{document}